\documentclass{article}

\PassOptionsToPackage{numbers, compress}{natbib}

\usepackage[preprint]{neurips_2026}

\usepackage[utf8]{inputenc} 
\usepackage[T1]{fontenc}    
\usepackage{hyperref}       
\usepackage{url}            
\usepackage{booktabs}       
\usepackage{amsfonts}       
\usepackage{nicefrac}       
\usepackage{microtype}      
\usepackage{xcolor}         

\usepackage{algorithm}
\usepackage{algorithmic}

\usepackage{graphicx}
\usepackage{subcaption}

\usepackage{amsmath}
\usepackage{amssymb}
\usepackage{mathtools}
\usepackage{amsthm}

\usepackage{thmtools}
\definecolor{darkgreen}{rgb}{0.0, 0.5, 0.0}
\usepackage{adjustbox}
\usepackage{wrapfig}
\usepackage{float}
\usepackage{makecell}

\usepackage{thmtools}

\theoremstyle{plain}
\newtheorem{theorem}{Theorem}[section]

\theoremstyle{definition}

\theoremstyle{remark}

\title{Missing Pattern Recognized Diffusion Imputation Model for Missing Not At Random}

\author{
\textbf{Gyuwon Sim}$^{1}$ \quad
\textbf{Sumin Lee}$^{1}$ \quad
\textbf{Heesun Bae}$^{1}$ \quad
\textbf{Byeonghu Na}$^{1}$ \\
\textbf{Doyun Kwon}$^{1}$ \quad
\textbf{Ju-Hee Hwang}$^{2}$ \quad
\textbf{Jae-Young Lim}$^{2}$ \quad
\textbf{Il-Chul Moon}$^{1}$ \\
$^{1}$KAIST \qquad
$^{2}$Seoul National University \\
\texttt{
\{gkwlaks4886,
sumlee,
cat2507,
byeonghu.na,
dy.kwon,
icmoon\}@kaist.ac.kr
}\\
\texttt{
\{wngml30777,
drlim1\}@snu.ac.kr
}
}

\begin{document}

\maketitle
\vspace{-0.4cm}

\begin{abstract}
Missing data frequently arises across diverse domains, including time-series and image domains. In the real world, missing occurrences often depend on the unobservable values themselves, which are referred to as Missing Not at Random (MNAR).
In this work, we introduce the Missing \textbf{P}attern \textbf{R}ecognized \textbf{D}iffusion \textbf{I}mputation \textbf{M}odel (\textbf{PRDIM}), a novel framework that explicitly captures the missing pattern and precisely imputes unobserved values. PRDIM iteratively maximizes the likelihood of the joint distribution for observed values and missing mask under an Expectation-Maximization (EM) algorithm. In this sense, we first employ a pattern recognizer, which approximates the underlying missing pattern and provides guidance during every inference toward more plausible imputations with respect to the missing information. Through extensive experiments, we demonstrate that PRDIM consistently achieves strong imputation performance under MNAR settings across multiple data modalities.
\end{abstract}

\addtocontents{toc}{\protect\setcounter{tocdepth}{0}}

\section{Introduction}
\label{main:intro}

\begin{wrapfigure}{r}{0.45\linewidth}
    \centering
    \vspace{-2.0cm}
    \includegraphics[width=1.0\linewidth]{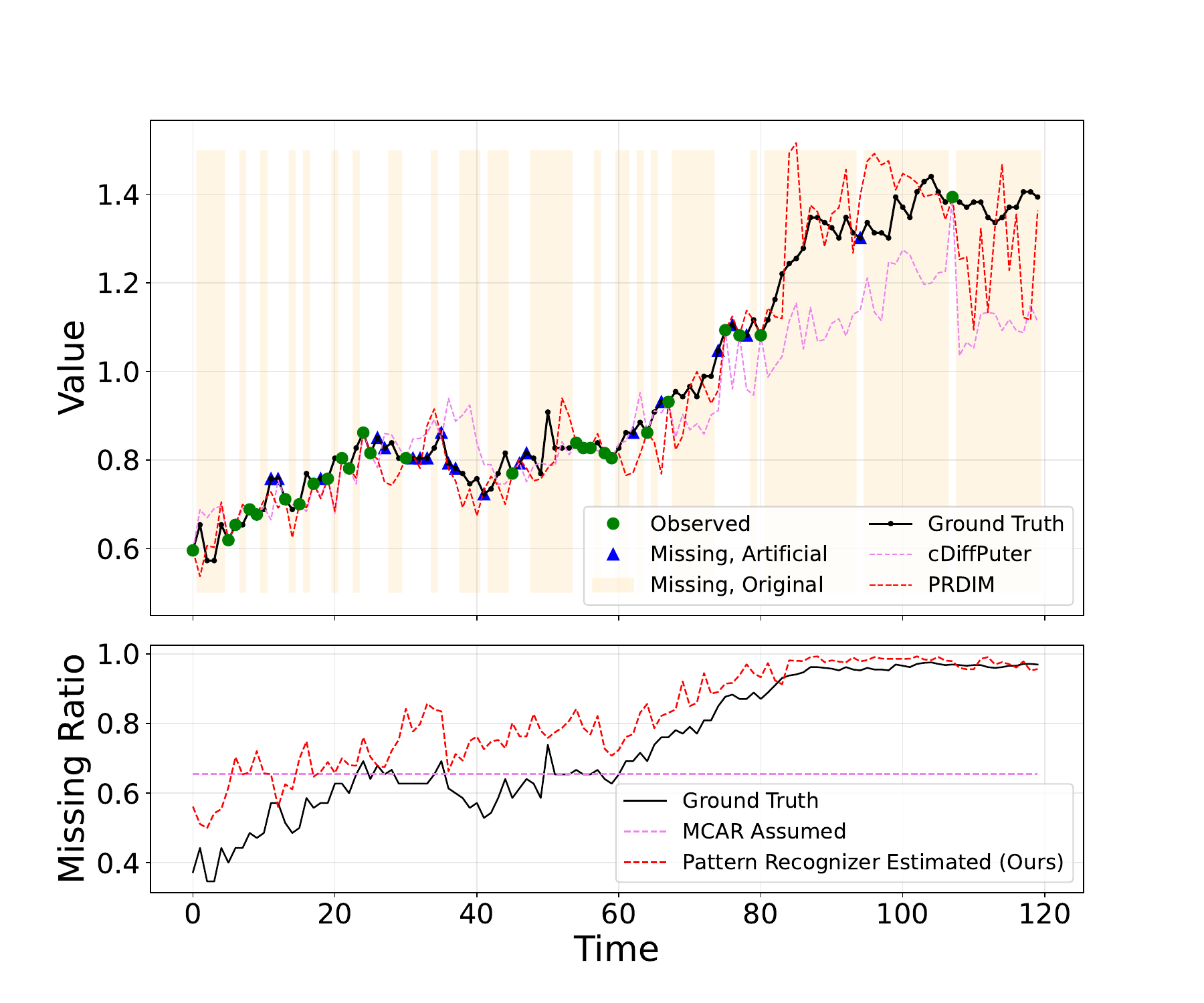}
    \vspace{-0.8cm}
    \caption{(Top) Comparison of imputation performance on original missing (\textcolor{orange}{orange region}) versus artificial missing entries (\textcolor{blue}{blue dots}) under observed values (\textcolor{darkgreen}{green dots}). (Bottom) Estimated missing ratio by pattern recognizer (\textcolor{red}{red curve}) regard to true missing ratio.}
    \label{fig:artf_orig}
    \captionof{table}{Numerical results; imputing original missing presents a more challenging task.}
    \resizebox{1.0\linewidth}{!}{
    \begin{tabular}{l|ccc|ccc}
    \toprule
     & \multicolumn{3}{c|}{Artificial missing entries} 
     & \multicolumn{3}{c}{Original missing entries} \\
    \cmidrule(lr){2-4}\cmidrule(lr){5-7}
    Method & RMSE($\downarrow$) & MAE($\downarrow$) & MRE($\downarrow$) & RMSE($\downarrow$) & MAE($\downarrow$) & MRE($\downarrow$) \\
    \midrule
    cDiffPuter~\citep{zhang2025diffputer} & 0.230 & 0.146 & 21.572 & 1.209 & 0.782 & 46.188 \\
    PRDIM       & \textbf{0.201} & \textbf{0.124} & \textbf{18.310} & \textbf{1.057} & \textbf{0.663} & \textbf{39.156} \\
    \bottomrule
    \end{tabular}}
    \label{tab:intro_tab}
    \vspace{-1.0cm}
\end{wrapfigure}

Missing data imputation aims to recover missing values from partially observed incomplete datasets, and the imputation algorithms serve as a fundamental component in many domains, including healthcare~\citep{goldberger2000physiobank}, traffic~\citep{li2017diffusion}, and image domain~\citep{xiao2017fashion}. Formally, the imputation goal is to accurately estimate missing values conditioned on observed values. Many recent imputation models assume the missing to be random or to depend on observed values, which are referred to as Missing Completely at Random (MCAR) and Missing at Random (MAR) respectively~\citep{little1987statistical,schafer1997analysis}. However, in real-world scenarios, we have missing values because of some underlying causes such as health deterioration or mortality~\citep{carreras2021missing}; so the missing tends to have patterns in its occurrences. Therefore, such patterned missing cases are referred to as \textit{Missing Not at Random} (MNAR), which is considered to be a more realistic situation. Therefore, developing imputation methods that can reliably handle MNAR mechanisms is crucial for achieving practical performance in real-world applications.

To impute high-dimensional data, recent studies have increasingly employed diffusion models~\citep{ho2020denoising, song2021scorebased}, which provide a powerful framework for capturing complex data distributions under incomplete settings. Existing methods are generally built upon conditional diffusion frameworks~\citep{tashiro2021csdi, zhou2024mtsci}, in which target entries are artificially masked and treated as missing values for training and evaluation. In such standard frameworks, we recognize potential weaknesses in the model objective, the training algorithm, and the evaluation procedure. Consequently, we first propose a diffusion-based imputation model with \textit{missing pattern} estimation to better recover missing values under the MNAR setting.

\section{Preliminaries}
\label{main:prelim}

\subsection{Missing Data Imputation for MNAR}
\label{prelim:imputaion}

As discussed earlier, accurate estimation of the missing values requires conditioning on both the observed variables $X^{\text{obs}}$ and the missing mask $M$. 
Let $X=[X_{d}] \in \mathbb{R}^{D}$ be a complete instance with $D$ dimensions, and $M \in \{0,1\}^{D}$ the missing indicator where $M_{d}=1$ if $X_{d}$ is observed, otherwise 0. 
We denote the observed and missing subsets by $X^{\text{obs}} = X \odot M$ and $X^{\text{mis}} = X \odot (\mathbf{1}-M)$, respectively.
Given access to the underlying complete data $X$ and its mask $M$, the ultimate goal is to recover the joint distribution of the observed data, missing mask, and missing data:
\begin{align}
    \max_{\theta,\phi}\mathbb{E}_{p_{\text{data}}({X},{M})}\big[\log p_{\theta,\phi}(X^{\text{obs}}, X^{\text{mis}}, M)\big]
\end{align}
Here, $\theta$ and $\phi$ are the parameters of the conditional distribution, which will be discussed in their individual roles on describing distributions of ${X}$ and ${M}|{X}$, respectively.

In the scenario of missing value imputation on incomplete data, the inference becomes maximizing the joint distribution of only two random variables ${X}^\text{obs}$ and ${M}$ because of an unobservable property of ${X}^\text{mis}$. This likelihood maximization is formulated with the expectation on ${X}^\text{mis}$, which eventually turns the problem into the Expectation-Maximization (EM) framework.
\begin{align}
    p_{\theta,\phi}({X}^{\text{obs}},{M})&=\int_{{X}^{\text{mis}}} p_{\theta,\phi}({X}^{\text{obs}},{X}^{\text{mis}},{M})d{X}^{\text{mis}}=\int_{{X}^{\text{mis}}} p_\theta(X)\,p_\phi({M}|X)d{X}^{\text{mis}}
\end{align}
Consequently, principled inference requires joint modeling of $p_{\theta}(X)$ and $p_{\phi}(M|X)$ and optimizing a suitable lower bound of the EM algorithm. Now, the focus becomes how to infer the two distributions: $p_{\theta}({X})$ and $p_{\phi}({M}|{X})$. Particularly, the inference requirement on $p_{\phi}({M}|{X})$ becomes different depending on the assumed missing mechanism across different scenarios.

\paragraph{Missing Mechanisms}

Since the distribution $p_{\theta,\phi}(X, M)$ can be decomposed as $p_{\theta}(X)p_{\phi}(M|X)$, it becomes necessary to explicitly model the generation process of the mask variable ${M}$. Under the standard taxonomy of missing data mechanisms~\citep{little1987statistical}, the missing process is characterized by the conditional distribution $p_\phi(M|X)$:
\begin{align}
\label{eq:taxonomy}
p_\phi(M|X) = 
\begin{cases} 
p_\phi(M) & \text{(MCAR)} \\
p_\phi(M|X^\text{obs}) & \text{(MAR)} \\
p_\phi(M|X^\text{obs}, X^\text{mis}) & \text{(MNAR)}
\end{cases}
\end{align}
Under MCAR/MAR, the likelihood $p_{\theta,\phi}({X}^{\text{obs}},{M})$ which is proportional to $p_\theta({X}^\text{obs})$ can be learned while ignoring missing process \citep{mattei2019miwae} with respect to the missing variable $X^{\text{mis}}$. In contrast, the mask also depends on unobserved values under MNAR; makes the missing process non-ignorable \citep{ipsen2020not}. While more realistic scenario comes from MNAR, this new requirement of inferring $p_\phi({M}|{X})$ renders the imputation models under MAR and MCAR to be ineffective and needs to be overhauled significantly.

\paragraph{Missing Model}

Some previous works have incorporated the missing mask $M$ as a supervised learning target. 
Originated from Generative Adversarial Network~\citep{goodfellow2014generative}, GAIN \citep{yoon2018gain} first proposed that a discriminator $D_\phi(X)$ can be trained to approximate $p(M|X)$, with the objective of optimal discriminator $D_{\phi^*}$ towards the specific missing ratio. 
This design ensures that, when the missing mechanism is independent of the data, the discriminator converges to a uniform prediction over missing and observed variables. 

Modified from Variational Autoencoder~\citep{kingma2013auto}, not-MIWAE \citep{ipsen2020not} extended the discriminator framework to MNAR by directly modeling $p(M|X^{\text{obs}}, X^{\text{mis}})$. Their approach demonstrated that the discriminator loss can be integrated into a variational objective, allowing optimization via minimization of an additional term in the ELBO.

Under the shared assumption adopted by GAIN and not-MIWAE, each missing value indicator follows Bernoulli distribution conditioned on the entire data. The loss of the missing model can be formulated as a binary cross-entropy (BCE) objective:
\vspace{0.1cm}
\begin{equation}
\small
\mathcal{L}(M,X,D_\phi) 
=-M^{\top}\log D_{\phi}(X) - (\mathbf{1}-M)^{\top}\log \big(\mathbf{1} - D_{\phi}(X)\big)
\end{equation}
\vspace{0.1cm}\noindent
where optimal $D_{\phi^*}(X)$ predicts the probability whether each entry would be observed or not (\textit{i.e.} $D_{\phi^*}(X)=[p(M_d=1|X)]\in\mathbb{R}^D$). 
This formulation provides a flexible way to incorporate the missing mechanism into generative imputation models. In this sense, we hereafter refer to the discriminator $D_{\phi}$ as the \textit{pattern recognizer}. Correspondingly, the loss $\mathcal{L}$ will be denoted as $\mathcal{L}_\text{PR}$.

\vspace{-0.2cm}
\subsection{Diffusion Models}
\label{prelim:diffusion}
\vspace{-0.1cm}

Diffusion models have achieved state-of-the-art generation performances in multiple domains, including vision \citep{esser2024scaling}, audio \citep{kong2020diffwave}, graphs \citep{jo2022score}, and time-series \citep{coletta2023constrained}.
They learn a data distribution by inverting a Markovian noising process. The forward process gradually corrupts a clean sample $X_0 \sim q(X_0)$ with Gaussian noise according to a prescribed variance schedule $\{\beta_t\}_{t=1}^T$:
\begin{align}
    q(X_{1:T}|X_0)\coloneqq \Pi_{t=1}^{T}q(X_t|X_{t-1})\text{ where }q(X_t|X_{t-1}) = \mathcal{N} (\sqrt{\alpha_t}X_{t-1}, (1-\alpha_t)\mathbf{I})
\end{align}
Here, $\alpha_t \coloneqq 1-\beta_t$ and $\bar{\alpha}_t \coloneqq \Pi_{s=1}^t \alpha_s$. This yields the closed form to sample $X_t$ at time $t$ given $X_0$:
\vspace{0.1cm}
\begin{align}
    q(X_t|X_0) = \mathcal{N} (\sqrt{\bar{\alpha}_t} X_0, (1-\bar{\alpha}_t) \mathbf{I})
\label{eq:gaussian_trans}    
\end{align}
\vspace{0.1cm}\noindent
The reverse process is modeled as a learned Markov chain that gradually removes noise:
\begin{align}
    p_\theta(X_{0:T})\coloneqq \Pi_{t=1}^{T} p(X_T)p_\theta(X_{t-1}|X_{t}) 
\end{align}
where $p(X_T)$ is a standard Gaussian prior, $\theta$ is the learnable parameter of a diffusion model.

Diffusion models are trained by maximizing a variational lower bound on the log-likelihood:
\vspace{0.1cm}
\begin{align}
    \mathbb{E}_{X_0 \sim q(X_0)} \left[\log p_\theta (X_0)\right]\geq \mathbb{E}_{X_0 \sim q(X_0),{X}_{1:T}\sim q({X}_{1:T}|{X_0})} \left [ \log\frac{p_\theta({{X}_{0:T}})}{q({{X}_{1:T}}|{X_0})}\right]
\end{align}
\vspace{0.1cm}\noindent
This objective could be reduced as the data reconstruction~\citep{karras2022elucidating}, the noise prediction~\citep{ho2020denoising}, or the score matching objective~\citep{song2021scorebased}.
At inference time, sampling proceeds by drawing $X_T \sim p(X_T)$ and iteratively applying the learned reverse transitions $p_\theta (X_{t-1} | X_t)$ to obtain a clean sample $X_0$.

We present the entire methodology of this work from the perspective of data reconstruction (\textit{i.e.} $X_0$ prediction). Under this view, the diffusion model is trained to reconstruct the clean data $X_0$ from noisy samples $X_t$ across timesteps. Accordingly, the diffusion loss function for a single sample $X_0$ is defined as (Refer to Equation 2 of EDM~\cite{karras2022elucidating}):
\vspace{0.1cm}
\begin{equation}
\mathcal{L}_\text{diff}(X_0, X_t, t, f_\theta) 
= \lambda(t)\| f_\theta(X_t, t) - X_0 \|_2^2
\end{equation}
\vspace{0.1cm}\noindent
where $t\sim\mathcal{U}(0,T),\,\,\lambda(t)=1$ in our experiment, $f_\theta$ is the $X_0$ prediction network parameterized by $\theta$, 
and $X_t$ follows Equation~\ref{eq:gaussian_trans}.

\vspace{-0.2cm}
\section{Methodology}
\label{main:method}
\vspace{-0.2cm}

We decompose PRDIM into two major components: (i) diffusion backbone pre-training and (ii) missing model training with joint distribution fine-tuning. Each component contains practical implementation details that stabilize training and enhance imputation performance. We first present an overview of PRDIM in Section~\ref{mtd:overview}, followed by theoretical analyses in subsequent sections.

\vspace{-0.2cm}
\subsection{Pre-imputation and EM algorithm}
\label{mtd:overview}
\vspace{-0.1cm}

\begin{figure*}[t]
\centering
    \centering
    \vspace{0.2cm}
    \includegraphics[width=0.9\linewidth]{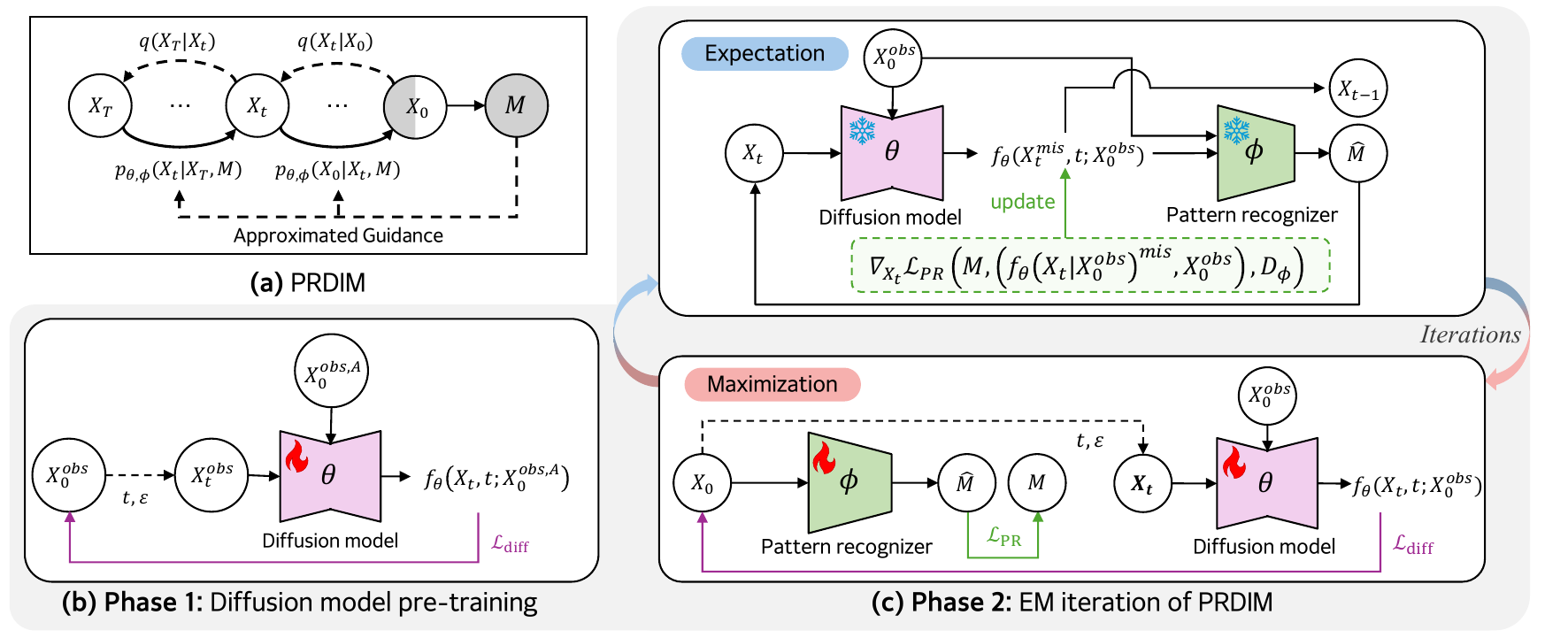}
    \caption{Overall training procedure of PRDIM; (a) Graphical model of PRDIM, (b) Diffusion model pre-training, (c) EM iteration for diffusion model fine-tuning and pattern recognizer training. In expectation step, PRDIM progressively denoising the sample with additional guidance of pattern recognizer (Top). In maximization step, diffusion model and pattern recognizer are trained independently based on $X_0^\text{obs}, X_0^\text{mis}$ which generated in the previous expectation step (Bottom).}
    \label{fig:PRDIM_main_fig}
    \vspace{-0.4cm}
\end{figure*}

Figure~\ref{fig:PRDIM_main_fig} illustrates the graphical model and overall training procedure of PRDIM. The framework consists of two complementary phases: a diffusion-based pre-imputation stage (Phase 1) and an EM iteration stage (Phase 2). The combination of these two phases enables PRDIM to learn both the data distribution and the missing pattern in a principled manner.

\vspace{-0.2cm}
\paragraph{Phase 1: Diffusion model Pre-training and Pre-imputation.}
Direct optimization of the joint distribution $p_\theta(X_0^{\text{obs}}, X_0^{\text{mis}})$ often suffers from instability and overfitting when missing entries are simply set to zero. Phase 1 of PRDIM pre-trains a diffusion backbone under a conditional diffusion framework~\cite{tashiro2021csdi} enabling the model to capture plausible data distributions from observed variables. For generic imputation tasks, we extend the diffusion target up to conditional input which is proposed at \cite{du2023saits} as Observed Reconstruction Task, which demonstrates joint distribution learning under a conditional modeling. Therefore, to train the diffusion model for pre-imputation, we need to artificially select missing entries.
Unlike CSDI~\cite{tashiro2021csdi}, which adopts artificial masking under the MCAR assumption, we introduce an adjacent target masking scheme, where artificial missing entries located near original missing values are randomly selected to exploit potential correlations.
In multivariate time-series data, the artificial missing values are chosen only across the temporal axis. In image data, they are selected as the top, bottom, left, or right neighboring missing pixels.
All subsequent experimental results are conducted under this adjacent target masking setup. (except for tabular experiment, Table~\ref{tab:tabular_imputation}.)
We denote this artificial missing mask as $A$ and define the corresponding subset $X^{\text{obs},A}\coloneqq X^\text{obs}\odot A$. Through this strategy, the diffusion backbone learns to reconstruct plausible imputations while being robust to any missing pattern. Then, the diffusion loss objective can be rewritten as follows:
\begin{equation}
\begin{aligned}
\mathcal{L}_{\text{diff}}(X_0, X_t, t, f_\theta)=\Vert \big( f_\theta(X_t,t;X_0^{\text{obs},A}) - X_0\big)\odot M \Vert_2^2
\end{aligned}
\end{equation}

where $f_\theta(X_0, X_t, t;X_0^{\text{obs},A})\odot M$ is the subset of $X_0$ prediction corresponds to observed entries at timestep $t$ and $\odot$ is element-wise product. Recent studies have shown that plausible data distributions can be effectively estimated from artificially masked data, as evidenced by \cite{he2022masked, peebles2023scalable}. Furthermore, we demonstrate the performance gain across adjacent target masking and different ratio of MCAR masking schemes in Table~\ref{tab:stock_ablation}. 

\paragraph{Phase 2: EM Iteration of PRDIM.}
After diffusion pre-training, PRDIM proceeds with the EM algorithm phase. Whereas DiffPuter~\cite{zhang2025diffputer} trains only a joint diffusion model, PRDIM trains both diffusion model $\theta$ and pattern recognizer $\phi$ simultaneously to enhance the imputation performance. 

In the maximization step, the diffusion model $\theta$ is updated to capture the full joint distribution $p_\theta(X_0^{\text{obs}}, X_0^{\text{mis}})$, while the pattern recognizer $D_\phi(X_0)$ is trained to discriminate the mask variable $M$ as a supervised learning target. 
\begin{align}
    &\mathcal{L}_\text{diff}(X_0, X_t, t, f_\theta) =\big\| f_\theta(X_t, t;X_0^{\text{obs}}) - X_0 \big\|_2^2 \\
    &\mathcal{L}_\text{PR}(M,X_0, D_\phi)=-M^{\top}\log D_{\phi}(X_0) - (\mathbf{1}-M)^{\top}\log \big(\mathbf{1} - D_{\phi}(X_0)\big) 
\end{align}
This enables the model to explicitly incorporate the missing pattern. Section \ref{sec:elbo} provides the details of maximization objective.

In the expectation step, the diffusion model generates $X^{\text{mis}}$ conditioned on $M$ and $X^{\text{obs}}$, while the pattern recognizer provides an additional approximated guidance signal that biases the generation toward imputations consistent with the estimated missing patterns. Importantly, during the early iterations, it is acceptable to use a randomly initialized pattern recognizer as guidance. Since such a recognizer has no discriminative ability, the guidance provides a degenerate signal toward a near-zero vector, yielding a neutral effect on the generation process~\citep{kim2022refining}. Corresponding experiment results are reported in Appendix~\ref{pr_analysis:intermediate_pr}.

Furthermore, while DiffPuter adopts a soft EM strategy, our framework employs a hard EM variant to enhance the exploration ability to generate $X_0^\text{mis}$ distribution, which has been theoretically justified within the prior EM literature~\cite{samdani2012unified}. (See Appendix~\ref{app:soft_hard} for a detailed description.) The EM steps under the $X_0$ prediction parameterization are summarized in Algorithm~\ref{alg:exp} and Algorithm~\ref{alg:max}.

\subsection{ELBO objective of MNAR on Diffusion Model Framework}
\label{sec:elbo}
To address the MNAR problem, we formulate the evidence lower bound (ELBO) of the log-likelihood $\log{p_{\theta,\phi}(X_0^{\text{obs}},M)}$ under a diffusion framework. 
For the true data ${X}_0 = (X_0^{\text{obs}}, X_0^{\text{mis}})$, we derive the ELBO of $\log{p_{\theta,\phi}(X_0^{\text{obs}},{M})}$ as follows:

\begin{restatable}{proposition}{propelbo}
\label{prop:elbo}
Given the graphical model in Figure~\ref{fig:PRDIM_main_fig}, assume that the forward process and missing mechanism satisfy the following conditional independencies: $q(X_{1:T}|X_0,M)=q(X_{1:T}|X_0)$ and $p_\phi(M|X_0,X_{1:T})=p_\phi(M|X_0)$.

Then, the ELBO of joint log-likelihood objective of the observed data and mask can be expressed as
\begin{align}
\log p_{\theta,\phi}(X_0^{\text{obs}}, M) &\geq\mathbb{E}_{X_{1:T}} \Big[\log \frac{p_\theta(X_T)\prod_{t=1}^{T} p_{\theta}(X_{t-1}|X_t)}{\prod_{t=1}^{T} q(X_t | X_{t-1})} \Big] \label{eq:max}+\mathbb{E}_{X_0^{\text{mis}}}\Big[\log p_{\phi}(M | X_0)\Big]\\
&+\mathbb{H}(q(X_0^\text{mis}|X_0^\text{obs},M)) \nonumber
\end{align}
\end{restatable}

where $\mathbb{H}(q(x)):=-\int_{x\sim q} q\log{q}dx$. The proof is given in Appendix~\ref{app:prf1}. This formulation explicitly incorporates the pattern recognizer into the diffusion-based imputation objective under MNAR. Unlike not-MIWAE~\cite{ipsen2020not}, the iterative forward and reverse processes inherent to diffusion models render direct ELBO optimization intractable.

Thus, we utilize the above ELBO in the EM framework, so that the generation of $X^\text{mis}$ becomes an expectation step. In the maximization step, we simultaneously train the missing model $\phi$ for the missing mask ${M}$ and the diffusion model $\theta$ that estimates the joint probability of ${X}^{\text{obs}}, {X}^{\text{mis}}$ by the Equation \ref{eq:max}.

Based on this EM formulation, alternating between the E step and M step guarantees a monotonic increase in the log-likelihood of the observed data and the missing mask, as formally stated in the following Corollary:

\begin{restatable}[EM Monotonicity]{corollary}{cormonotonic}
\label{cor:monotonic}
Under the same assumptions in Proposition~\ref{prop:elbo} with the idealized EM setting, the sequence of parameters $(\theta^{(k)}, \phi^{(k)})$ generated by the proposed EM algorithm monotonically increases the joint log-likelihood of the observed variables at each iteration $k$:
\begin{align}
    \log p_{\theta^{(k+1)}, \phi^{(k+1)}}(X_0^{\text{obs}}, M) \geq \log p_{\theta^{(k)}, \phi^{(k)}}(X_0^{\text{obs}}, M)
\end{align}
\end{restatable}

The theoretical justification follows directly from the monotonicity theorem of EM~\cite{dempster1977maximum} (see Theorem 1 of Section 3). We provide detailed proof in Appendix~\ref{app:cor1}.

\subsection{Find Best $X_0^{\text{mis}}$ with Approximated Guidance of Pattern Recognizer}
\label{sec:estep}

After training both the missing model and the diffusion model, the expectation step allows us to replace $q(X_0^{\text{mis}} | X_0^{\text{obs}}, M)$ with the parameters $\theta,\phi$ that were optimized in the preceding maximization step. Since the underlying generative process is score-based, the gradient term $\nabla_{X_t}\log p_{\theta,\phi}(X_t|X_0^{\text{obs}},M)$ can be decomposed into two components: the score function term $\nabla_{X_t} \log p_\theta(X_t | X_0^{\text{obs}})$ corresponding to the joint distribution, and the pattern recognizer guidance term $\nabla_{X_t} \log p_\phi(M|X_0^{\text{obs}}, X_t^{\text{mis}})$ reflecting the missing pattern. We further show that the pattern recognizer guidance term can be approximated with the pattern recognizer $D_\phi$ according to Proposition \ref{prop:appx_guide}.

\begin{figure}[t]
    \centering
    \begin{minipage}{0.48\linewidth}
        \vspace{-1.0cm}
        \begin{algorithm}[H]
            \caption{PRDIM E step}
            \label{alg:exp}
            \begin{algorithmic}[1]
                \STATE Diffusion model $\theta$, pattern recognizer $\phi$, observed data $X_0^\text{obs}$, and mask data $M$.
                \STATE Sample $X_T\sim\mathcal{N}(0,\mathbf{I})$
                \FOR{$t = T, \dots, 1$}
                    \STATE Get $\hat{X}_0 =f_\theta(X_t,t;X_0^\text{obs})$ 
                    \STATE $\hat{X}_0\leftarrow\hat{X}_0\odot (\mathbf{1}-M)+X_0^{\text{obs}}\odot M$
                    \STATE $\hat{X}_0\leftarrow\hat{X}_0-\frac{1-\bar{\alpha}_t}{\sqrt{\bar{\alpha}_t}}\nabla_{X_t}\mathcal{L}_\text{PR}(M,\hat{X}_0, D_\phi)$
                    \STATE Sample $\varepsilon\sim\mathcal{N}(0,\mathbf{I})$
                    \STATE $X_{t-1}=\sqrt{\bar\alpha_{t-1}}\hat{X}_0+\sqrt{1-\bar\alpha_{t-1}}\varepsilon$
                \ENDFOR
                \STATE \textbf{return} $X_0$
            \end{algorithmic}
        \end{algorithm}
        \vspace{-1.0cm}
    \end{minipage}
    \hfill
    \begin{minipage}{0.48\linewidth}
        \vspace{-1.0cm}
        \begin{algorithm}[H]
            \caption{PRDIM M step}
            \label{alg:max}
            \begin{algorithmic}[1]
                \STATE Diffusion model $\theta$, pattern recognizer $\phi$, corresponding learning rate $\eta_\theta, \eta_\phi$, imputed data $X_0$, mask data $M$, and maximization epoch $N_m$.
                \FOR{$i = 1, \dots, N_m$}
                    \STATE Sample $t,\varepsilon\sim \mathcal{U}(0,T),\, \mathcal{N}(0,\mathbf{I})$
                    \STATE $X_t=\sqrt{\bar\alpha_t}X+\sqrt{1-\bar\alpha}_t\varepsilon$
                    \STATE $\theta\leftarrow\theta-\eta_\theta\nabla_\theta \mathcal{L}_\text{diff}(X_0, X_t,t,f_\theta)$
                    \STATE $\phi\leftarrow\phi-\eta_\phi\nabla_\phi\mathcal{L}_\text{PR}(M,X_0,D_\phi)$
                \ENDFOR
                \STATE \textbf{return} $\theta,\,\phi$
            \end{algorithmic}
        \end{algorithm}
        \vspace{-1.0cm}
    \end{minipage}
\end{figure}

\begin{restatable}{proposition}{propguide}
\label{prop:appx_guide}
Suppose that the pattern recognizer $D_\phi$ is optimal, satisfying $D_{\phi^*}(X_0)=[p_\phi(M_d|X_0)]\in\mathbb{R}^D$, the score function of the joint log-likelihood with respect to the missing mask can be approximated as
\begin{align}
\nabla_{X_t} \log p_{\theta,\phi}(X_t | X_0^{\text{obs}}, M) \simeq \nabla_{X_t} \log p_{\theta}(X_t | X_0^{\text{obs}}) - \nabla_{X_t}  \mathcal{L}_\text{PR}\Big(M, \hat{X}_0, D_{\phi^*}\Big)
\end{align}

where $\hat{X}_0\coloneqq f_\theta(X_t,t;X_0^{\text{obs}})\odot(\mathbf{1}-M)+ X_0^{\text{obs}}\odot M$.
\end{restatable}
\vspace{-0.2cm}
Here, $f_\theta(X_t,t;X_0^\text{obs})$ is $X_0$ prediction at timestep $t$, and $\hat{X}_0$ represents the estimated $X_0$ in the intermediate timestep $t$. The detailed proof is given in Appendix~\ref{app:prf2}. Proposition~\ref{prop:appx_guide} is noteworthy because it steers the more informative gradient of the intermediate sample ${X}_{t}$ with respect to the estimated missing pattern ${D}_{\phi}$; it provides a meaningful signal according to the negative missing probability $\mathcal{L}_\text{PR}$. Further analysis of the practical behavior of the pattern recognizer is provided in Appendix~\ref{app:pr_analysis}.

To summarize, within the DDPM framework, given $X_t$ at time step $t$, the variable $X_{t-1}$ can be obtained through the following three steps.
First, using the diffusion model together with Tweedie’s formula \citep{carlin2000bayes, efron2011tweedie}, we can compute the posterior mean $\hat{X}_0$ \cite{ho2020denoising} only with the diffusion parameter $\theta$:
\begin{align}
    \hat{X}_0=f_\theta(X_t,t;X_0^\text{obs})\odot(\mathbf{1}-M)+X_0^{\text{obs}}\odot M \label{eq:posterior_mean}
\end{align}
Second, the pattern recognizer evaluates the missing probability for each entry based on the estimated missing values and the observed variables. Finally, according to Proposition~\ref{prop:appx_guide}, the denoised sample $X_{t-1}=(X_{t-1}^\text{obs},X_{t-1}^\text{mis})$ at time $t-1$ is updated by incorporating the approximated guidance.
\begin{equation}
    X_{t-1}=\sqrt{\bar\alpha_{t-1}}(\hat{X}_0-\frac{1-\bar{\alpha}_t}{\sqrt{\bar{\alpha}_t}}\nabla_{X_t}\mathcal{L}_\text{PR}(M,\hat{X}_0, D_\phi))+\sqrt{1-\bar\alpha_{t-1}}\varepsilon \,\,\text{where}\,\,\varepsilon\sim\mathcal{N}(0,\mathbf{I})
\end{equation}

\vspace{-0.2cm}
\section{Experiments}
\label{main:experiments}
\vspace{-0.2cm}

\begin{table*}[t]
\centering
\caption{Overall \textbf{MAE ($\downarrow$)} performance on three benchmark datasets. We report mean $\pm$ std over 5 runs according to each methodology. Best results are in \textbf{bold}, and second best results are in \underline{underline}.}
\label{tab:main_tab}
\resizebox{0.85\linewidth}{!}{
\begin{tabular}{lccc|ccc}
\toprule
Method 
& \multicolumn{3}{c|}{Original / Out-of-Sample} 
& \multicolumn{3}{c}{Original / In-Sample} \\
\cmidrule(lr){2-4} \cmidrule(lr){5-7}
& ETT & STOCK & PEMS-Bay & ETT & STOCK & PEMS-Bay \\
\midrule
Mean         & 2.034{\scriptsize$\pm$0.000} & 1.949{\scriptsize$\pm$0.000} & 0.813{\scriptsize$\pm$0.000} & 1.486{\scriptsize$\pm$0.000} & 2.039{\scriptsize$\pm$0.000} & 0.789{\scriptsize$\pm$0.000} \\
\midrule
\multicolumn{7}{l}{\textit{Discriminative models}} \\
TimesNet      & 1.044{\scriptsize$\pm$0.065} & 1.111{\scriptsize$\pm$0.073} & 0.291{\scriptsize$\pm$0.007} & 1.154{\scriptsize$\pm$0.068} & 1.221{\scriptsize$\pm$0.077} & 0.225{\scriptsize$\pm$0.001} \\
TimeMixer++   & 1.642{\scriptsize$\pm$0.025} & 1.287{\scriptsize$\pm$0.239} & 0.579{\scriptsize$\pm$0.018} & 1.100{\scriptsize$\pm$0.032} & 1.369{\scriptsize$\pm$0.260} & 0.557{\scriptsize$\pm$0.020} \\
BRITS         & 0.992{\scriptsize$\pm$0.037} & 0.627{\scriptsize$\pm$0.010} & 0.278{\scriptsize$\pm$0.006} & 0.491{\scriptsize$\pm$0.008} & 0.701{\scriptsize$\pm$0.010} & 0.182{\scriptsize$\pm$0.003} \\
SAITS         & 0.814{\scriptsize$\pm$0.046} & 0.442{\scriptsize$\pm$0.022} & 0.302{\scriptsize$\pm$0.009} & 0.366{\scriptsize$\pm$0.014} & 0.498{\scriptsize$\pm$0.025} & 0.212{\scriptsize$\pm$0.003} \\
\midrule
\multicolumn{7}{l}{\textit{Generative models}} \\
GP-VAE         & 1.511{\scriptsize$\pm$0.011} & 0.902{\scriptsize$\pm$0.109} & 0.345{\scriptsize$\pm$0.001} & 0.896{\scriptsize$\pm$0.018} & 1.010{\scriptsize$\pm$0.118} & 0.292{\scriptsize$\pm$0.002} \\
not-MIWAE      & 1.311{\scriptsize$\pm$0.016} & 0.681{\scriptsize$\pm$0.045} & 0.396{\scriptsize$\pm$0.005} & 0.637{\scriptsize$\pm$0.011} & 0.759{\scriptsize$\pm$0.039} & 0.352{\scriptsize$\pm$0.005} \\
\midrule
\multicolumn{7}{l}{\textit{Diffusion-based models}} \\
CSDI          & 1.071{\scriptsize$\pm$0.001} & 0.641{\scriptsize$\pm$0.000} & \underline{0.177}{\scriptsize$\pm$0.000} & 0.522{\scriptsize$\pm$0.001} & 0.710{\scriptsize$\pm$0.000} & \underline{0.158}{\scriptsize$\pm$0.000} \\
MTSCI         & 0.957{\scriptsize$\pm$0.001} & 0.736{\scriptsize$\pm$0.001} & 0.193{\scriptsize$\pm$0.000} & 0.500{\scriptsize$\pm$0.000} & 0.809{\scriptsize$\pm$0.001} & 0.179{\scriptsize$\pm$0.000} \\
cDiffPuter    & \underline{0.782}{\scriptsize$\pm$0.000} & \underline{0.406}{\scriptsize$\pm$0.000} & 0.182{\scriptsize$\pm$0.000} & \underline{0.362}{\scriptsize$\pm$0.000} & \underline{0.450}{\scriptsize$\pm$0.000} & 0.168{\scriptsize$\pm$0.000} \\
\midrule
\textbf{PRDIM} 
            & \textbf{0.663}{\scriptsize$\pm$0.000} & \textbf{0.254}{\scriptsize$\pm$0.000} & \textbf{0.170}{\scriptsize$\pm$0.000}
            & \textbf{0.303}{\scriptsize$\pm$0.000} & \textbf{0.275}{\scriptsize$\pm$0.000} & \textbf{0.154}{\scriptsize$\pm$0.000} \\
\bottomrule
\end{tabular}}
\vspace{-0.3cm}
\end{table*}

In this section, we evaluate PRDIM on various datasets under the MNAR setting. The results show that PRDIM achieves consistent improvements over prior methods, confirming the effectiveness of incorporating the missing model into the diffusion framework. The detailed experimental settings are provided in the Appendix~\ref{app:exp_details}.

\vspace{-0.25cm}
\subsection{Revisiting the Objective of Diffusion Imputation Models under MNAR}
\label{subsec:revisiting}
\vspace{-0.1cm}

Primarily, we clarify how the target objective adopted in our work differs in perspective from those used in existing diffusion imputation models, and we justify the validity of our proposed objective.
In this work, we refer to the missing entries inherent to the dataset, which constitute the target for imputation as \textit{original} missing values, where $M$ denotes the corresponding indicator mask variable. In some cases, the underlying values of $M$ are often inaccessible such as PhysioNet~\cite{goldberger2000physiobank} and AirQuality~\cite{zhang2017cautionary}. Therefore, for quantitative evaluation, we simulate this setting by applying specific missing mechanisms (MCAR, MAR, and MNAR) to real-world complete datasets, thereby generating incomplete distributions ($X_0^{\text{obs}},M$).

In many prior works~\citep{tashiro2021csdi, alcaraz2022diffusion, zhou2024mtsci, liu2024self}, it is common practice to impose an additional mask on $X_0^{\text{obs}}$ during training. Let $M-A$ denotes indicator mask variable of \textit{artificially} missing entries, and $X_0^{\text{obs},A}$ indicates the subset of observed entries remaining after the additional \textit{artificial} masking is applied (\textit{i.e.}, $X_0^{\text{obs},A}\coloneqq X_0^{\text{obs}} \odot A$). These artificially masked entries supply the supervision required for both training and evaluation. We can formulate these works aim to obtain $p_\theta(X_0 \mid X_0^{\text{obs},A}, A)$ and evaluation follows the same protocol by inserting an artificial mask into the test data. 

Considering the above current practice of imputation studies, our research question emerges from a fundamentally different viewpoint. If the distribution of the original missing mask $M$ differs substantially from that of the artificial mask $A$, we conjecture that a model trained under $p_\theta(X_0 \mid X_0^{\text{obs},A}, A)$ is not expected to perform well under $p_\theta(X_0 \mid X_0^{\text{obs}}, M)$. We provide a more detailed description in Appendix~\ref{app:detail_data}.

This distinction is crucial to the motivation of our work. PRDIM introduces a pattern recognizer, which explicitly learns the missing pattern, and it enables us to model $p_{\theta,\phi}(X_0 \mid X_0^{\text{obs}}, M)$ with regard to underlying missing pattern. We hypothesize that this distributional discrepancy is most pronounced under MNAR settings. We empirically validate this hypothesis in the subsequent experiments; and we also provide the performance by following the practices from the past work, as well.

\vspace{-0.2cm}
\subsection{Experimental Setting}
\vspace{-0.2cm}

\paragraph{Datasets}
We evaluate our method on widely used multivariate time-series datasets, image datasets, and tabular datasets:  
(1) \textbf{ETT} \citep{zhou2021informer}, which records load and temperature of electricity transformers and has been a standard benchmark in time-series imputation task; (2) \textbf{STOCK}\footnote{\url{http://github.com/Y-debug-sys/Diffusion-TS}}, which contains historical daily stock prices and reflects complex temporal dynamics; (3) \textbf{PEMS-Bay} \citep{li2017diffusion}, which consists of road occupancy rates collected from highway sensors, exhibiting spatial-temporal patterns; (4) \textbf{Fashion-MNIST} (FMNIST) \citep{xiao2017fashion}, which consists of gray-scale images of clothing items across 10 categories, (5) \textbf{CelebA-HQ} \citep{CelebAMask-HQ}, which represents RGB-based human face images, and (6) five different tabular datasets are from UCI machine Learning Repository.\footnote{\url{archive.ics.uci.edu/}} We provide a detailed description of all of these datasets in Appendix~\ref{exp_details:dataset}.

\vspace{-0.2cm}
\paragraph{MNAR mechanism} Since the datasets are complete, we arbitrarily generate missing under MNAR mechanisms. As discussed in the section~\ref{subsec:revisiting}, all baseline models are trained on the resulting incomplete data, while evaluation is performed on the imputation of unobserved ground-truth values. For example, we adopt the missing mechanism from not-MIWAE~\cite{ipsen2020not} to validate the imputation performance of PRDIM in Table~\ref{tab:main_tab} and Figure~\ref{fig:app_fmnist_mnar}. In addition, detailed description of the employed missing mechanisms on time-series, image, and tabular datasets are provided in the Appendix~\ref{exp_details:dataset}, \ref{app:add_exp-celeba}, and \ref{app:exp_tabular} respectively.

\vspace{-0.2cm}
\paragraph{Baselines}
We compare PRDIM against 10 representative imputation methods on time-series datasets. \textbf{Mean} serves as a traditional statistical baseline. \textbf{TimesNet} \citep{wu2022timesnet}, \textbf{TimeMixer++} \citep{wang2024timemixer++}, \textbf{BRITS} \citep{cao2018brits}, and \textbf{SAITS} \citep{du2023saits} are discriminative models, known to achieve strong performance for time-series imputations. \textbf{GP-VAE}~\citep{fortuin2020gp} and \textbf{not-MIWAE}~\citep{ipsen2020not} are VAE-based imputation model that could be implemented on incomplete data.  

Among diffusion-based methods, we reproduce \textbf{CSDI}~\citep{tashiro2021csdi}, \textbf{MTSCI}~\citep{zhou2024mtsci}, and \textbf{DiffPuter}~\citep{zhang2025diffputer} which shows robust performance on various datasets. Specifically, We modified DiffPuter into conditional diffusion framework, which denote as \textbf{cDiffPuter}. In image domain, we additionally reproduce \textbf{misGAN}~\citep{li2019misgan} and \textbf{MCFlow}~\citep{richardson2020mcflow} which are GAN, Flow based generative imputation model respectively. Further explanations of baselines can be found in Appendix~\ref{app:rel_works} and hyperparameter settings are reported in Appendix~\ref{exp_details:baselines}.

\vspace{-0.2cm}
\paragraph{Evaluation Metrics}
We consider two types of evaluation: (i) \textit{in-sample} imputation and \textit{out-of-sample} imputation, both targeting the recovery of original missing values (Detailed description in Appendix~\ref{app:detail_data}). As highlighted in the introduction, our main goal is to impute original missing entries, as opposed to artificially masked ones. We report the three error-based metrics (i) RMSE, (ii) MAE, and (iii) MRE which are defined in Appendix~\ref{app:eval_metric}, throughout this section.

\vspace{-0.2cm}
\subsection{Overall Performance}
\label{overall_performance}

\begin{wraptable}{R}{0.4\linewidth}
    \centering
    \vspace{-0.4cm}
    \caption{MAE ($\downarrow$) performance for artificial missing entries ($M-A$) across diffusion imputation models on 5 runs.}
    \label{tab:artf_mae}
    \scriptsize 
    \setlength{\tabcolsep}{3pt}
    \begin{tabular}{l|ccc}
    \toprule
    Method & \multicolumn{3}{c}{Artificial / Out-of-Sample} \\ \cmidrule(lr){2-4} 
     & ETT & STOCK & PEMS-Bay \\ \cmidrule(lr){1-4} 
    CSDI & 0.243{\scriptsize$\pm$0.001} & 0.117{\scriptsize$\pm$0.000} & 0.115{\scriptsize$\pm$0.000} \\ 
    MTSCI & 0.208{\scriptsize$\pm$0.001} & 0.165{\scriptsize$\pm$0.002} & 0.123{\scriptsize$\pm$0.000} \\ 
    cDiffPuter & 0.146{\scriptsize$\pm$0.001} & 0.101{\scriptsize$\pm$0.000} & 0.117{\scriptsize$\pm$0.000} \\ 
    PRDIM & \textbf{0.124}{\scriptsize$\pm$0.001} & \textbf{0.092}{\scriptsize$\pm$0.002} & \textbf{0.114}{\scriptsize$\pm$0.001} \\
    \bottomrule
    \end{tabular}
    \vspace{-0.4cm}
\end{wraptable}

\paragraph{Time-Series Dataset} We first evaluate PRDIM against representative discriminative, generative, and diffusion-based baselines on multivariate time-series datasets. As reported in Table~\ref{tab:main_tab} and detailed in the Appendix~\ref{app:add_exp}, PRDIM consistently achieves the best performance across all metrics and datasets. In particular, the gains are most pronounced on out-of-sample imputation tasks, while in-sample results remain competitive, suggesting that PRDIM generalizes well to unseen missing values. We further evaluate the performance on artificial missing entries to align with the metrics used in prior studies, as shown in Table~\ref{tab:artf_mae}. A comparison with Table~\ref{tab:main_tab} confirms a notable performance gap between imputing original and artificial missing entries, supporting our claim that artificial masking may not fully reflect the challenges of MNAR scenarios. Nevertheless, PRDIM consistently achieves strong results even on artificial entries, ensuring its effectiveness regardless of the masking scheme.

\vspace{-0.3cm}
\paragraph{Image Dataset} Due to space constraints, experiment results on FMNIST are provided in Figures~\ref{fig:app_fmnist_mar} and~\ref{fig:app_fmnist_mnar} of Appendix~\ref{app:fmnist}. These results illustrate that PRDIM generates semantically more consistent reconstructions than other approaches. They also highlight the importance of explicitly modeling the missing mechanism in the diffusion process.

\begin{figure}[t]
    \centering
    \begin{minipage}{0.42\linewidth}
        \centering
        \includegraphics[width=\linewidth]{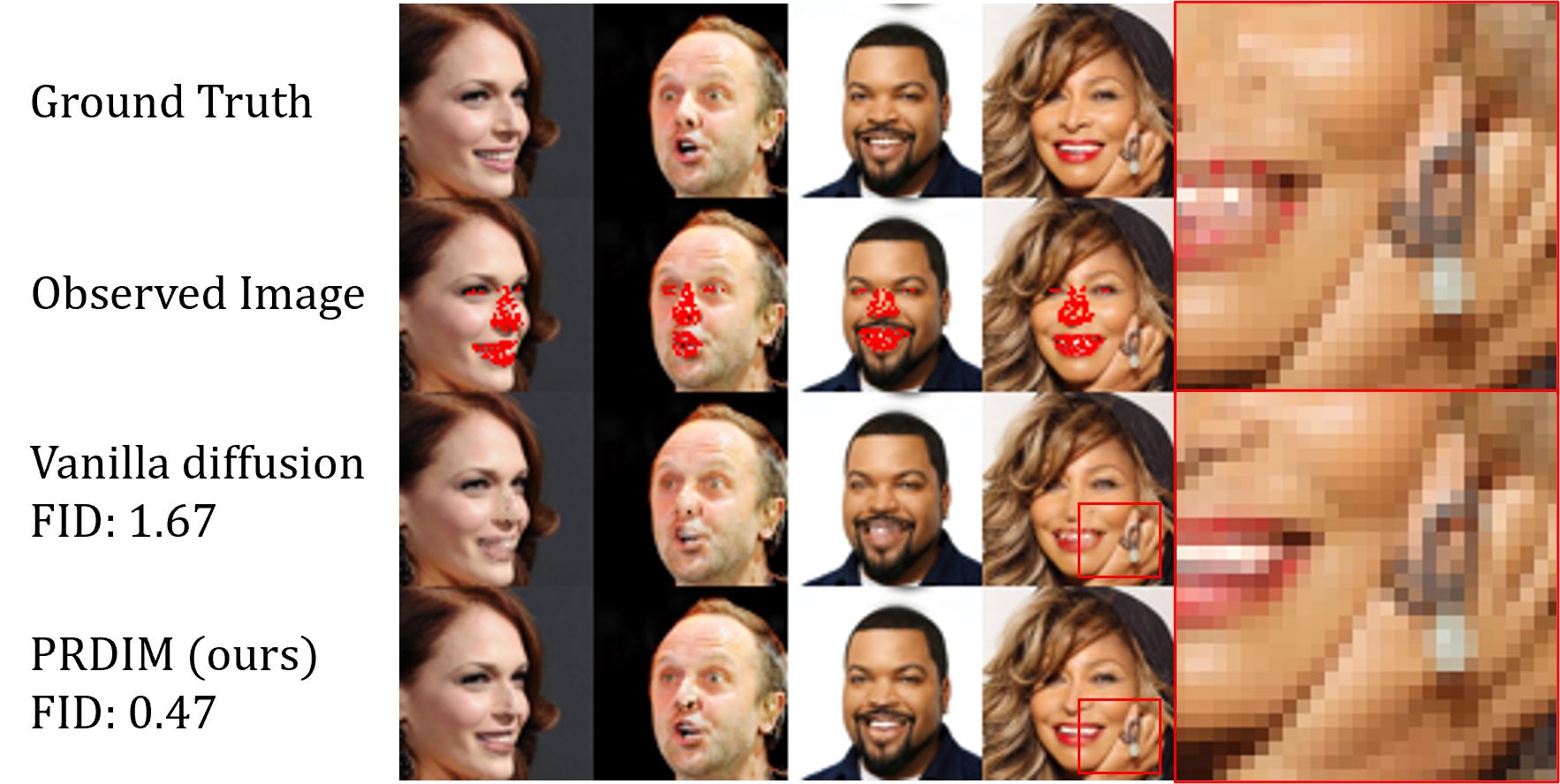}
        \caption{Qualitative imputation results on CelebA-HQ 64. Second row shows observed inputs, where red pixels indicate missing. The detailed description of data processing is written in Appendix~\ref{app:add_exp-celeba}.}
        \label{fig:main-celeba}
        \vspace{-0.5cm}
    \end{minipage}
    \hfill
    \begin{minipage}{0.54\linewidth}
        \centering
        \captionof{table}{Quantitative results on tabular datasets. We report the average MAE($\downarrow$) on 5 runs and total ranking across the datasets for in-sample \textbf{MNAR} imputation. Baseline details are in Appendix~\ref{exp_details:baselines}. \textbf{Bold} denotes the best result.}
        \resizebox{1.0\linewidth}{!}{
        \begin{tabular}{l|ccccc|c}
        \toprule
        Data & \textbf{Adult} & \textbf{Bean} & \textbf{Default} & \textbf{Gesture} & \textbf{Magic} & \textbf{Avg Rank} \\
        \midrule
        MissForest~\cite{stekhoven2012missforest} & 0.609 & 0.264 & 0.370 & 0.389 & 0.518 & 4.4 \\
        MICE~\cite{van2011mice} & 0.969 & 0.203 & 0.580 & 0.662 & 0.766 & 6.6 \\
        MOT~\cite{muzellec2020missing} & 0.506 & 0.255 & 0.354 & 0.438 & 0.494 & 4.0 \\
        TDM~\cite{zhao2023transformed} & 0.523 & 0.216 & 0.374 & 0.409 & \textbf{0.485} & 3.7\\
        TabCSDI~\cite{zheng2022diffusion} & 0.631 & 0.764 & 0.572 & 0.562 & 0.770 & 7.4\\
        Hyperimpute~\cite{jarrett2022hyperimpute} & 0.533 & 0.284 & 0.345 & \textbf{0.386} & 0.588 & 4.2 \\
        DiffPuter~\cite{zhang2025diffputer} & 0.497 & 0.240 & 0.374 & 0.391 & 0.539 & 3.9 \\
        \textbf{PRDIM} & \textbf{0.474} & \textbf{0.199} & \textbf{0.336} & 0.394 & 0.490 & \textbf{1.8}\\
        \bottomrule
        \end{tabular}}
        \label{tab:tabular_imputation}
        \vspace{-0.5cm}
    \end{minipage}
\end{figure}

\vspace{-0.1cm}
We extend this analysis to the more complicate RGB image space, CelebA-HQ in Figure~\ref{fig:main-celeba}.\footnote{Our intention is only to demonstrate the feasibility of applying PRDIM to general datasets in diverse modalities; we do not intend to compare our method to the \textit{inpainting} methods which utilize the pre-trained generative models since we use diffusion models trained with incomplete data.} While the vanilla model struggles to infer missing semantic components such as eyes, noses, and mouths and tends to merely fill the regions with averaged color derived from the global context, PRDIM effectively leverages local contextual cues. We reproduce additional quantitative metrics in Appendix~\ref{app:add_exp-celeba}. 


\vspace{-0.3cm}
\paragraph{Tabular Dataset} To verify the effectiveness of PRDIM's pattern recognizer on tabular data, we evaluate our model using the identical experimental setup as DiffPuter, incorporating the proposed pattern recognizer. The results on Table~\ref{tab:tabular_imputation} confirm that PRDIM outperforms DiffPuter on tabular datasets as well. Comprehensive results on out-of-sample are provided in Appendix~\ref{app:exp_tabular}.

\vspace{-0.2cm}
\paragraph{Computational Overhead}
PRDIM incurs additional inference cost due to gradient-based pattern guidance during inference, but the overhead remains moderate relative to the diffusion baselines. As further analyzed in Appendix~\ref{app:cost-perform}, the pattern recognizer is lightweight across modalities, and this additional cost is justified by consistent performance gains under MNAR settings.

\vspace{-0.3cm}
\subsection{Ablation Studies}
\label{abl_exp}
\vspace{-0.2cm}

We further conduct ablation experiments to quantify the contribution of each component in PRDIM. Table \ref{tab:stock_ablation} reports results when either the pattern recognizer is removed or hard EM is replaced with soft EM. Both modifications lead to a significant performance drop, indicating that explicit missing modeling and iterative EM updates are indispensable for exploring the missing data distribution.
Furthermore, we investigate the impact of the artificial missing mask $A$. While previous works typically generate artificial missing entries under the MCAR ($M-A$ is the corresponding indicator mask of artificial missing entries), we report performances across different missing rates of $10\%$, $50\%$, and $90\%$ for $M-A$. Overall, the missing rate of $M-A$ has limited influence on imputation.

Table~\ref{tab:add_missing} investigate robustness under different missing mechanisms by applying MNAR and MCAR masks to the ETT dataset. Implementation details for the corresponding mechanisms are described in Appendix~\ref{app:exp_details}. PRDIM consistently outperforms baselines under MNAR, whereas under MCAR the advantage diminishes, as the pattern recognizer learns randomness in this scenario. Together, these ablation studies confirm the necessity of PRDIM’s design choices and its robustness across varying missing conditions.

\begin{table}[h]
    \centering
    \begin{minipage}[t]{0.51\linewidth}
    \vspace{-0.4cm}
        \centering
        \caption{Ablation study results on the \textbf{STOCK} dataset. X$\%$ denotes the missing rate of $M-A$ in Phase 1 pre-training, and the last line is same with DiffPuter.}
        \label{tab:stock_ablation}
        \resizebox{1.0\linewidth}{!}{
        \begin{tabular}{lccc|ccc}
        \toprule
        Method & \multicolumn{3}{c|}{Out-of-Sample} & \multicolumn{3}{c}{In-Sample} \\
        \cmidrule(lr){2-4} \cmidrule(lr){5-7}
        & RMSE($\downarrow$) & MAE($\downarrow$) & MRE($\downarrow$) & RMSE($\downarrow$) & MAE($\downarrow$) & MRE($\downarrow$) \\
        \midrule
        PRDIM & 0.599 & \textbf{0.254} & \textbf{16.79} & 0.633 & \textbf{0.275} & \textbf{17.15} \\
        \midrule
        10\% & \textbf{0.590} & 0.258 & 17.06 & \textbf{0.630} & 0.282 & 17.63 \\
        50\% & 0.624 & 0.259 & 17.12 & 0.667 & 0.281 & 17.56 \\
        90\% & 0.631 & 0.293 & 19.40 & 0.677 & 0.323 & 20.15 \\
        w/o PR & 0.650 & 0.306 & 20.23 & 0.691 & 0.339 & 21.17 \\
        \makecell[l]{w/o PR + \\ w/o hard EM}& 0.734 & 0.406 & 26.88 & 0.778 & 0.450 & 28.06 \\
        \bottomrule
        \end{tabular}}
        \vspace{-0.3cm}
    \end{minipage}
    \hfill
    \begin{minipage}[t]{0.45\linewidth}
        \vspace{-0.4cm}
        \centering
        \caption{Imputation results on the \textbf{ETT} dataset in (top) different \textbf{MNAR} and (bottom) \textbf{MCAR} pattern. We report the average RMSE / MAE / MRE over 5 runs.}
        \label{tab:add_missing}
        \resizebox{1.0\columnwidth}{!}{
        \begin{tabular}{l|c|c}
        \toprule
        Method & Out-of-Sample & In-Sample \\
        \midrule
        CSDI       & 0.345 / 0.218 / 13.825 & 0.233 / 0.159 / 13.259 \\
        cDiffPuter & \textbf{0.264} / 0.177 / 11.168 & 0.281 / 0.178 / 14.836 \\
        PRDIM      & 0.282 / \textbf{0.171} / \textbf{10.837} & \textbf{0.197} / \textbf{0.130} / \textbf{10.816} \\
        \bottomrule
        \end{tabular}}
        \vspace{0.05cm}

        \resizebox{1.0\columnwidth}{!}{
        \begin{tabular}{l|c|c}
        \toprule
        CSDI       & 0.237 / 0.160 / 15.550 & 0.182 / 0.128 / 16.752 \\
        cDiffPuter & 0.251 / 0.162 / 15.737 & 0.202 / 0.136 / 17.795 \\
        PRDIM      & \textbf{0.225} / \textbf{0.147} / \textbf{14.267} & \textbf{0.172} / \textbf{0.118} / \textbf{15.485} \\
        \bottomrule
        \end{tabular}}
        \vspace{-0.3cm}
    \end{minipage}
\end{table}

\vspace{-0.3cm}
\subsection{Applications of PRDIM}
\vspace{-0.2cm}

\begin{wraptable}{R}{0.35\linewidth}
    \centering
    \vspace{-0.5cm}
    \begin{minipage}[t]{1.0\linewidth}
        \centering
        \caption{Out-of-Sample MAE ($\downarrow$) under different Phase 1 methods.}
        \label{tab:different_phase1}
        \resizebox{1.0\linewidth}{!}{
        \begin{tabular}{l|ccc}
        \toprule
        Initialization Method 
        & ETT & STOCK & PEMS-Bay \\
        \midrule
        MEAN & 0.766 & 0.326 & 0.207 \\
        BRITS & 0.824 & 0.300 & 0.188 \\
        SAITS & 0.774 & 0.307 & 0.180 \\
        \textbf{CSDI (PRDIM)} 
        & \textbf{0.663} & \textbf{0.254} & \textbf{0.170} \\
        \bottomrule
        \end{tabular}}    
    \end{minipage}

    \vspace{0.2cm}
    \begin{minipage}[t]{1.0\linewidth}
        \centering
        \caption{Post-imputation classification results on FMNIST.}
        \label{tab:post_imputation}
        \resizebox{1.0\linewidth}{!}{
        \begin{tabular}{lc}
        \toprule
        Method & Clean test data accuracy (\%) \\
        \midrule
        Clean Data & 92.59 \small{$\pm$0.09} \\
        \midrule
        PRDIM & \textbf{91.14} \small{$\pm$0.33} \\
        DiffPuter & 91.09 \small{$\pm$0.27} \\
        MCFlow & 90.49 \small{$\pm$0.25} \\
        CSDI & 87.41 \small{$\pm$0.87} \\
        misGAN & 84.34 \small{$\pm$1.48} \\
        \bottomrule
        \end{tabular}}
        \vspace{-0.5cm}
    \end{minipage}
\end{wraptable}

We conducted two additional experiments to investigate the importance of Phase 1 initialization and the generalization capability of PRDIM.
First, we replaces Phase 1 with several different strategies (MEAN, BRITS, SAITS) under identical conditions to examine how different pre-imputation methods influence PRDIM’s performance. As shown in Table~\ref{tab:different_phase1}, initialization with CSDI yields the best results.

Second, we assessed post-imputation classification accuracy on the FMNIST dataset to evaluate how well the imputed data support downstream tasks. This experiment examines whether higher-quality imputations translate into improved task performance. As reported in Table~\ref{tab:post_imputation}, PRDIM surpasses Flow-based, GAN-based, and prior diffusion-based approaches, indicating that PRDIM is capable of restoring meaningful semantic information even in the image domain.


\vspace{-0.3cm}
\section{Conclusion}
\label{main:conclusion}
\vspace{-0.3cm}

We presented PRDIM, a diffusion-based imputation framework that incorporates an additional discriminator denoted \textit{pattern recognizer} under an EM algorithm to explicitly estimate missing patterns. We show that the guidance which understands missing pattern can be helpful for generating missing values precisely in diffusion imputation model. Our theoretical derivation and extensive experiments demonstrate that PRDIM consistently improves imputation compared to existing methods.

\bibliography{references}

@article{cao2018brits,
  title={Brits: Bidirectional recurrent imputation for time series},
  author={Cao, Wei and Wang, Dong and Li, Jian and Zhou, Hao and Li, Lei and Li, Yitan},
  journal={Advances in neural information processing systems},
  volume={31},
  year={2018}
}

@article{du2023saits,
  title={Saits: Self-attention-based imputation for time series},
  author={Du, Wenjie and C{\^o}t{\'e}, David and Liu, Yan},
  journal={Expert Systems with Applications},
  volume={219},
  pages={119619},
  year={2023},
  publisher={Elsevier}
}

@article{liu2019naomi,
  title={Naomi: Non-autoregressive multiresolution sequence imputation},
  author={Liu, Yukai and Yu, Rose and Zheng, Stephan and Zhan, Eric and Yue, Yisong},
  journal={Advances in neural information processing systems},
  volume={32},
  year={2019}
}

@inproceedings{fortuin2020gp,
  title={Gp-vae: Deep probabilistic time series imputation},
  author={Fortuin, Vincent and Baranchuk, Dmitry and R{\"a}tsch, Gunnar and Mandt, Stephan},
  booktitle={International conference on artificial intelligence and statistics},
  pages={1651--1661},
  year={2020},
  organization={PMLR}
}

@inproceedings{yoon2018gain,
  title={Gain: Missing data imputation using generative adversarial nets},
  author={Yoon, Jinsung and Jordon, James and Schaar, Mihaela},
  booktitle={International conference on machine learning},
  pages={5689--5698},
  year={2018},
  organization={PMLR}
}

@article{ho2020denoising,
  title={Denoising diffusion probabilistic models},
  author={Ho, Jonathan and Jain, Ajay and Abbeel, Pieter},
  journal={Advances in neural information processing systems},
  volume={33},
  pages={6840--6851},
  year={2020}
}

@article{kong2020diffwave,
  title={Diffwave: A versatile diffusion model for audio synthesis},
  author={Kong, Zhifeng and Ping, Wei and Huang, Jiaji and Zhao, Kexin and Catanzaro, Bryan},
  journal={arXiv preprint arXiv:2009.09761},
  year={2020}
}

@inproceedings{jo2022score,
  title={Score-based generative modeling of graphs via the system of stochastic differential equations},
  author={Jo, Jaehyeong and Lee, Seul and Hwang, Sung Ju},
  booktitle={International conference on machine learning},
  pages={10362--10383},
  year={2022},
  organization={PMLR}
}

@inproceedings{rasul2021autoregressive,
  title={Autoregressive denoising diffusion models for multivariate probabilistic time series forecasting},
  author={Rasul, Kashif and Seward, Calvin and Schuster, Ingmar and Vollgraf, Roland},
  booktitle={International conference on machine learning},
  pages={8857--8868},
  year={2021},
  organization={PMLR}
}

@article{tashiro2021csdi,
  title={Csdi: Conditional score-based diffusion models for probabilistic time series imputation},
  author={Tashiro, Yusuke and Song, Jiaming and Song, Yang and Ermon, Stefano},
  journal={Advances in neural information processing systems},
  volume={34},
  pages={24804--24816},
  year={2021}
}

@inproceedings{zhou2024mtsci,
  title={Mtsci: A conditional diffusion model for multivariate time series consistent imputation},
  author={Zhou, Jianping and Li, Junhao and Zheng, Guanjie and Wang, Xinbing and Zhou, Chenghu},
  booktitle={Proceedings of the 33rd ACM International Conference on Information and Knowledge Management},
  pages={3474--3483},
  year={2024}
}

@inproceedings{zhang2025diffputer,
  title={Diffputer: Empowering diffusion models for missing data imputation},
  author={Zhang, Hengrui and Fang, Liancheng and Wu, Qitian and Yu, Philip S},
  booktitle={The Thirteenth International Conference on Learning Representations},
  year={2025}
}

@inproceedings{zhou2021informer,
  title={Informer: Beyond efficient transformer for long sequence time-series forecasting},
  author={Zhou, Haoyi and Zhang, Shanghang and Peng, Jieqi and Zhang, Shuai and Li, Jianxin and Xiong, Hui and Zhang, Wancai},
  booktitle={Proceedings of the AAAI conference on artificial intelligence},
  volume={35},
  number={12},
  pages={11106--11115},
  year={2021}
}

@article{li2017diffusion,
  title={Diffusion convolutional recurrent neural network: Data-driven traffic forecasting},
  author={Li, Yaguang and Yu, Rose and Shahabi, Cyrus and Liu, Yan},
  journal={arXiv preprint arXiv:1707.01926},
  year={2017}
}

@article{ipsen2020not,
  title={not-MIWAE: Deep generative modelling with missing not at random data},
  author={Ipsen, Niels Bruun and Mattei, Pierre-Alexandre and Frellsen, Jes},
  journal={arXiv preprint arXiv:2006.12871},
  year={2020}
}

@article{wu2022timesnet,
  title={Timesnet: Temporal 2d-variation modeling for general time series analysis},
  author={Wu, Haixu and Hu, Tengge and Liu, Yong and Zhou, Hang and Wang, Jianmin and Long, Mingsheng},
  journal={arXiv preprint arXiv:2210.02186},
  year={2022}
}

@article{wang2024timemixer++,
  title={Timemixer++: A general time series pattern machine for universal predictive analysis},
  author={Wang, Shiyu and Li, Jiawei and Shi, Xiaoming and Ye, Zhou and Mo, Baichuan and Lin, Wenze and Ju, Shengtong and Chu, Zhixuan and Jin, Ming},
  journal={arXiv preprint arXiv:2410.16032},
  year={2024}
}

@article{little1987statistical,
  title={Statistical analysis with missing data},
  author={Little, Roderick JA and Rubin, Donald B},
  journal={New York: Wiley},
  year={1987}
}

@inproceedings{mattei2019miwae,
  title={MIWAE: Deep generative modelling and imputation of incomplete data sets},
  author={Mattei, Pierre-Alexandre and Frellsen, Jes},
  booktitle={International conference on machine learning},
  pages={4413--4423},
  year={2019},
  organization={PMLR}
}

@inproceedings{chung2023diffusion,
  title={DIFFUSION POSTERIOR SAMPLING FOR GENERAL NOISY INVERSE PROBLEMS},
  author={Chung, Hyungjin and Kim, Jeongsol and McCann, Michael T and Klasky, Marc L and Ye, Jong Chul},
  booktitle={11th International Conference on Learning Representations, ICLR 2023},
  year={2023}
}

@article{van2011mice,
  title={mice: Multivariate imputation by chained equations in R},
  author={Van Buuren, Stef and Groothuis-Oudshoorn, Karin},
  journal={Journal of statistical software},
  volume={45},
  pages={1--67},
  year={2011}
}

@article{stekhoven2012missforest,
  title={MissForest—non-parametric missing value imputation for mixed-type data},
  author={Stekhoven, Daniel J and B{\"u}hlmann, Peter},
  journal={Bioinformatics},
  volume={28},
  number={1},
  pages={112--118},
  year={2012},
  publisher={Oxford University Press}
}

@book{kantardzic2011data,
  title={Data mining: concepts, models, methods, and algorithms},
  author={Kantardzic, Mehmed},
  year={2011},
  publisher={John Wiley \& Sons}
}

@article{yoon2018estimating,
  title={Estimating missing data in temporal data streams using multi-directional recurrent neural networks},
  author={Yoon, Jinsung and Zame, William R and Van Der Schaar, Mihaela},
  journal={IEEE Transactions on Biomedical Engineering},
  volume={66},
  number={5},
  pages={1477--1490},
  year={2018},
  publisher={IEEE}
}

@inproceedings{
song2021scorebased,
title={Score-Based Generative Modeling through Stochastic Differential Equations},
author={Yang Song and Jascha Sohl-Dickstein and Diederik P Kingma and Abhishek Kumar and Stefano Ermon and Ben Poole},
booktitle={International Conference on Learning Representations},
year={2021},
url={https://openreview.net/forum?id=PxTIG12RRHS}
}

@article{alcaraz2022diffusion,
  title={Diffusion-based time series imputation and forecasting with structured state space models},
  author={Alcaraz, Juan Miguel Lopez and Strodthoff, Nils},
  journal={arXiv preprint arXiv:2208.09399},
  year={2022}
}

@article{yuan2024diffusion,
  title={Diffusion-TS: Interpretable Diffusion for General Time Series Generation},
  author={Yuan, Xinyu and Qiao, Yan},
  journal={CoRR},
  year={2024}
}

@inproceedings{richardson2020mcflow,
  title={Mcflow: Monte carlo flow models for data imputation},
  author={Richardson, Trevor W and Wu, Wencheng and Lin, Lei and Xu, Beilei and Bernal, Edgar A},
  booktitle={Proceedings of the IEEE/CVF conference on computer vision and pattern recognition},
  pages={14205--14214},
  year={2020}
}

@article{xiao2017fashion,
  title={Fashion-mnist: a novel image dataset for benchmarking machine learning algorithms},
  author={Xiao, Han and Rasul, Kashif and Vollgraf, Roland},
  journal={arXiv preprint arXiv:1708.07747},
  year={2017}
}

@book{carlin2000bayes,
  title={Bayes and empirical Bayes methods for data analysis},
  author={Carlin, Bradley P and Louis, Thomas A and others},
  year={2000},
  publisher={Chapman \& Hall/CRC}
}

@article{efron2011tweedie,
  title={Tweedie’s formula and selection bias},
  author={Efron, Bradley},
  journal={Journal of the American Statistical Association},
  volume={106},
  number={496},
  pages={1602--1614},
  year={2011},
  publisher={Taylor \& Francis}
}

@inproceedings{peebles2023scalable,
  title={Scalable diffusion models with transformers},
  author={Peebles, William and Xie, Saining},
  booktitle={Proceedings of the IEEE/CVF international conference on computer vision},
  pages={4195--4205},
  year={2023}
}

@inproceedings{he2022masked,
  title={Masked autoencoders are scalable vision learners},
  author={He, Kaiming and Chen, Xinlei and Xie, Saining and Li, Yanghao and Doll{\'a}r, Piotr and Girshick, Ross},
  booktitle={Proceedings of the IEEE/CVF conference on computer vision and pattern recognition},
  pages={16000--16009},
  year={2022}
}

@article{goldberger2000physiobank,
  title={PhysioBank, PhysioToolkit, and PhysioNet: components of a new research resource for complex physiologic signals},
  author={Goldberger, Ary L and Amaral, Luis AN and Glass, Leon and Hausdorff, Jeffrey M and Ivanov, Plamen Ch and Mark, Roger G and Mietus, Joseph E and Moody, George B and Peng, Chung-Kang and Stanley, H Eugene},
  journal={circulation},
  volume={101},
  number={23},
  pages={e215--e220},
  year={2000},
  publisher={Lippincott Williams \& Wilkins}
}

@article{li2019misgan,
  title={Misgan: Learning from incomplete data with generative adversarial networks},
  author={Li, Steven Cheng-Xian and Jiang, Bo and Marlin, Benjamin},
  journal={arXiv preprint arXiv:1902.09599},
  year={2019}
}

@article{zhang2017cautionary,
  title={Cautionary tales on air-quality improvement in Beijing},
  author={Zhang, Shuyi and Guo, Bin and Dong, Anlan and He, Jing and Xu, Ziping and Chen, Song Xi},
  journal={Proceedings of the Royal Society A: Mathematical, Physical and Engineering Sciences},
  volume={473},
  number={2205},
  pages={20170457},
  year={2017},
  publisher={The Royal Society Publishing}
}

@article{coletta2023constrained,
  title={On the constrained time-series generation problem},
  author={Coletta, Andrea and Gopalakrishnan, Sriram and Borrajo, Daniel and Vyetrenko, Svitlana},
  journal={Advances in Neural Information Processing Systems},
  volume={36},
  pages={61048--61059},
  year={2023}
}

@inproceedings{esser2024scaling,
  title={Scaling rectified flow transformers for high-resolution image synthesis},
  author={Esser, Patrick and Kulal, Sumith and Blattmann, Andreas and Entezari, Rahim and M{\"u}ller, Jonas and Saini, Harry and Levi, Yam and Lorenz, Dominik and Sauer, Axel and Boesel, Frederic and others},
  booktitle={Forty-first international conference on machine learning},
  year={2024}
}

@article{ho2022classifier,
  title={Classifier-free diffusion guidance},
  author={Ho, Jonathan and Salimans, Tim},
  journal={arXiv preprint arXiv:2207.12598},
  year={2022}
}

@article{ma2021identifiable,
  title={Identifiable generative models for missing not at random data imputation},
  author={Ma, Chao and Zhang, Cheng},
  journal={Advances in Neural Information Processing Systems},
  volume={34},
  pages={27645--27658},
  year={2021}
}

@article{kim2022refining,
  title={Refining generative process with discriminator guidance in score-based diffusion models},
  author={Kim, Dongjun and Kim, Yeongmin and Kwon, Se Jung and Kang, Wanmo and Moon, Il-Chul},
  journal={arXiv preprint arXiv:2211.17091},
  year={2022}
}

@inproceedings{samdani2012unified,
  title={Unified expectation maximization},
  author={Samdani, Rajhans and Chang, Ming-Wei and Roth, Dan},
  booktitle={Proceedings of the 2012 Conference of the North American Chapter of the Association for Computational Linguistics: Human Language Technologies},
  pages={688--698},
  year={2012}
}

@book{schafer1997analysis,
  title={Analysis of incomplete multivariate data},
  author={Schafer, Joseph L},
  year={1997},
  publisher={CRC press}
}

@article{goodfellow2014generative,
  title={Generative adversarial nets},
  author={Goodfellow, Ian J and Pouget-Abadie, Jean and Mirza, Mehdi and Xu, Bing and Warde-Farley, David and Ozair, Sherjil and Courville, Aaron and Bengio, Yoshua},
  journal={Advances in neural information processing systems},
  volume={27},
  year={2014}
}

@article{kingma2013auto,
  title={Auto-encoding variational bayes},
  author={Kingma, Diederik P and Welling, Max},
  journal={arXiv preprint arXiv:1312.6114},
  year={2013}
}

@article{du2023pypots,
    title = {{PyPOTS: A Python Toolkit for Machine Learning on Partially-Observed Time Series}},
    author = {Wenjie Du and Yiyuan Yang and Linglong Qian and Jun Wang and Qingsong Wen},
    journal = {arXiv preprint arXiv:2305.18811},
    year = {2023}
}

@article{heusel2017gans,
  title={Gans trained by a two time-scale update rule converge to a local nash equilibrium},
  author={Heusel, Martin and Ramsauer, Hubert and Unterthiner, Thomas and Nessler, Bernhard and Hochreiter, Sepp},
  journal={Advances in neural information processing systems},
  volume={30},
  year={2017}
}

@misc{Seitzer2020FID,
  author={Maximilian Seitzer},
  title={{pytorch-fid: FID Score for PyTorch}},
  month={August},
  year={2020},
  note={Version 0.3.0},
  howpublished={\url{https://github.com/mseitzer/pytorch-fid}},
}

@article{carreras2021missing,
  title={Missing not at random in end of life care studies: multiple imputation and sensitivity analysis on data from the ACTION study},
  author={Carreras, Giulia and Miccinesi, Guido and Wilcock, Andrew and Preston, Nancy and Nieboer, Daan and Deliens, Luc and Groenvold, Mogensm and Lunder, Urska and van der Heide, Agnes and Baccini, Michela and others},
  journal={BMC medical research methodology},
  volume={21},
  number={1},
  pages={13},
  year={2021},
  publisher={Springer}
}

@inproceedings{CelebAMask-HQ,
  title = {MaskGAN: Towards Diverse and Interactive Facial Image Manipulation},
  author = {Lee, Cheng-Han and Liu, Ziwei and Wu, Lingyun and Luo, Ping},
  booktitle = {IEEE Conference on Computer Vision and Pattern Recognition (CVPR)},
  year = {2020}
}

@inproceedings{liu2024self,
  title={Self-supervision improves diffusion models for tabular data imputation},
  author={Liu, Yixin and Ajanthan, Thalaiyasingam and Husain, Hisham and Nguyen, Vu},
  booktitle={Proceedings of the 33rd ACM International Conference on Information and Knowledge Management},
  pages={1513--1522},
  year={2024}
}

@article{karras2022elucidating,
  title={Elucidating the design space of diffusion-based generative models},
  author={Karras, Tero and Aittala, Miika and Aila, Timo and Laine, Samuli},
  journal={Advances in neural information processing systems},
  volume={35},
  pages={26565--26577},
  year={2022}
}

@inproceedings{jarrett2022hyperimpute,
  title={Hyperimpute: Generalized iterative imputation with automatic model selection},
  author={Jarrett, Daniel and Cebere, Bogdan C and Liu, Tennison and Curth, Alicia and van der Schaar, Mihaela},
  booktitle={International Conference on Machine Learning},
  pages={9916--9937},
  year={2022},
  organization={PMLR}
}

@inproceedings{muzellec2020missing,
  title={Missing data imputation using optimal transport},
  author={Muzellec, Boris and Josse, Julie and Boyer, Claire and Cuturi, Marco},
  booktitle={International Conference on Machine Learning},
  pages={7130--7140},
  year={2020},
  organization={PMLR}
}

@article{miao2015identification,
  title={Identification, doubly robust estimation, and semiparametric efficiency theory of nonignorable missing data with a shadow variable},
  author={Miao, Wang and Liu, Lan and Tchetgen, Eric Tchetgen and Geng, Zhi},
  journal={arXiv preprint arXiv:1509.02556},
  year={2015}
}

@article{cursio2019latent,
  title={Latent trait shared-parameter mixed models for missing ecological momentary assessment data},
  author={Cursio, John F and Mermelstein, Robin J and Hedeker, Donald},
  journal={Statistics in Medicine},
  volume={38},
  number={4},
  pages={660--673},
  year={2019},
  publisher={Wiley Online Library}
}

@article{shah2017distribution,
  title={Distribution based nearest neighbor imputation for truncated high dimensional data with applications to pre-clinical and clinical metabolomics studies},
  author={Shah, Jasmit S and Rai, Shesh N and DeFilippis, Andrew P and Hill, Bradford G and Bhatnagar, Aruni and Brock, Guy N},
  journal={BMC bioinformatics},
  volume={18},
  number={1},
  pages={114},
  year={2017},
  publisher={Springer}
}

@article{cazelles2020wasserstein,
  title={The Wasserstein-Fourier distance for stationary time series},
  author={Cazelles, Elsa and Robert, Arnaud and Tobar, Felipe},
  journal={IEEE Transactions on Signal Processing},
  volume={69},
  pages={709--721},
  year={2020},
  publisher={IEEE}
}

@article{dempster1977maximum,
  title={Maximum likelihood from incomplete data via the EM algorithm},
  author={Dempster, Arthur P and Laird, Nan M and Rubin, Donald B},
  journal={Journal of the royal statistical society: series B (methodological)},
  volume={39},
  number={1},
  pages={1--22},
  year={1977},
  publisher={Wiley Online Library}
}

@inproceedings{ronneberger2015u,
  title={U-net: Convolutional networks for biomedical image segmentation},
  author={Ronneberger, Olaf and Fischer, Philipp and Brox, Thomas},
  booktitle={International Conference on Medical image computing and computer-assisted intervention},
  pages={234--241},
  year={2015},
  organization={Springer}
}

@inproceedings{zhao2023transformed,
  title={Transformed distribution matching for missing value imputation},
  author={Zhao, He and Sun, Ke and Dezfouli, Amir and Bonilla, Edwin V},
  booktitle={International Conference on Machine Learning},
  pages={42159--42186},
  year={2023},
  organization={PMLR}
}

@article{zheng2022diffusion,
  title={Diffusion models for missing value imputation in tabular data},
  author={Zheng, Shuhan and Charoenphakdee, Nontawat},
  journal={arXiv preprint arXiv:2210.17128},
  year={2022}
}

@inproceedings{lugmayr2022repaint,
  title={Repaint: Inpainting using denoising diffusion probabilistic models},
  author={Lugmayr, Andreas and Danelljan, Martin and Romero, Andres and Yu, Fisher and Timofte, Radu and Van Gool, Luc},
  booktitle={Proceedings of the IEEE/CVF conference on computer vision and pattern recognition},
  pages={11461--11471},
  year={2022}
}

@article{meng2021sdedit,
  title={Sdedit: Guided image synthesis and editing with stochastic differential equations},
  author={Meng, Chenlin and He, Yutong and Song, Yang and Song, Jiaming and Wu, Jiajun and Zhu, Jun-Yan and Ermon, Stefano},
  journal={arXiv preprint arXiv:2108.01073},
  year={2021}
}

@article{austin2021structured,
  title={Structured denoising diffusion models in discrete state-spaces},
  author={Austin, Jacob and Johnson, Daniel D and Ho, Jonathan and Tarlow, Daniel and Van Den Berg, Rianne},
  journal={Advances in neural information processing systems},
  volume={34},
  pages={17981--17993},
  year={2021}
}

@article{hoogeboom2021argmax,
  title={Argmax flows and multinomial diffusion: Learning categorical distributions},
  author={Hoogeboom, Emiel and Nielsen, Didrik and Jaini, Priyank and Forr{\'e}, Patrick and Welling, Max},
  journal={Advances in neural information processing systems},
  volume={34},
  pages={12454--12465},
  year={2021}
}

@inproceedings{shi2024tabdiff,
  title={TabDiff: a unified diffusion model for multi-modal tabular data generation},
  author={Shi, Juntong and Xu, Minkai and Hua, Harper and Zhang, Hengrui and Ermon, Stefano and Leskovec, Jure},
  booktitle={NeurIPS 2024 Third Table Representation Learning Workshop},
  year={2024}
}
\bibliographystyle{unsrt}


\newpage
\appendix

\section*{Appendix}
\renewcommand{\contentsname}{\vspace{-1.0cm}}

\addtocontents{toc}{\protect\setcounter{tocdepth}{2}}

\tableofcontents
\vspace{1cm}


\section{Proofs}
\label{app:proof}

\subsection{Proof of Proposition \ref{prop:elbo}}
\label{app:prf1}

\propelbo*
\begin{proof}
    
    \begin{align}
    \log{p_{\theta,\phi}({X}^{\text{obs}},{M})} &=\log \int_{{Z}}\int_{{X}^{\text{mis}}} p_{\theta,\phi}({X}^{\text{obs}},{X}^{\text{mis}},{M},{Z})dX^{\text{mis}}d{Z} \\
    &=\log \int_{{Z}}\int_{{X}^{\text{mis}}} \frac{p_{\theta,\phi}({X}^{\text{obs}},{X}^{\text{mis}},{M},{Z})}{q({X}^{\text{mis}},{Z}|{X}^{\text{obs}},{M})}q({X}^{\text{mis}},{Z}|{X}^{\text{obs}},{M})dX^{\text{mis}}d{Z} \\
    &=\log \int_{{Z}}\int_{{X}^{\text{mis}}} \frac{p_{\phi}({M}|{X})p_\theta({X}|{Z})p_\theta({Z})}{q({Z}|{X})q({X}^{\text{mis}}|{X}^{\text{obs}},{M})}q({Z}|{X})q({X}^{\text{mis}}|{X}^{\text{obs}},{M})dX^{\text{mis}}d{Z}\\
    &\geq \mathbb{E}_{{Z}\sim q({Z}|{X}),\, {X}^{\text{mis}}\sim q({X}^{\text{mis}}|{X}^{\text{obs}},{M})} [\log p_{\phi}({M}|{X})+ \log\frac{p_\theta({X}|{Z})p_\theta({Z})}{q({Z}|{X})}] \label{eq:dlvm_elbo} \\
    &\quad + \mathbb{H}(q({X}^{\text{mis}}|{X}^{\text{obs}},{M}))
    \end{align}
    Equation \ref{eq:dlvm_elbo} defines the loss objective between the true parameters and the corresponding variational distribution. Replacing the latent variable ${Z}$ with the diffusion latents ${X}_{1:T}$, can be formulated as equation \ref{eq:diff_elbo} under the Markov property of the diffusion process.
    \begin{align}
        \mathbb{E}_{{X}_{1:T}\sim q({X}_{1:T}|{X}),\, {X}^{\text{mis}}\sim q({X}^{\text{mis}}|{X}^{\text{obs}},{M})} [\log p_{\phi}({M}|{X})+ \log\frac{p_\theta({X}|{X}_1)p_\theta({{X}_{1:T}})}{q({{X}_{1:T}}|{X})}]
    \label{eq:diff_elbo}
    \end{align}
\end{proof}

\subsection{Proof of Corollary \ref{cor:monotonic}}
\label{app:cor1}

\cormonotonic*

\begin{proof}
Regarding the missing values and diffusion trajectories as latent variables
$$
Z \coloneqq (X_0^{\mathrm{mis}}, X_{1:T})
$$
where
$$
X_0 = (X_0^{\mathrm{obs}}, X_0^{\mathrm{mis}})
$$
For notational simplicity, let
$$
Y \coloneqq (X_0^{\mathrm{obs}}, M)
$$
denote the observed variables. Then the marginal likelihood optimized by the
EM procedure is
$$
p_{\theta,\phi}(Y)=p_{\theta,\phi}(X_0^{\mathrm{obs}},M)=\int p_{\theta,\phi}(Y,Z)\,dZ
$$

For any variational distribution \(q(Z)\), the marginal log-likelihood can be
decomposed as
\begin{align}
\log p_{\theta,\phi}(Y)=\log \int p_{\theta,\phi}(Y,Z)\,dZ=\log \int q(Z)\frac{p_{\theta,\phi}(Y,Z)}{q(Z)}\,dZ\geq\mathbb{E}_{q(Z)}\left[\log \frac{p_{\theta,\phi}(Y,Z)}{q(Z)}\right]
\end{align}
We define the corresponding evidence lower bound as
\begin{align}
\mathcal{L}(q,\theta,\phi)\coloneqq\mathbb{E}_{q(Z)}\left[\log p_{\theta,\phi}(X_0^{\mathrm{obs}},X_0^{\mathrm{mis}},M,X_{1:T})-\log q(Z)\right]
\end{align}
where the variational distribution $q(Z)$ is decomposed as we mentioned in Proposition~\ref{prop:elbo},
\begin{align}
q(Z)=q(X_0^{\mathrm{mis}},X_{1:T}\mid X_0^{\mathrm{obs}},M)=q(X_0^{\mathrm{mis}}\mid X_0^{\mathrm{obs}},M)\prod_{t=1}^{T}q(X_t\mid X_{t-1})
\end{align}

We now show the monotonicity of the EM framework. At iteration \(k\), the E step sets the variational distribution to the posterior under the current parameters:
\begin{equation}
q^{(k+1)}(Z)=p_{\theta^{(k)},\phi^{(k)}}(Z\mid Y)=p_{\theta^{(k)},\phi^{(k)}}(X_0^{\mathrm{mis}},X_{1:T}\mid X_0^{\mathrm{obs}},M)
\end{equation}
For this choice of \(q^{(k+1)}\), the ELBO is tight at the current parameters.
\begin{align}
\mathcal{L}(q^{(k+1)},\theta^{(k)},\phi^{(k)})
&=
\mathbb{E}_{q^{(k+1)}(Z)}\left[\log\frac{p_{\theta^{(k)},\phi^{(k)}}(Y,Z)}{q^{(k+1)}(Z)}
\right] \\
&=\mathbb{E}_{p_{\theta^{(k)},\phi^{(k)}}(Z\mid Y)}\left[\log\frac{p_{\theta^{(k)},\phi^{(k)}}(Y,Z)}{p_{\theta^{(k)},\phi^{(k)}}(Z\mid Y)}\right] \\
&=\mathbb{E}_{p_{\theta^{(k)},\phi^{(k)}}(Z\mid Y)}\left[\log p_{\theta^{(k)},\phi^{(k)}}(Y)\right] \\
&=\log p_{\theta^{(k)},\phi^{(k)}}(Y)
\end{align}
The M step then updates the parameters such that
\begin{equation}
\mathcal{L}(q^{(k+1)},\theta^{(k+1)},\phi^{(k+1)})
\geq
\mathcal{L}(q^{(k+1)},\theta^{(k)},\phi^{(k)})
\end{equation}
Furthermore, since \(\mathcal{L}(q,\theta,\phi)\) is a lower bound on the log-likelihood, we have
\begin{equation}
\log p_{\theta^{(k+1)},\phi^{(k+1)}}(Y)
\geq
\mathcal{L}(q^{(k+1)},\theta^{(k+1)},\phi^{(k+1)})
\end{equation}
Combine the above inequalities
\begin{align}
\log p_{\theta^{(k+1)},\phi^{(k+1)}}(Y)
&\geq\mathcal{L}(q^{(k+1)},\theta^{(k+1)},\phi^{(k+1)}) \\
&\geq\mathcal{L}(q^{(k+1)},\theta^{(k)},\phi^{(k)}) \\
&=\log p_{\theta^{(k)},\phi^{(k)}}(Y)
\end{align}
Recall that \(Y=(X_0^{\mathrm{obs}},M)\), we obtain
\begin{equation}
\log p_{\theta^{(k+1)},\phi^{(k+1)}}(X_0^{\mathrm{obs}},M)
\geq
\log p_{\theta^{(k)},\phi^{(k)}}(X_0^{\mathrm{obs}},M)
\end{equation}
\end{proof}

This result should be interpreted as a justification of the underlying EM objective rather than a guarantee for the approximate neural implementation.

\subsection{Proof of Proposition \ref{prop:appx_guide}}
\label{app:prf2}

\propguide*
\begin{proof}
From the graphical model of Figure~\ref{fig:PRDIM_main_fig} (a), conditional independence for missing process satisfies $p_\phi(M|X_0,X_t)=p_\phi(M|X_0)$ at Equation~\ref{eq:ci}. The approximation~\ref{eq:dps} came from the Thm 1 of \citep{chung2023diffusion}.
\begin{align}
    &\nabla_{X_t} \log p_{\theta,\phi}(X_t | X_0^{\text{obs}}, {M}) 
    = \nabla_{X_t} \log p_\theta(X_t | X_0^{\text{obs}})
    + \nabla_{X_t} \log p_{\theta, \phi}({M} | X_t, X_0^{\text{obs}}) \\
    &= \nabla_{X_t} \log p_\theta({X}_t | X_0^{\text{obs}})
    + \nabla_{X_t} \log \int p_\phi(M |X_0)p_\theta(X_0^{\text{mis}}| X_t, X_0^{\text{obs}})dX_0^{\text{mis}} \label{eq:ci} \\
    &\simeq \nabla_{X_t} \log p_\theta(X_t | X_0^{\text{obs}}) + \nabla_{X_t}\log p_\phi(M|f_\theta(X_t,t;X_0^{\text{obs}})^\text{mis},X_0^{\text{obs}}) \label{eq:dps} \\
    &=\nabla_{X_t} \log p_\theta(X_t | X_0^{\text{obs}}) + \nabla_{X_t}M\log D_{\phi^*}\Big(f_\theta(X_t,t;X_0^{\text{obs}})^\text{mis},X_0^{\text{obs}}\Big) \\
    &+\nabla_{X_t}(1-M)\log \Big\{1-D_{\phi^*}\Big(f_\theta(X_t,t;X_0^{\text{obs}})^\text{mis},X_0^{\text{obs}}\Big)\Big\} \\
    &=\nabla_{X_t} \log p_\theta(X_t | X_0^{\text{obs}}) -\nabla_{X_t}  \mathcal{L}_\text{PR}\Big(M,(f_\theta(X_t,t;X_0^{\text{obs}})^\text{mis}, X_0^{\text{obs}}), D_{\phi^*}\Big)
\end{align}
Since $d$-th element of $D_{\phi^*}(X_0)$ converges to $p_\phi(M_d=1|X_0)$ \citep{ipsen2020not, ma2021identifiable}, entire probability of mask variable $M$ follows:
\begin{align}
    &\log p_\phi(M|f_\theta(X_t,t;X_0^{\text{obs}})^\text{mis},X_0^{\text{obs}})=\sum_{i=1}^{D}\log p_\phi(M_d|f_\theta(X_t,t;X_0^{\text{obs}})^\text{mis},X_0^{\text{obs}}) \\
    &= M\log D_{\phi^*}\Big(f_\theta(X_t,t;X_0^{\text{obs}})^\text{mis},X_0^{\text{obs}}\Big)+(1-M)\log \Big\{1-D_{\phi^*}\Big(f_\theta(X_t,t;X_0^{\text{obs}})^\text{mis},X_0^{\text{obs}}\Big)\Big\}
\end{align}
where $(f_\theta(X_t,t;X_0^{\text{obs}})^\text{mis},X_0^{\text{obs}})= f_\theta(X_t,t;X_0^{\text{obs}})\odot(\mathbf{1}-M)+ X_0^{\text{obs}}\odot M$.
\end{proof}

\subsection{Rewritten Theorem 1 of DiffPuter}
\label{app:diffputer_thm}

\begin{theorem}
    Let $X_T$ be a sample from the prior distribution $p_\theta(X_T)=\mathcal{N}(0, \mathbf{I})$, $X$ be the data to impute, and the known entries of $X$ are denoted by $X^\text{obs}=X_0^\text{obs}$. The score function $\nabla_{X_t}\log p(X_t)$ could be parameterized by neural network $f_{\theta} (X_t, t;X_0^\text{obs})$. Applying forward and reverse process of the diffusion model iteratively from $t=T \gg 0$ until $t=0$ with  $\Delta t\rightarrow 0$, then $\hat{X}_0$ is a sample from $p_\theta(X)$, under the condition that its observed entries $\hat{X}_0^\text{obs}=X_0^\text{obs}$. Formally,  
\begin{align}
\hat{X}_0 \sim p_\theta(X| X^\text{obs}=X_0^\text{obs})     
\end{align}
\end{theorem}

We presents a rewritten version of Theorem 1 from DiffPuter~\cite{zhang2025diffputer}, and refer the reader to the original paper for the detailed proof. The theorem establishes that when learning the joint probability of $X^{\text{obs}}$ and $X^{\text{mis}}$, the missing values can be inferred by conditioning on the observed values of a given sample. We adopt the same line of reasoning in Section~\ref{sec:estep} to support our theoretical development.

\section{Related Works}
\label{app:rel_works}

Traditional imputation approaches typically rely on simple statistical heuristics such as mean, median, or last observation carried forward to fill in missing entries of multivariate time-series data \citep{kantardzic2011data}. These strategies often fail to capture the complex temporal dynamics and cross-feature dependencies in them. To improve upon these methods, more sophisticated methods such as MICE \citep{van2011mice}, which iteratively applies the EM algorithm, and MissForest \citep{stekhoven2012missforest}, which leverages random forests for iterative refinement, have been proposed. Although these techniques provide more better imputation performances than naive statistical rules, their capacity remains limited when handling high-dimensional time-series data with intricate dependencies.

\paragraph{Imputation with Deep Learning}

Early attempts to exploit deep learning for time-series imputation include mRNN \citep{yoon2018estimating}, which leverages recurrent neural networks to capture temporal dependencies and model complex patterns in partially observed sequences. BRITS \citep{cao2018brits} further improves upon RNN-based imputers by introducing a bidirectional structure, allowing information to flow forward and backward across time to enhance estimation accuracy. Then NAOMI \citep{liu2019naomi} combines multi-resolution RNNs with adversarial training strategies to refine imputations at different time scales. SAITS \citep{du2023saits} introduces a self-attention mechanism to better capture long-range temporal dependencies.
Most of them could be reproduced with the released Python toolkit~\citep{du2023pypots}.

Beyond time-series-specific architectures, several general-purpose imputation frameworks have also shaped the development of recent methods. GAIN \citep{yoon2018gain}, although not tailored to time-series data, was the first to introduce an adversarial discriminator to imputation, providing a novel mechanism to distinguish observed values and missing values. Successively, misGAN~\citep{li2019misgan} provided multiple generators and discriminators system for stable training across varied missing patterns. Flow-based approaches such as MCFlow \citep{richardson2020mcflow} further demonstrated that the EM algorithm can be combined with invertible generative models to jointly optimize flow parameters and missing entries. not-MIWAE \citep{ipsen2020not} addressed the MNAR scenario by explicitly optimizing a missing model within the ELBO objective, which can inference missing values by missing model weighted importance sampling. Although originally proposed for general imputation tasks, these frameworks have significantly influenced subsequent advances in time-series imputation by highlighting the value of probabilistic, adversarial, and likelihood-based modeling.

\paragraph{Diffusion-based Approaches for Imputation}

In the time-series domain, TimeGrad \citep{rasul2021autoregressive} applies diffusion to probabilistic forecasting, though its design mainly focuses on forecasting task. For the imputation task, CSDI \citep{tashiro2021csdi} introduces a conditional diffusion framework with masking to handle arbitrary missing. Building on this line of research, methods such as SSSD \citep{alcaraz2022diffusion} and Diffusion-TS \citep{yuan2024diffusion} incorporate additional regularization losses tailored to time-series characteristics, thereby enhancing the interpretability of the imputed sequences. On the other hand, MTSCI \citep{zhou2024mtsci} integrates a contrastive loss to maximize mutual information between observed variable and missing variable which improves generated sample consistency. More recently, DiffPuter \citep{zhang2025diffputer} further improves probabilistic imputation with the EM algorithm, which progressively refines the missing values.

\section{Terminology Details}
\label{app:term_details}

\subsection{In-Sample and Out-of-Sample imputation}
\label{app:in_out}

In our experiments, In-Sample imputation refers to evaluating imputation performance on the same dataset used for training, whereas out-of-sample imputation evaluates the model on held-out test splits containing unseen data. The simulated MNAR pattern distributions are generated consistently for each split to ensure complete evaluation. Figure~\ref{fig:inout_description} summarizes the distinction between the in-sample imputation process and the out-of-sample imputation process.

\begin{figure}[t]
    \centering
    \includegraphics[width=0.5\linewidth]{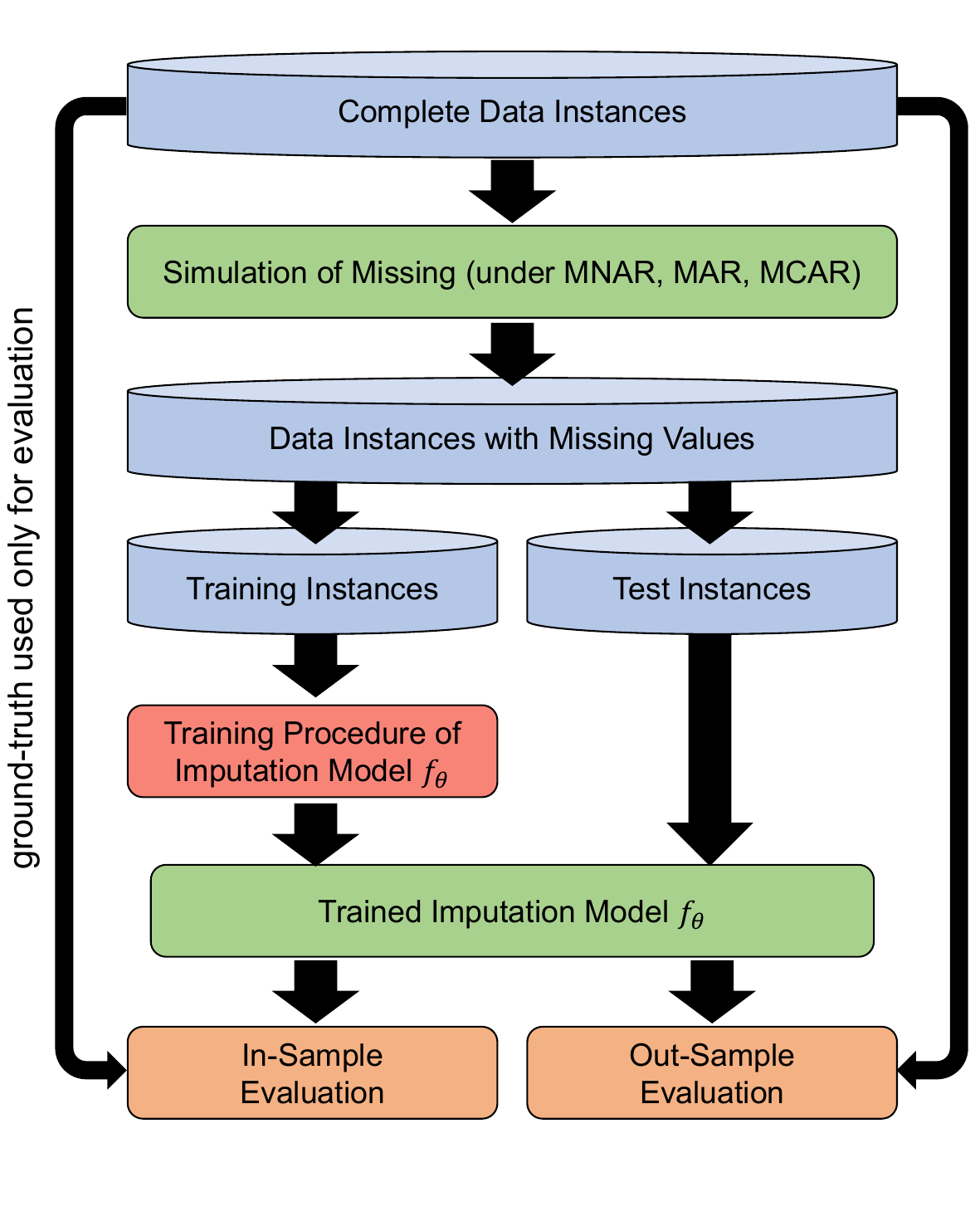}
    \vspace{-0.8cm}
    \caption{Overall information flow between the in-sample imputation and the out-of-sample imputation.}
    \label{fig:inout_description}
\end{figure}

\subsection{Soft EM and Hard EM}
\label{app:soft_hard}

Soft EM refers to an approach in which multiple samples are generated using Monte Carlo sampling, and the model parameters are updated based on the expected imputed value. In this case, the imputed sample corresponds to the posterior expectation
$$
\hat{X}_0=\mathbb{E}_{p_\theta(X_0|X_0^{obs}, M)}[X_0]
$$
In contrast, Hard EM can be viewed as producing a single imputed output obtained by the posterior mode:
$$
\hat{X}_0=\text{arg}\max\limits_{X_0}\ p_\theta(X_0|X_0^{obs}, M)
$$
We adopt Hard EM for two main reasons. (i) It significantly reduces the sampling time because only one imputed sample is required at each iteration. (ii) In our experiments, Hard EM yields slightly better imputation performance than Soft EM.
This observation is consistent with the findings of the prior work~\cite{samdani2012unified}, who report that Hard EM performs better when the initialization is strong, whereas Soft EM is preferable under uninformed initialization. Since PRDIM benefits from Phase 1, where the diffusion model is well initialized by CSDI, the Hard EM variant is naturally more effective in our setting.

\subsection{Detailed Description of Data Processing and Objective}
\label{app:detail_data}

In this section, we aim to clarify the rationale behind our choice of datasets, draw theoretical connections to the EM-based training procedure in PRDIM, and contrast our evaluation protocol with that of prior studies that share a similar experimental framework.
Figure~\ref{fig:data_processing} provides an overview of two distinct classes of objectives used in existing imputation research, highlighting why directly evaluating models on naturally incomplete datasets such as PhysioNet~\citep{goldberger2000physiobank} or AirQuality~\citep{zhang2017cautionary} can be problematic.

Imputation applicable diffusion models including CSDI~\citep{tashiro2021csdi}, SSSD~\citep{alcaraz2022diffusion}, and Diffusion-TS~\citep{yuan2024diffusion}, generally rely on one of two strategies. (i) injecting artificial missing into a complete dataset so that ground-truth values are available during training, or (ii) into already incomplete datasets, thereby increasing the overall missing ratio and using the resulting data as model input. A key commonality between the two imputation paradigms is that the ground-truth values employed for evaluation are implicitly utilized during model training.

Let $M$ denote the original missing mask of the incomplete dataset $X_0^{\text{obs}}$, and let $A$ denote the mask obtained after applying additional artificial missing. The missing distributions induced by these two masks differ intrinsically, which can be formalized as $p(M|X_0)\neq p(A|X_0,M)$. Consequently, the imputed results generated under these differing mask conditions also become different. (\textit{i.e.} $p_\theta(X_0|M,X_0^{\text{obs}})\neq p_\theta(X_0|A,X_0^{\text{obs}}\odot A)$.) Such discrepancies indicate that the imputation task inevitably involves a latent missing variable $X_0^{\text{mis}}$, whose distribution cannot be directly inferred from artificially masked data alone. This observation motivates the necessity of adopting an Expectation–Maximization (EM) training framework, wherein the missing entries are treated as latent variables and iteratively refined during model optimization.

\begin{figure}[t]
\centering
    \centering
    \vspace{0.2cm}
    \includegraphics[width=0.95\linewidth]{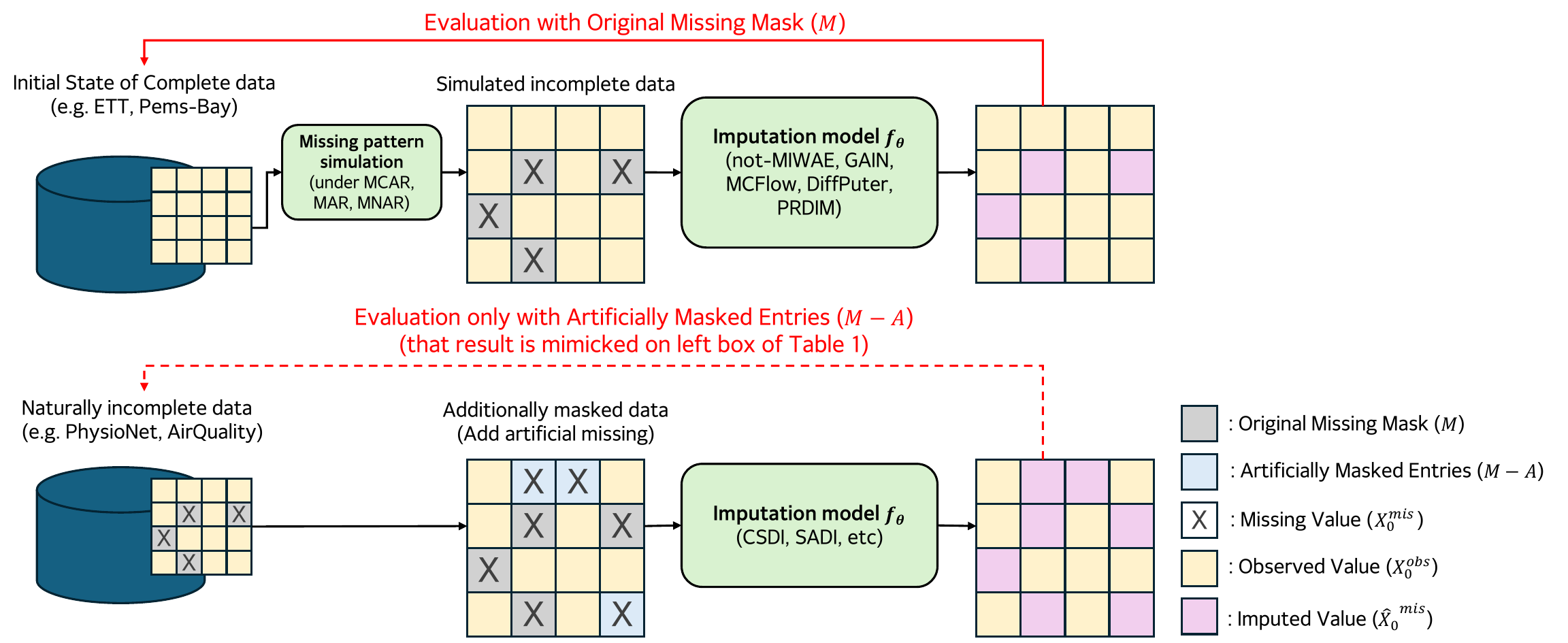}
    \caption{Overview of the data processing pipeline and the distinction between two classes of imputation objectives.}
    \label{fig:data_processing}
\end{figure}

\section{Experimental Details} 
\label{app:exp_details}

\subsection{Evaluation Metrics}
\label{app:eval_metric}

We report three error-based metrics:  
\begin{align}
    \text{RMSE} &= \sqrt{\frac{\sum_{n=1,d=1}^{N,D} (X_{d}^n - \hat{X}_{d}^n)^2\times M_{d}^n}{\sum_{n=1,d=1}^{N,D} M_{d}^n}},
\end{align}
\begin{align}
    \text{MAE}  &= \frac{\sum_{n=1,d=1}^{N,D} |X_{d}^n - \hat{X}_{d}^n|\times M_{d}^n}{\sum_{n=1,d=1}^{N,D} M_{d}^n},
\end{align}
\begin{align}
    \text{MRE}  &= \frac{\sum_{n=1,d=1}^{N,D} |X_{d}^n - \hat{X}_{d}^n|\times M_{d}^n}{\sum_{n=1,d=1}^{N,D} |X_{d}^n|\times M_{d}^n}\times100(\%)
\end{align}
where $X^n_{d}$ denotes the ground-truth value of dimension $d$ of $n$th sample and $\hat{X}^n_{d}$ its imputed counterpart. These complementary measures assess squared error, absolute error, and relative error, providing a comprehensive evaluation of imputation performance.

\subsection{Missing Mechanisms and Datasets}
\label{exp_details:dataset}

\paragraph{Missing Mechanisms}
Table \ref{tab:main_tab} reports results obtained under the following MNAR mechanism. Inspired by the MNAR mechanism of not-MIWAE \citep{ipsen2020not}, we design the MNAR mechanism such that the probability of a missing entry increases exponentially with its value:  
\begin{align}
   p(M_{d} = 1 | X_{d}) = \frac{1}{1 + e^{-\text{logits}}}, \quad 
   \text{logits} = W(X_{d} - b),
\end{align}
where $M_{d} \in \{0,1\}$ denotes the mask variable of entry $X_{d}$, $W$ controls the slope, and $b$ is a bias term. This mechanism ensures that entries with values larger than the mean are more likely to be missing, thus faithfully mimicking MNAR conditions. In main experiments on section~\ref{overall_performance}, we set $W=5$ and $b=0.8$ for all time-series dataset while $W=7$ and $b=0.6$ for FMNIST.

To verify coherent results under MCAR and MNAR mechanisms, we follow the missing simulation procedures of Hyperimpute~\citep{jarrett2022hyperimpute} and MissingOT~\citep{muzellec2020missing}, as reported in Table~\ref{tab:add_missing}. To construct the MNAR setting (\textit{i.e.} top of Table~\ref{tab:add_missing}.), we employ a quantile-based mechanism distinct from the previous logistic approach. A subset of variables is randomly selected, and missing values are generated within the $q$-quantile. Otherwise, in MCAR setting (\textit{i.e.} bottom of Table~\ref{tab:add_missing}.), each value is excluded following the Bernoulli random variable with a fixed parameter. In our implementation, we randomly assign 10\% of the entries as missing to maintain MCAR property.

\paragraph{Time-series and Image Datasets}
Table~\ref{tab:data_config} summarizes the dataset configurations employed in our experiments across (train / test / valid) set. For the three time-series datasets (ETT, STOCK, and PEMS-Bay), the data size are represented as \textit{time length $\times$ feature dimension}, while for FMNIST and CelebA-HQ, the dimensions are denoted as \textit{width $\times$ height}. The reported missing ratios correspond to the proportion of original missing values observed after applying the MNAR mechanism to generate incomplete data for training, thereby reflecting the intrinsic difficulty of the imputation task.

\begin{table}[h]
\centering
\caption{Dataset configuration used in Table~\ref{tab:main_tab}. For time-series datasets (ETT, STOCK, PEMS-Bay), the data size is denoted as \textit{time length $\times$ feature dimension}, whereas for FMNIST, it is expressed as \textit{width $\times$ height}. The missing ratios represent the proportion of original missing values after applying the MNAR mechanism to construct incomplete data. }
\label{tab:data_config}
\resizebox{1.0\textwidth}{!}{
\begin{tabular}{l|ccccc}
\midrule
Dataset & ETTm1 & STOCK & PEMS-Bay & FMNIST & CelebA-HQ\\
\midrule
Data size & $24 \times 7$ & $24 \times 6$ & $12 \times 325$ & $28 \times 28$ & $64\times 64$ \\
\# of Samples & 3861 / 983 / 959 & 2418 / 622 / 622 & 5788 / 1448 / 1448 & 60000 / 5000 / 5000 & 25000 / - / 5000 \\
Missing ratio (\%) & 21.4 / 43.9 / 14.0 & 21.2 / 20.0 / 20.9 & 13.5 / 13.0 / 14.1 & 25.8 / 25.8 / 25.8 & 2.7 / - / 2.7 \\
\midrule
\end{tabular}}
\label{tab:dataset_config}
\vspace{-0.5cm}
\end{table}

\paragraph{Tabular Datasets}
To verify the effectiveness of the pattern recognizer in tabular data imputation, we report the in-sample MAE performance in Table~\ref{tab:tabular_imputation}. To reproduce the tabular data experiments from the DiffPuter framework, we selected a subset of five datasets from the ten UCI datasets originally evaluated. Specifically, the Bean\footnote{\url{https://archive.ics.uci.edu/dataset/602/dry+bean+dataset}}, Magic\footnote{\url{https://archive.ics.uci.edu/dataset/159/magic+gamma+telescope}}, and Gesture\footnote{\url{https://archive.ics.uci.edu/dataset/302/gesture+phase+segmentation}} datasets consist exclusively of continuous features, whereas the Adult\footnote{\url{https://archive.ics.uci.edu/dataset/2/adult}} and Default\footnote{\url{https://archive.ics.uci.edu/dataset/350/default+of+credit+card+clients}} datasets contain both continuous and discrete features. Detailed statistics for these five datasets are summarized in Table~\ref{tab:tabular_data_config}. It is important to note that PRDIM is principally designed as a diffusion imputation framework for continuous feature domain. Consequently, for the Adult and Default datasets, we follow DiffPuter's procedure that applying label encoding and subsequently treating them as continuous features.

\begin{table}[h]
    \centering
    \caption{Statistics of tabular datasets. \# Num denotes the number of numerical features, and \# Cat denotes the number of discrete features.}
    \label{tab:tabular_data_config}
    \resizebox{0.8\linewidth}{!}{
    \begin{tabular}{l|ccccc}
        \toprule
        Dataset & \textbf{\# Total dataset} & \textbf{\# Num} & \textbf{\# Cat} & \textbf{\# In-Sample} & \textbf{\# Out-of-Sample} \\
        \midrule
        Adult & 32,561 & 6 & 8 & 22,792 & 9,769 \\
        Bean & 13,610 & 17 & - & 9,527 & 4,083 \\
        Default & 30,000 & 14 & 10 & 21,000 & 9,000 \\
        Gesture & 9,522 & 32 & - & 6,665 & 2,857 \\
        Magic & 19,020 & 10 & - & 13,314 & 5,706 \\
        \bottomrule
    \end{tabular}}
    \vspace{-0.5cm}
\end{table}

\subsection{Diffusion Imputation Model Configuration}
\label{exp_details:architecture}

In our implementation, the pattern recognizer is designed as a lightweight neural network architecture across dataset modalities to minimize additional complexity.

\paragraph{Time-series Imputation}
As summarized in Table~\ref{tab:model_config}, the overall model size remains comparable to other diffusion-based approaches. The higher inference time of PRDIM, relative to competing methods, arises from the use of autogradient of the input with respect to the outputs of the pattern recognizer during the inference process. This design choice, while incurring additional computational cost, enables the model to provide more informative guidance for imputing missing values. Futhermore, for all diffusion-based imputation models (CSDI, MTSCI, cDiffPuter, and PRDIM) used in time-series imputation, the diffusion backbone architecture follows the official CSDI repository\footnote{\url{https://github.com/ermongroup/csdi}}, with only the depth of the diffusion layers adjusted 4 to 2. In PRDIM, the pattern recognizer is constructed from a single diffusion layer, where the transformer blocks are replaced with simple MLP layers operating along the time and feature axes. This results in a lightweight module with approximately 10\% of the diffusion model size.

\begin{table}[h]
\centering
\caption{Detailed model configuration. "\# of Params" indicates the number of parameter of each diffusion-based model, and $+n$ in PRDIM shows the number of pattern recognizer. The inference time (s) is measured based on the a single inference required to impute the entire out-of-sample data.}
\label{tab:model_config}
\resizebox{1.0\textwidth}{!}{
\begin{tabular}{l|ccc|ccc}
\toprule
 & \multicolumn{3}{c|}{\textbf{ETT}} & \multicolumn{3}{c}{\textbf{STOCK}} \\
\cmidrule(lr){2-4}\cmidrule(lr){5-7}
\textbf{Method} & \# of Params & Training Time (s) & Inference Time (s) & \# of Params & Training Time (s) & Inference Time (s) \\
\midrule
CSDI       & 164,025       & 366  & 22  & 164,017       & 138  & 11 \\
MTSCI      & 162,321       & 585  & 21  & 146,969       & 354  & 12 \\
cDiffPuter & 163,769       & 981  & 28  & 163,761       & 516  & 11 \\
PRDIM      & 163,769+17,376 & 1812 & 46  & 163,761+17,359 & 1052 & 27 \\
\bottomrule
\end{tabular}}
\end{table}

\paragraph{Image and Tabular Imputation} For image datasets (FMNIST and CelebA), the diffusion model uses a U-Net~\cite{ronneberger2015u} architecture composed of two stacked blocks (each with down, mid, and up components), while the pattern recognizer uses only a single U-Net block, corresponding to approximately 38\% of the full model size. For tabular datasets, the diffusion backbone follows the architecture reported in the official DiffPuter repository\footnote{\url{https://github.com/hengruizhang98/DiffPuter}}, and the pattern recognizer is implemented as a simple MLP with three 512-dimension hidden layers. For example, in the Adult dataset, the number of parameters of the diffusion architecture is 4.89 times larger than that of the pattern recognizer. Additionally, we provide the architectural details of the pattern recognizer for all dataset types in Table~\ref{tab:pr_architectures}.

\begin{table}[h]
\centering
\caption{Pattern recognizer architectures for time-series, image, and tabular dataset.}
\label{tab:pr_architectures}
\renewcommand{\arraystretch}{1.15}
\setlength{\tabcolsep}{10pt}
\resizebox{1.0\linewidth}{!}{
\begin{tabular}{lll}
\toprule
\textbf{Time-series data} & \textbf{Image data} & \textbf{Tabular data} \\
\midrule
Input $x \in \mathbb{R}^{B \times 1 \times K \times L}$ &
Input $x \in \mathbb{R}^{B \times 1 \times H \times W}$ &
Input $x \in \mathbb{R}^{d_{\text{in}}}$ \\
Reshape$(B,1,K,L \rightarrow B,1,KL)$ &
Conv2D$(1 \rightarrow C)$ &
Linear$(d_{\text{in}} \rightarrow d_t)$ \\
$1 \times 1$ Conv1D$(1 \rightarrow C)$ &
DownBlock$(C \rightarrow C)$ &
SiLU \\
ReLU &
DownBlock$(C \rightarrow 2C)$ &
Linear$(d_t \rightarrow d_t)$ \\
Reshape$(B,C,KL \rightarrow B,C,K,L)$ &
MidBlock$(2C \rightarrow 2C)$ &
SiLU \\
Repeat for $r=1,\dots,R$: &
ConvTranspose2D$(2C \rightarrow C)$ &
Linear$(d_t \rightarrow d_t)$ \\
\quad Time-MLP along length $L$ &
ConvBlock$(C \rightarrow C)$ &
SiLU \\
\quad Feature-MLP along dimension $K$ &
ConvTranspose2D$(C \rightarrow C)$ &
Linear$(d_t \rightarrow d_t)$ \\
\quad $1 \times 1$ Conv1D$(C \rightarrow 2C)$ &
ConvBlock$(C \rightarrow C)$ &
SiLU \\
\quad Gated activation &
$1 \times 1$ Conv2D$(C \rightarrow 1)$ &
Linear$(d_t \rightarrow d_{\text{in}})$ \\
\quad $1 \times 1$ Conv1D$(C \rightarrow 2C)$ &
Sigmoid &
Sigmoid \\
\quad Residual update &
Squeeze channel dimension &
Output $\hat{m} \in (0,1)^{d_{\text{in}}}$ \\
$1 \times 1$ Conv1D$(C \rightarrow 1)$ &
Output $\hat{m} \in (0,1)^{B \times H \times W}$ &
\\
Reshape$(B,1,KL \rightarrow B,K,L)$ &
&
\\
Sigmoid &
&
\\
Output $\hat{m} \in (0,1)^{B \times K \times L}$ &
&
\\
\bottomrule
\end{tabular}}
\end{table}

\subsection{Hyperparameter Settings of Baselines}
\label{exp_details:baselines}

We followed the established guidelines from existing benchmarks to ensure a fair comparison across all models. For the time-series imputation task, which includes BRITS~\cite{cao2018brits}, SAITS~\cite{du2023saits}, and GP-VAE~\cite{fortuin2020gp}, we adopted the hyperparameter search space and tuning procedures specified in Appendix A of the SAITS paper~\cite{du2023saits}. For TimesNet~\cite{wu2022timesnet}, the hyperparameter search space is defined as follows: the number of layers from (3, 4, 5), model dimension from (256, 512), feed-forward network dimension from (128, 256), and training epochs from (30, 40, 50, 60). TimeMixer++~\cite{wang2024timemixer++} is tuned over an identical search space, with the exception that the number of layers is sampled from (2, 3, 4). Additionally, since the official source code for not-MIWAE~\cite{ipsen2020not} is unavailable, we implement not-MIWAE by integrating a simple MLP missing model into the MIWAE~\cite{mattei2019miwae} plug-in module provided in the HyperImpute repository\footnote{\url{https://github.com/vanderschaarlab/hyperimpute/blob/main/src/hyperimpute/plugins/imputers/plugin_miwae.py}}.

Similarly, for the tabular imputation experiments (Table~\ref{tab:tabular_imputation} and Table~\ref{tab:tabular_out_sample}), the settings for MissForest~\cite{stekhoven2012missforest}, MICE~\cite{van2011mice}, MOT~\cite{muzellec2020missing}, TDM~\cite{zhao2023transformed}, TabCSDI~\cite{zheng2022diffusion}, and HyperImpute~\cite{jarrett2022hyperimpute} were configured based on the implementation details and hyperparameter guidelines provided in Appendices D.5 and D.6 of DiffPuter~\cite{zhang2025diffputer}. By directly adopting these validated configurations, we ensure that each baseline is evaluated under its intended optimal settings.

\section{Additional Experiments} 
\label{app:add_exp}

In this section, we provide additional results on the overall imputation performance across all time-series datasets in Table~\ref{tab:ett_results}, Table~\ref{tab:stock_results}, and Table~\ref{tab:pems_results}. These results complement the main findings reported in the paper and further validate the effectiveness of our approach.

\subsection{Fashion-MNIST}
\label{app:fmnist}

In the image domain, missing values are not limited to the MNAR mechanism demonstrated in the main experiment. As a representative example in the image dataset, we additionally exhibit imputation results under the \textit{block missing} mechanism. Experiments are conducted on the FMNIST dataset, and we further include comparisons with representative GAN-based and Flow-based imputation approaches, namely misGAN~\citep{li2019misgan} and MCFlow~\citep{richardson2020mcflow}. 

As illustrated by the following Figure~\ref{fig:app_fmnist_mar}, among the three methods, our proposed PRDIM most effectively captures the underlying object structure and achieves the most faithful reconstructions. These findings demonstrate that PRDIM can generalize beyond MNAR to handle other types of missing, such as block-MAR, while retaining its ability to generate semantically plausible imputations.

Furthermore, to highlight the general imputation ability of our model under the MNAR mechanism from the main experiment, we also present additional qualitative results in Figure~\ref{fig:app_fmnist_mnar}. For both experiments, we evaluate the quality of generated samples using the Fréchet Inception Distance (FID)~\citep{heusel2017gans}, which is computed with the released Python library~\citep{Seitzer2020FID}.

\subsection{CelebA-HQ}
\label{app:add_exp-celeba}

To verify the scalability of PRDIM on high-dimensional data, we conducted an additional imputation experiment on the RGB image benchmark dataset named CelebA-HQ~\citep{CelebAMask-HQ}. We compared our method with a vanilla diffusion model trained under the CSDI objective. Each image in CelebA-HQ is accompanied by corresponding annotation mask vectors that label facial attributes such as eyes, nose, mouth, and hair.
To design an incomplete dataset under an MNAR pattern, we utilized the annotation masks of eyes, nose, and mouth to construct a missing-value mask. Specifically, for each facial attribute, we introduced missing pixels with an 80\% probability within the annotated regions, forming an MNAR missing mechanism for the experiment.

To efficiently manage the training time of the EM-based PRDIM model trained from a scratch, both images and their corresponding mask vectors were resized to a resolution of 64×64. The original CelebA-HQ dataset consists of 1024×1024 images and 512×512 annotation masks.

Qualitative results are shown in Figure~\ref{fig:celebA_vertical}. The vanilla diffusion model trained with the CSDI objective tends to fill in missing areas with averaged color tones around the missing regions, resulting in naive reconstructions and relatively high FID scores despite a moderate missing ratio. In contrast, PRDIM generates more detailed and realistic facial structures, accurately reconstructing attribute boundaries and color variations, which leads to significantly improved perceptual quality and lower FID values.

\subsubsection{Comparison with Image Inpainting Methods Trained on Complete Data Distribution}
\label{add_exp-celeba:inpainting}

As noted in the footnote in section~\ref{overall_performance}, there is a fundamental difference between diffusion image inpainting and imputation: inpainting methods assume access to a diffusion model trained on a fully observed data distribution, whereas PRDIM addresses the setting where only incomplete data is available. Therefore, the two approaches are not directly comparable. Nevertheless, given the shared goal of restoring masked images, we additionally conduct a comparison by applying zero-shot inpainting methods in Table~\ref{tab:inpainting}. Specifically, we use a pre-trained diffusion model on the FFHQ\footnote{\url{https://github.com/nvlabs/ffhq-dataset}} dataset (The model weight is provided in the EDM repository\footnote{\url{https://github.com/nvlabs/edm}}.) and apply SDEdit~\cite{meng2021sdedit} and RePaint~\cite{lugmayr2022repaint} to the CelebA-HQ test dataset.

Inference settings are follows: DDPM sampling steps = 100, $t_0=0.5$ (Refer to Figure 3 of SDEdit~\cite{meng2021sdedit}), and we set jump length $j=2$ and resampling number $r=2$ (Refer to Section 5.6 of RePaint~\cite{lugmayr2022repaint}). For quantitative comparison, we report peak signal-to-noise-ratio (PSNR), structural similarity index (SSIM), and Learned Perceptual Image Patch Similarity (LPIPS) distance. From these results, we observe that models trained directly on the incomplete target distribution can outperform models trained on a complete but different source distribution, even when the latter is semantically similar (e.g., FFHQ vs CelebA).

\begin{table}[h]
\centering
\caption{Image imputation performance comparison with alternative methods for zero-shot image inpainting. Lower is better for FID and LPIPS, higher is better for PSNR and SSIM.}
\vspace{0.2cm}
\small
\begin{tabular}{l|cccc}
\toprule
Method & FID ($\downarrow$) & PSNR ($\uparrow$) & SSIM ($\uparrow$) & LPIPS ($\downarrow$, $\times 10^{-3}$) \\
\midrule
Vanilla Diffusion & 1.67 & $37.39 \pm 2.32$ & $0.990 \pm 0.004$ & $5.21 \pm 3.53$ \\
DiffPuter & 1.00 & $38.76 \pm 2.38$ & $0.992 \pm 0.004$ & $2.98 \pm 2.30$ \\
SDEdit & 1.49 & $36.51 \pm 2.90$ & $0.987 \pm 0.007$ & $2.38 \pm 2.44$ \\
RePaint & 0.48 & $39.49 \pm 2.80$ & $0.993 \pm 0.004$ & $4.94 \pm 4.63$ \\
\midrule
PRDIM & \textbf{0.47} & \textbf{40.43 $\pm$ 2.62} & \textbf{0.995 $\pm$ 0.003} & \textbf{2.09 $\pm$ 1.78} \\
\bottomrule
\end{tabular}
\label{tab:inpainting}
\end{table}

\begin{figure}[t]
    \centering
    \begin{minipage}[c]{0.20\textwidth}
        \centering \small
        
        \textbf{(a) Ground Truth}
        \par\vspace{2.5em} 
        \textbf{(b) Observed Input}
        \par\vspace{2.5em}
        \textbf{(c) Vanilla Diffusion} \\
        \par\vspace{2.5em}
        \textbf{(d) PRDIM (Ours)} \\
    \end{minipage}
    \hfill
    \begin{minipage}[c]{0.78\textwidth}
        \centering
        \includegraphics[width=\linewidth]{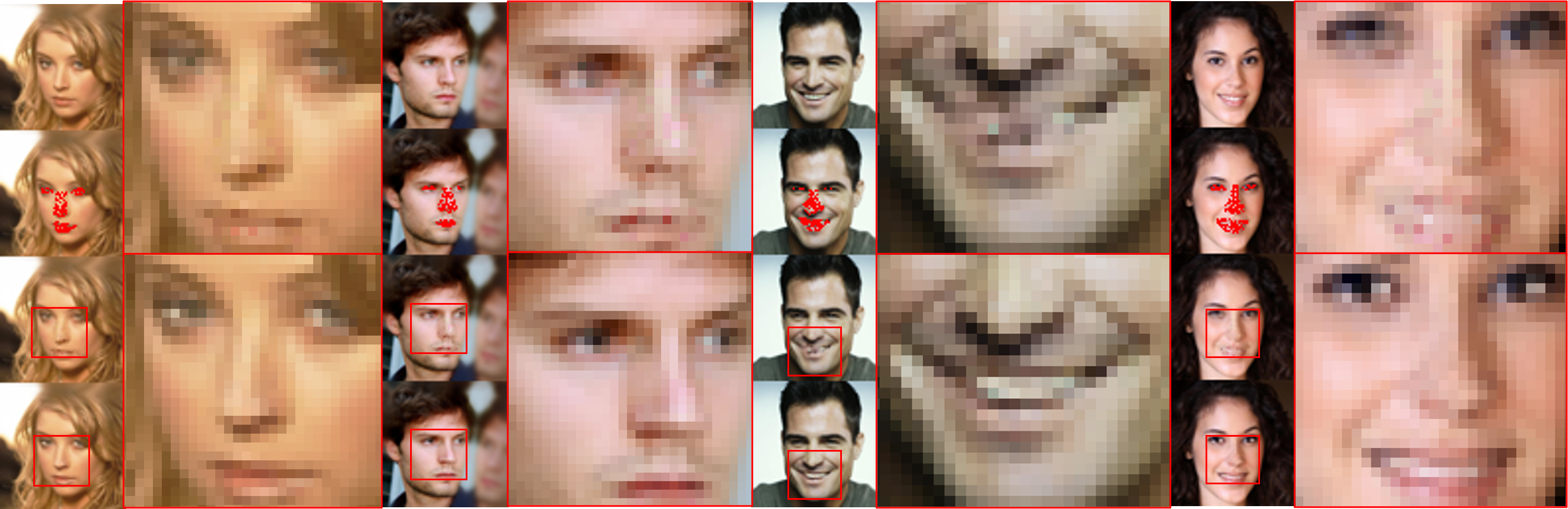}
    \end{minipage}

    \caption{Qualitative comparison on CelebA-HQ. The rows represent (a) Ground Truth, (b) Observed Input, (c) Vanilla Diffusion, and (d) PRDIM, respectively. We demonstrate that PRDIM achieves superior reconstruction quality and quantitative performance than vanilla diffusion model which is trained with CSDI objective.}
    \label{fig:celebA_vertical}
\end{figure}

\subsection{Tabular data}
\label{app:exp_tabular}

To further verify the generalization capability of PRDIM, we reproduced the official implementation of DiffPuter\footnote{\url{https://github.com/hengruizhang98/DiffPuter}} and compared its performance with PRDIM under the MNAR setting.

In this experiment, we followed the practical implementation of DiffPuter with default configuration regardless to dataset, which differs from the main experiments in that incomplete samples were not used as conditional information during imputation.

We selected 5 different tabular datasets available from the UCI Machine Learning Repository\footnote{\url{https://archive.ics.uci.edu/}}. Among them, bean, gesture, and magic consist solely of continuous features, while adult and default contain both continuous and discrete features. The discrete attributes were label-encoded to preserve the original data dimensionality, and the corresponding mask vectors were designed to match this structure.

Table~\ref{tab:tabular_out_sample} present out-of-sample imputation performance of DiffPuter and PRDIM, respectively. Across most datasets, PRDIM achieves more accurate imputations, validating its robustness and adaptability across different data modalities.

\begin{table}[h]
\centering
\caption{Out-of-Sample imputation results across tabular datasets. We report the mean $\pm$ std over five different \textbf{MNAR} missing simulations.}
\resizebox{0.95\textwidth}{!}{
\begin{tabular}{l|cccccc}
\toprule
 & adult & bean & default & gesture & magic & Avg.Rank \\
\midrule
MissForest & 0.635$\pm$0.037 & 0.283$\pm$0.046 & 0.357$\pm$0.051 & 0.373$\pm$0.011 & 0.534$\pm$0.031 & 5.2\\
MOT & 0.504$\pm$0.017 & 0.263$\pm$0.026 & 0.334$\pm$0.036 & 0.427$\pm$0.014 & 0.501$\pm$0.034 & 4.1\\
TDM & 0.514$\pm$0.015 & 0.170$\pm$0.029 & 0.336$\pm$0.062 & 0.399$\pm$0.013 & 0.508$\pm$0.037 & 4.0\\
HyperImpute & 0.559$\pm$0.016 & \textbf{0.138$\pm$0.009} & 0.300$\pm$0.023 & 0.389$\pm$0.019 & \textbf{0.471$\pm$0.036} & 2.6\\
DiffPuter & 0.504$\pm$0.012 & 0.219$\pm$0.053 & 0.315$\pm$0.040 & \textbf{0.353$\pm$0.007} & 0.539$\pm$0.049 & 3.3\\
PRDIM     & \textbf{0.482$\pm$0.022} & 0.199$\pm$0.053 & \textbf{0.279$\pm$0.039} & 0.371$\pm$0.052 & 0.488$\pm$0.047 & \textbf{1.8}\\
\bottomrule
\end{tabular}
}
\label{tab:tabular_out_sample}
\end{table}

\subsection{Cost–Performance Trade-offs under Different EM Configurations}
\label{app:cost-perform}

Since PRDIM is built upon the EM framework, both training time and imputation performance depend on the choice of EM configuration. In this section, we analyze how different configurations affect the trade-off between computational cost and performance on the STOCK dataset. We denote \textbf{1E $N_m$M} as the number of training epochs in the maximization step per expectation step, and $N$ as the total number of EM iterations. In the main experiments, both cDiffPuter and PRDIM are trained with $N=100$ EM iterations under the \textbf{1E 1M} configuration.

\begin{wrapfigure}{r}{0.44\linewidth}
\centering
    \centering
    \vspace{-0.5cm}
    \includegraphics[width=\linewidth]{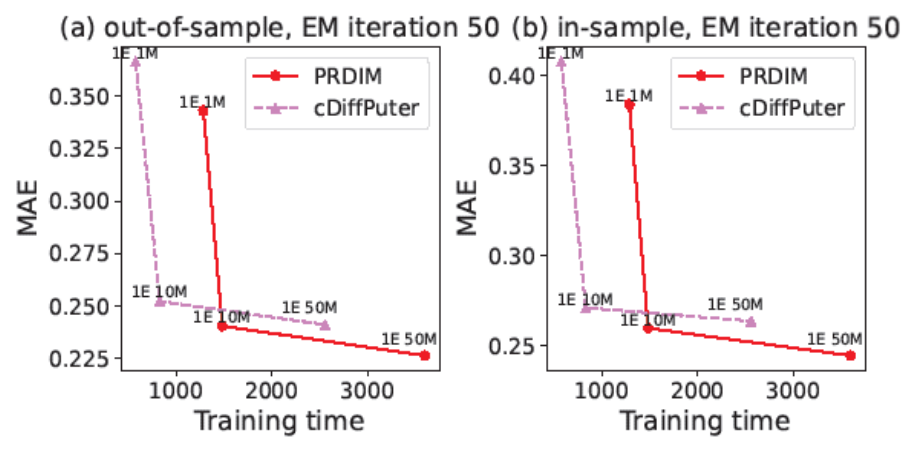}
    \caption{
    Training time and imputation performance (MAE) under different EM configurations on the \textbf{STOCK} dataset.
    }
    \vspace{-0.3cm}
    \label{fig:app_emconfig}
\end{wrapfigure}

As the number of training epochs in the maximization step increases, the overall training time grows proportionally, while the relative cost of the expectation step becomes negligible. Consequently, the training time of PRDIM approaches that of cDiffPuter as the maximization step becomes dominant. Importantly, however, we observe that PRDIM achieves a more favorable cost--performance trade-off. As shown in Figures~\ref{fig:app_emconfig} (a) and (b), PRDIM with the \textbf{1E 10M} configuration and $N=50$ EM iterations not only attains superior imputation performance, but also requires less training time compared to cDiffPuter with the \textbf{1E 50M} configuration and $N=50$ EM iterations.

This result indicates that PRDIM can achieve better performance at lower computational cost by leveraging the pattern recognizer under MNAR settings. It further highlights that appropriate EM configuration allows PRDIM to operate at more efficient points on the cost-performance frontier, demonstrating its practical scalability for real-world imputation tasks.

\subsection{Evaluation on Various Missing Senarios}
\label{app:various}

\paragraph{MNAR Subtypes} We additionally examined several real-world motivated MNAR subtypes and conducted further experiments on the ETT dataset. These subtypes were selected because they (i) appear in real-world domains and (ii) can be implemented without domain-specific knowledge, enabling reproducible evaluation:

\begin{itemize}
    \item \textbf{Self-censoring~\citep{miao2015identification}}: The true observed value directly determines whether the entry becomes missing.
    \item \textbf{Latent-trait MNAR~\citep{cursio2019latent}}: missing depends on an unobserved latent effect or individual-specific factor.
    \item \textbf{Censoring / Truncation MNAR~\citep{shah2017distribution}}: Values outside a certain interval are unobserved or collapsed, simulating realistic censoring processes.
\end{itemize}

To demonstrate the variation in MNAR severity across these subtypes, we provide the train/val/test missing ratios produced during data processing in Table~\ref{tab:mnar_ratio}.

\begin{table*}[h]
    \centering
    \begin{minipage}[b]{0.34\textwidth}
        \centering
        \caption{Missing ratios (\%) across different MNAR subtypes.}
        \label{tab:mnar_ratio}
        \vspace{0.1cm} 
        \resizebox{\linewidth}{!}{%
        \begin{tabular}{l|ccc}
        \toprule
        \textbf{Type} & \textbf{Train} & \textbf{Valid} & \textbf{Test} \\
        \midrule
        Self-censoring & 19.75 & 13.52 & 37.16 \\
        Latent traits& 30.05 & 30.09 & 29.95 \\
        Truncation. & 27.76 & 27.75 & 28.05 \\
        \bottomrule
        \end{tabular}}
    \end{minipage}
    \hfill 
    \begin{minipage}[b]{0.62\textwidth}
        \centering
        \caption{Performance comparison (MAE) on ETT dataset under different MNAR subtypes. We report both In-Sample (In) and Out-of-Sample (Out) imputation errors.}
        \label{tab:mnar_subtype_results}
        \vspace{0.1cm}
        \resizebox{\linewidth}{!}{%
        \begin{tabular}{l|cc|cc|cc}
        \toprule
        \textbf{Method} & \multicolumn{2}{c|}{\textbf{Self-censoring}} & \multicolumn{2}{c|}{\textbf{Latent traits}} & \multicolumn{2}{c}{\textbf{Truncation}} \\
         & In & Out & In & Out & In & Out \\
        \midrule
        CSDI       & 0.3073 & 0.5303 & 0.2338 & 0.2621 & 0.3682 & 0.5780 \\
        MTSCI      & 0.3296 & 0.4848 & \textbf{0.1629} & 0.1968 & 0.4368 & 0.5353 \\
        cDiffPuter & 0.3845 & \textbf{0.3819} & 0.1973 & 0.2074 & 0.5471 & 0.5709 \\
        \textbf{PRDIM} & \textbf{0.2738} & \textbf{0.3819} & 0.1711 & \textbf{0.1930} & \textbf{0.3570} & \textbf{0.4833} \\
        \bottomrule
        \end{tabular}}
    \end{minipage}
\end{table*}

Across all MNAR subtypes, PRDIM consistently outperforms prior diffusion-based imputation methods, as shown in Table~\ref{tab:mnar_subtype_results}. We observe that the performance gap between PRDIM and cDiffPuter is smallest under the latent-trait mechanism. This suggests that the latent-trait subtype may be the most challenging among those tested, which is intuitive as missing driven by latent attributes is more difficult to approximate using only simple CNN (or MLP) structures.

\paragraph{MAR missing situation}
To demonstrate robustness of PRDIM across general missing patterns, we conducted additional experiments using the PyGrinder\footnote{\url{https://github.com/WenjieDu/PyGrinder}} repository, a public toolkit for generating missing in time-series datasets. Following its MAR configuration, we introduced 25\% missing ratio to the ETT, STOCK, and PEMS-Bay datasets, and evaluated both out-of-sample and in-sample MAE performance.

Table~\ref{tab:mar25_results} show that PRDIM continues to perform consistently well under MAR, confirming that the pattern recognizer still provides useful guidance even when the missing mechanism no longer depends on unobserved values. Our empirical results demonstrate that PRDIM maintains performance comparable to DiffPuter under MCAR and MAR settings, while significantly outperforming the baseline in MNAR scenarios where the missing probability explicitly depends on the unobserved values $X_0^{\text{mis}}$

\begin{table}[h]
\centering
\caption{Imputation performance comparison with MAE ($\downarrow$) under \textbf{MAR 25\%} missing mechanism. We report Out-of-Sample and In-Sample errors for ETT, STOCK, and PEMS-Bay datasets respectively.}
\label{tab:mar25_results}
\small
\vspace{0.2cm}
\resizebox{0.8\textwidth}{!}{
\begin{tabular}{l|cc|cc|cc}
\toprule
Method & \multicolumn{2}{c|}{\textbf{ETT}} & \multicolumn{2}{c|}{\textbf{STOCK}} & \multicolumn{2}{c}{\textbf{PEMS-Bay}} \\
\cmidrule(lr){2-3} \cmidrule(lr){4-5} \cmidrule(lr){6-7}
 & Out-of-Sample & In-Sample & Out-of-Sample & In-Sample & Out-of-Sample & In-Sample \\
\midrule
CSDI & 0.1895 & 0.2428 & 0.1472 & 0.0477 & 0.2158 & 0.2034 \\
cDiffPuter & 0.1785 & 0.1853 & \textbf{0.1469} & \textbf{0.0268} & 0.2248 & 0.2012 \\
PRDIM & \textbf{0.1776} & \textbf{0.1699} & 0.1523 & 0.0315 & \textbf{0.2120} & \textbf{0.1971} \\
\bottomrule
\end{tabular}}
\end{table}

\subsection{Different Evaluation Metrics for Imputation}
\label{app:wd_wf}

In addition to point-wise error metrics, we further evaluate the quality of imputed samples using distribution-level generative metrics, namely the exact 2-Wasserstein distance (WD) and the Fourier Wasserstein distance (FWD).

\paragraph{Wasserstein Distance} We computed the exact 2-Wasserstein distance using a existing Python library and report the updated results in Table~\ref{tab:gen_metrics}.
As shown in the table~\ref{tab:gen_metrics} (a), PRDIM consistently achieves the smallest $W_2$ discrepancy among diffusion-based imputation models across datasets, indicating that the generated samples from PRDIM are closer to the target data distribution in the sense of optimal transport.

\paragraph{Fourier Wasserstein Distance}
We additionally consider the Fourier Wasserstein distance (FWD)~\citep{cazelles2020wasserstein}, which measures the Wasserstein discrepancy between the normalized power spectral densities (NPSD) of two time-series distributions.
As discussed in the prior work, the FWD metric provides an interpretable measure of temporal misalignment while retaining the stability and geometric grounding of optimal transport in the spectral domain. Following this formulation, we extend the publicly available univariate FWD implementation to the multivariate setting and compute FWD distances for all diffusion-based imputation methods.
The corresponding results are also summarized in Table~\ref{tab:gen_metrics} (b).
Overall, PRDIM achieves competitive and consistently low FWD across most datasets.
We note that in certain cases, such as CSDI on the STOCK dataset or MTSCI on PEMS-Bay, alternative methods attain slightly lower FWD values.
One possible explanation is that models trained without an EM framework tend to generate less stochastic variability across diffusion trajectories.
While such reduced variability may lower the FWD, the additional sampling diversity induced by PRDIM is beneficial for modeling complex MNAR mechanisms, albeit at the cost of a marginally increased FWD metric in some settings.

\begin{table}[h]
\centering
\caption{Quantitative evaluation using generative quality metrics: (a) Wasserstein Distance (WD) and (b) Fourier Wasserstein Distance (Fourier WD) on three datasets.}
\label{tab:gen_metrics}
\vspace{0.2cm}
\small
\resizebox{1.0\linewidth}{!}{
\begin{tabular}{l|cc|cc|cc}
\multicolumn{7}{c}{\textbf{(a) Wasserstein Distance ($\downarrow$)}} \\
\toprule
Method & \multicolumn{2}{c|}{\textbf{ETT}} & \multicolumn{2}{c|}{\textbf{STOCK}} & \multicolumn{2}{c}{\textbf{PEMS-Bay}} \\
\cmidrule(lr){2-3} \cmidrule(lr){4-5} \cmidrule(lr){6-7}
 & In-Sample & Out-of-Sample & In-Sample & Out-of-Sample & In-Sample & Out-of-Sample \\
\midrule
CSDI       & 4.2535 & 11.7118 & 5.3450 & 4.8446 & 6.9161 & 7.5871 \\
MTSCI      & 4.3723 & 11.4063 & 5.6456 & 5.1334 & 7.4099 & 7.8256 \\
cDiffPuter & 3.6684 & 10.3398 & 4.1093 & 3.8013 & 7.5825 & 7.8158 \\
\textbf{PRDIM} & \textbf{3.2449} & \textbf{9.0755} & \textbf{3.3878} & \textbf{3.1422} & \textbf{6.6534} & \textbf{7.5517} \\
\bottomrule
\end{tabular}}

\vspace{0.5cm} 
\resizebox{1.0\linewidth}{!}{
\begin{tabular}{l|cc|cc|cc}
\multicolumn{7}{c}{\textbf{(b) Fourier Wasserstein Distance ($\downarrow$)}} \\
\toprule
Method & \multicolumn{2}{c|}{\textbf{ETT}} & \multicolumn{2}{c|}{\textbf{STOCK}} & \multicolumn{2}{c}{\textbf{PEMS-Bay}} \\
\cmidrule(lr){2-3} \cmidrule(lr){4-5} \cmidrule(lr){6-7}
 & In-Sample & Out-of-Sample & In-Sample & Out-of-Sample & In-Sample & Out-of-Sample \\
\midrule
CSDI       & 0.0988 & 0.3825 & \textbf{0.0532} & \textbf{0.0480} & 0.0230 & 0.0304 \\
MTSCI      & 0.1072 & 0.3489 & 0.0883 & 0.0794 & \textbf{0.0224} & \textbf{0.0290} \\
cDiffPuter & 0.0856 & 0.3233 & 0.0927 & 0.0841 & 0.0244 & 0.0313 \\
\textbf{PRDIM} & \textbf{0.0783} & \textbf{0.2213} & 0.0760 & 0.0694 & \textbf{0.0224} & 0.0294 \\
\bottomrule
\end{tabular}}
\end{table}

\subsection{Controlled Parameter Analysis}

To better understand the dynamics of PRDIM, we analyze the training behavior of the EM procedure. Figure \ref{fig:em-loss} shows the convergence of the pattern recognizer’s loss, where red curves indicate the ability to distinguish missing values and blue curves correspond to observed values.
The results on both the ETT and STOCK datasets demonstrate that the pattern recognizer effectively captures the missing pattern, thereby providing informative guidance during generation. Furthermore, Figure~\ref{fig:em-mre} illustrates the evolution of MRE across EM iterations on ETT and STOCK datasets, revealing a consistent improvement in imputation accuracy as the number of EM epochs increases.

\begin{figure*}[h] 
    \centering
    \begin{minipage}[t]{0.35\textwidth}
        \centering
        \includegraphics[width=\linewidth]{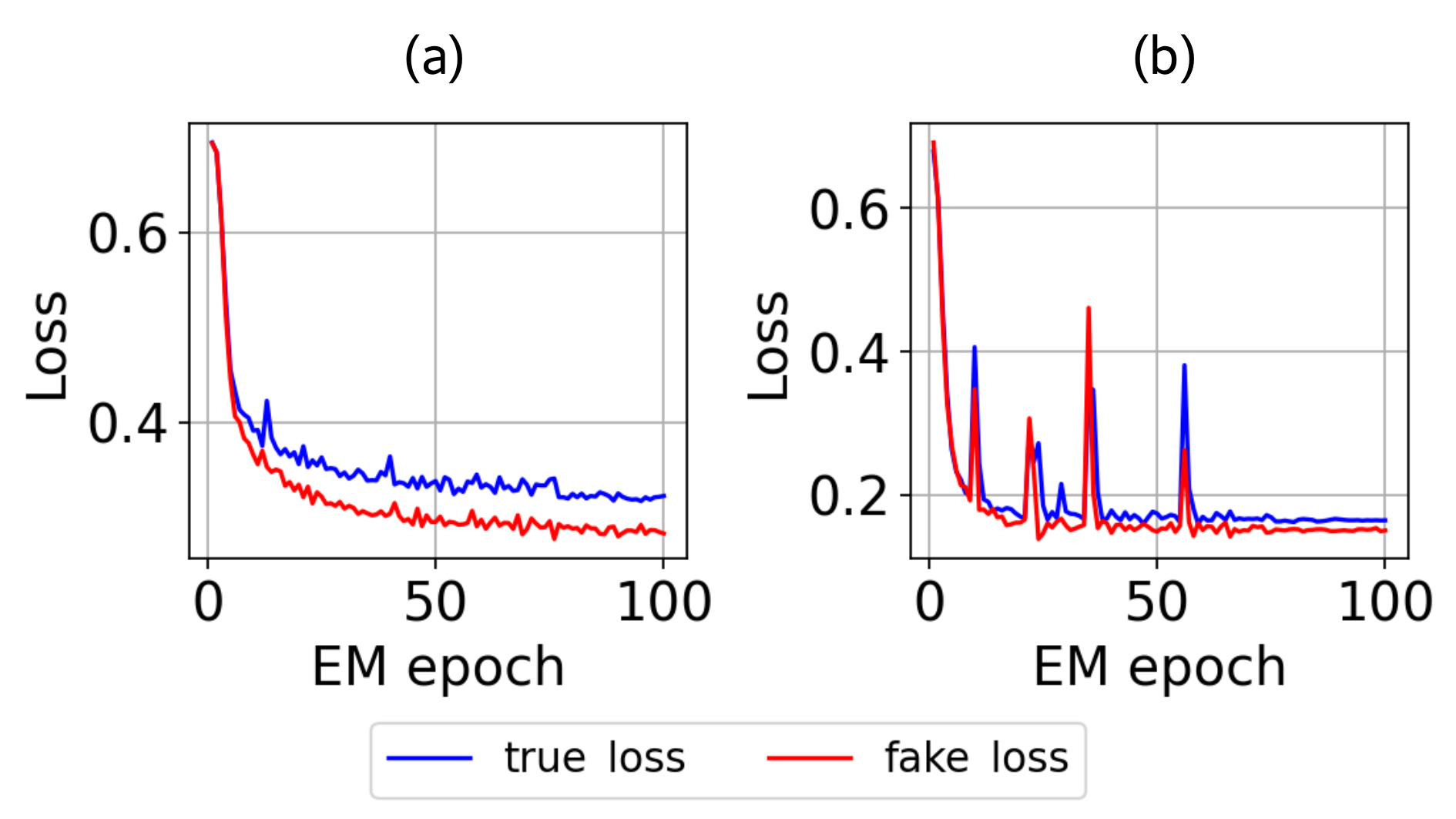}
        \caption{Convergence of the pattern recognizer’s loss during EM training. (a) ETT (b) STOCK.}
        \label{fig:em-loss}
    \end{minipage}
    \hfill 
    \begin{minipage}[t]{0.35\textwidth}
        \centering
        \includegraphics[width=\linewidth]{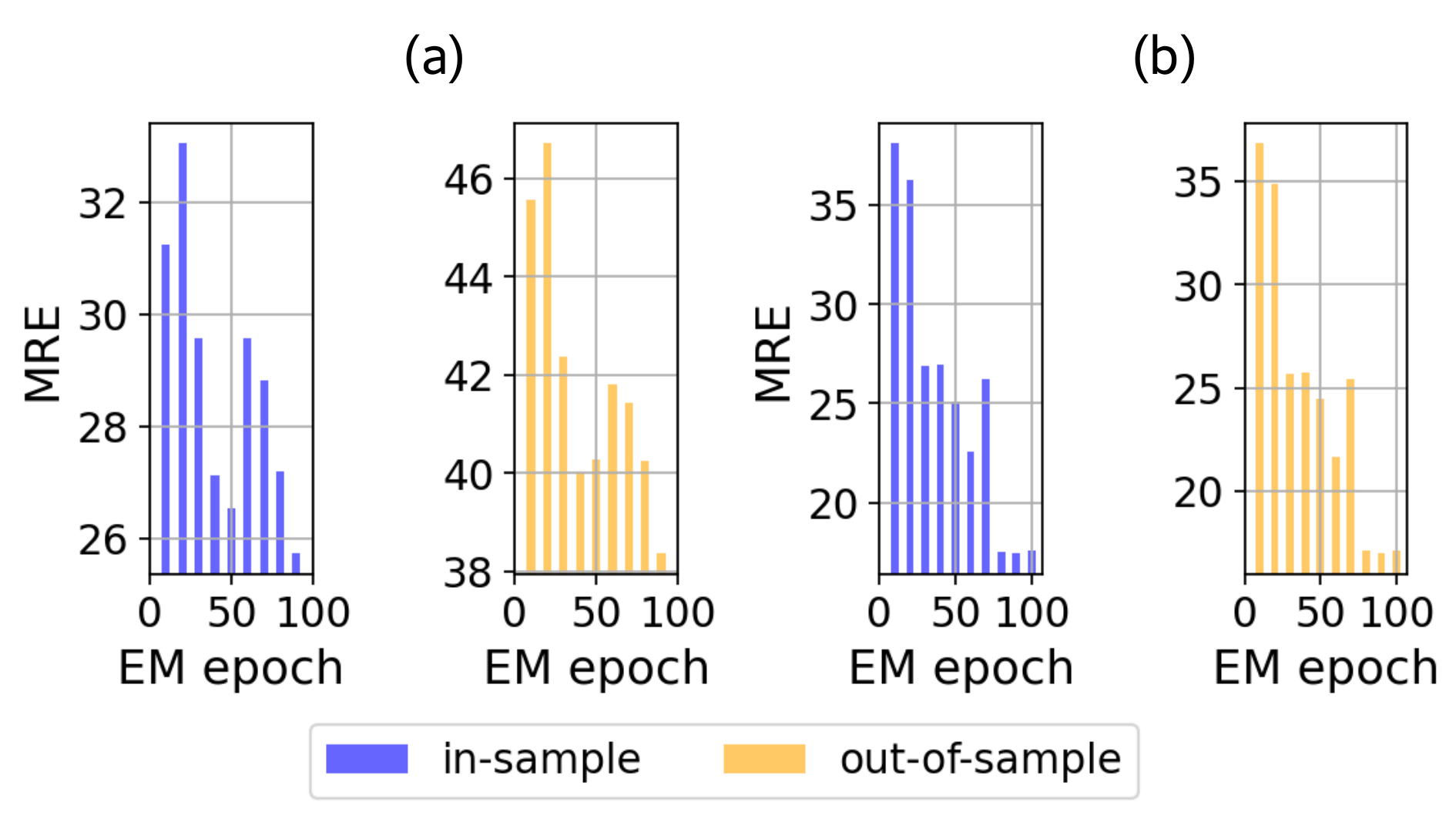}
        \caption{Evolution of MRE across EM epochs. (a) ETT (b) STOCK.}
        \label{fig:em-mre}
    \end{minipage}
    \hfill
    \begin{minipage}[t]{0.26\textwidth}
        \centering
        \includegraphics[width=\linewidth]{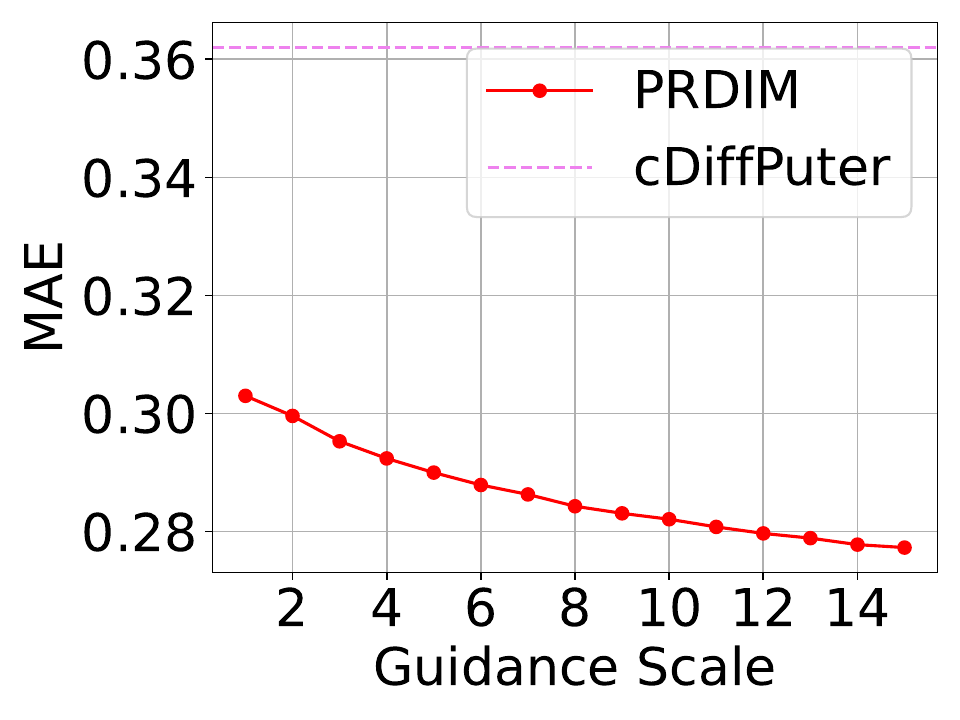}
        \caption{Effect of guidance scale on STOCK.}
        \label{fig:guidance_scale}
    \end{minipage}
\end{figure*}

 Moreover, according to the findings on a previous work~\citep{ho2022classifier}, increasing the weight of guidance generally leads to the better generation.
As shown in Figure~\ref{fig:guidance_scale}, the imputation performance consistently improves as the guidance scale increases, demonstrating the effectiveness of guidance weighting. These findings confirm that EM refinement is a critical component of PRDIM, substantially enhancing its capacity to model the joint distribution of the data and the missing pattern.
Despite the strong performance of our proposed PRDIM framework, it should be noted that introducing an additional guidance term in diffusion models inherently incurs extra computational cost~\citep{kim2022refining, chung2023diffusion}. To clarify this trade-off, we provide supplementary results in Table~\ref{tab:model_config} of Appendix~\ref{app:exp_details}, where training and inference times are compared across different diffusion-based imputation models.

\subsection{Qualitative Results Including Real-world Dataset}
\label{appc25}

To further illustrate the behavior of PRDIM under the out-of-sample imputation setting, we provide qualitative visualizations across ETT, STOCK, and PEMS-Bay on Figure~\ref{fig:qual_ett}, ~\ref{fig:qual_stock}, and ~\ref{fig:qual_pems} respectively. For each dataset, we randomly sample 4 test instances and display (i) the locations of missing values as yellow points and (ii) the corresponding imputation results produced by CSDI, MTSCI, cDiffPuter, MTSI, and PRDIM respectively. These visualization results allow for a direct visual comparison of reconstruction quality, highlighting the degree to which each model captures temporal structure and recovers unseen missing values.

To examine whether PRDIM can operate on data with naturally occurring missing values where ground-truth values for the missing entries are not available, we additionally conducted experiments on the PhysioNet~\citep{goldberger2000physiobank} dataset. Figure~\ref{fig:qual_physio} visualizes the out-of-sample imputation results of CSDI, cDiffPuter, and PRDIM on PhysioNet. We also provide the quantitative results under different missing ratios of out-of-sample imputation in Table~\ref{tab:physionet}.

One notable observation is that both cDiffPuter and PRDIM involve a joint optimization procedure over the latent missing variables $X_0^\text{mis}$ and the observed variables $X_0^{\text{obs}}$ during the EM iterations. As a consequence, when the natural missing rate is extremely high (approaching nearly 80\% in the PhysioNet dataset), the imputed values may become biased toward zero (\textit{i.e.} initial imputed value). This highlights an inherent limitation of EM-based diffusion imputation methods under severe natural missing.

\begin{table}[h]
    \centering
    \vspace{-0.5cm}
    \caption{MAE (mean$\pm$std) on 5 runs under different missing ratios of out-of-sample imputation.}
    \label{tab:physionet}
    \begin{tabular}{l|ccc}
        \toprule
        \textbf{Missing ratio} & \textbf{10\%} & \textbf{50\%} & \textbf{90\%} \\
        \midrule
        CSDI & \textbf{0.2733} $\pm$ 0.0002 & 0.3607 $\pm$ 0.0002 & 0.5191 $\pm$ 0.0003 \\
        DiffPuter & 0.2755 $\pm$ 0.0003 & 0.3471 $\pm$ 0.0001 & 0.4670 $\pm$ 0.0003 \\
        PRDIM & 0.2737 $\pm$ 0.0008 & \textbf{0.3384} $\pm$ 0.0001 & \textbf{0.4646} $\pm$ 0.0003 \\
        \bottomrule
    \end{tabular}
    \vspace{-0.5cm}
\end{table}

\section{Pattern Recognizer Analysis} 
\label{app:pr_analysis}

\subsection{Interpretability of the Pattern Recognizer with Case Study}
\label{appc24}

To evaluate whether the pattern recognizer $D_\phi$ trained under the EM iterations has effectively learned the underlying missing mechanism, we conducted a case study on three time-series datasets: ETT, STOCK, and PEMS-Bay. Specifically, we randomly sampled instances from the ETT dataset and plotted both the true missing ratio for each entry and the corresponding output of the trained pattern recognizer, $D_{\phi}(\hat{X}_0)$, where $\hat{X}_0$ represents the imputed sample generated through approximate guided generation by PRDIM. This visualization allows us to examine whether the learned $D_\phi$ accurately captures and mimics the MNAR missing patterns inherent in high dimensional data.

For each dataset, we sampled time intervals of length 72. The ETT and STOCK datasets contain 7 and 6 features, respectively, and we visualized all features for completeness. In the case of the PEMS-Bay dataset, which has a total of 325 feature dimensions, only the first 10 features were used for visualization due to its high dimensionality. The experimental results for ETT, STOCK, and PEMS-Bay are shown in Figure~\ref{tab:pr_case_study_ett}, Figure~\ref{tab:pr_case_study_stock}, and Figure~\ref{tab:pr_case_study_pems1}, ~\ref{tab:pr_case_study_pems2}, respectively.

Overall, the results demonstrate that while the pattern recognizer tends to slightly overestimate missing entries, it nonetheless captures the overall tendency and structure of the true missing pattern remarkably well, indicating its strong capability to model MNAR mechanisms.This interpretability analysis provides empirical evidence that the pattern recognizer contributes meaningful guidance during the generation phase.

\subsection{Control the Pattern Recognizer's Capacity}
\label{pr_analysis:control_pr}

To systematically analyze the impact of the pattern recognizer's capacity on the imputation, we reduced its parameter size, thereby restricting the solution space and forcing it to converge to less accurate estimates. We then measured the average cross-entropy on observed entries ($M_{i,j}=1$) and on missing entries ($M_{i,j}=0$). A larger average cross-entropy indicates a poorer ability to discriminate missing patterns. For example, when the value approaches $-\ln{0.5}\simeq0.693$, the recognizer effectively collapses to random guessing and fails to capture any meaningful missing pattern.

Using this setup, we evaluated PRDIM on the ETT and STOCK datasets while keeping all other conditions identical and varying only the parameter number of pattern recognizer. Below Table~\ref{tab:pr_ablation} summarizes the results. As expected, higher cross-entropy corresponds to degraded imputation accuracy. However, it is noteworthy that PRDIM still outperforms cDiffPuter even with significantly weakened pattern recognizers. We can demonstrate that the pattern recognizer provides meaningful guidance to the diffusion process when properly learned.

\begin{table}[h]
\centering
\vspace{-0.5cm}
\caption{Impact of Pattern Recognizer parameter size on imputation performance. We report average Cross Entropy (CE) and MAE on ETT and STOCK test datasets.}
\label{tab:pr_ablation}
\vspace{0.1cm}
\resizebox{0.95\linewidth}{!}{
\begin{tabular}{ccccc}
\toprule
\textbf{Number of parameters} & \textbf{CE (Observable Values)} & \textbf{CE (Missing Values)} & \textbf{MAE (Out)} & \textbf{MAE (In)} \\
\midrule
\multicolumn{5}{c}{\textbf{ETT Dataset}} \\
\midrule
17,376 & \textbf{0.323} & \textbf{0.283} & \textbf{0.663} & \textbf{0.303} \\
4,608  & 0.339 & 0.311 & 0.744 & 0.354 \\
1,296  & 0.480 & 0.456 & 0.763 & 0.360 \\
\midrule
\multicolumn{5}{c}{\textbf{STOCK Dataset}} \\
\midrule
17,359 & \textbf{0.165} & \textbf{0.151} & \textbf{0.254} & \textbf{0.275} \\
4,599  & 0.168 & 0.155 & 0.271 & 0.301 \\
1,291  & 0.169 & 0.162 & 0.277 & 0.303 \\
\bottomrule
\end{tabular}}
\end{table}

A fully non-trained recognizer provides no gain to the score function, consistent with Table 1 of Kim et al.~\citep{kim2022refining}, which shows that discriminator (or guidance) collapse recovers the unguided diffusion model. Thereby explaining why the theoretical gap between PRDIM and DiffPuter naturally diminishes under MCAR settings.

\subsection{Impact of Pattern Recognizer across Intermediate Training Process on Imputation}
\label{pr_analysis:intermediate_pr}

To further analyze the sensitivity of the guidance effect relative to the training progress of the pattern recognizer and degradation effect of a randomly initialized pattern recognizer, we evaluate imputation performance as a function of the training progress of the pattern recognizer under six different settings in Figure~\ref{tab:intermediate_pr}. Across all settings, we observe that as the pattern recognizer is trained, imputation performance consistently improves. However, the performance curves stabilize early, and further improvements in the recognizer yield diminishing returns.

\begin{table}[h]
\centering
\vspace{-0.5cm}
\caption{Ablation study on the effect of the pattern recognizer across different missing configurations and intermediate checkpoints. $n\%$ denotes the percentile of entire training procedure. Each column shows $(W,b)$ with corresponding train/test missing ratios (\%). We report MAE($\downarrow$) of Out-of-Sample imputation.}
\vspace{0.2cm}
\resizebox{0.75\linewidth}{!}{
\begin{tabular}{l|ccc|ccc}
\toprule
Method & \multicolumn{3}{c|}{ETT} & \multicolumn{3}{c}{STOCK} \\
\cmidrule(r){2-4} \cmidrule(l){5-7}
& \shortstack{(5,0.8) \\ \scriptsize 21.4 / 43.9}
& \shortstack{(5,0.5) \\ \scriptsize 32.2 / 57.1}
& \shortstack{(5,0.2) \\ \scriptsize 44.3 / 70.1}
& \shortstack{(5,0.8) \\ \scriptsize 21.2 / 20.0}
& \shortstack{(5,0.5) \\ \scriptsize 26.6 / 25.8}
& \shortstack{(5,0.2) \\ \scriptsize 33.5 / 32.9} \\
\midrule
w/o PR & 0.697 & 0.795 & 0.914 & 0.306 & 0.338 & 0.392 \\
random PR & 0.697 & 0.799 & 0.916 & 0.312 & 0.339 & 0.392 \\
20\% & 0.668 & 0.742 & 0.787 & 0.254 & 0.331 & 0.386 \\
40\% & 0.665 & 0.741 & 0.769 & 0.254 & 0.332 & \textbf{0.383} \\
60\% & 0.665 & 0.738 & 0.770 & 0.255 & 0.328 & 0.384 \\
80\% & 0.664 & 0.738 & \textbf{0.766} & \textbf{0.253} & 0.329 & \textbf{0.383} \\
PRDIM & \textbf{0.663} & \textbf{0.737} & \textbf{0.766} & 0.254 & \textbf{0.328} & \textbf{0.383} \\
\bottomrule
\end{tabular}}
\label{tab:intermediate_pr}
\end{table}

\section{Future Works} 
\label{app:limitations}

\subsection{Discrete State Diffusion for Imputation}

In this work, we only focus on continuous-state diffusion imputation models. However, real-world datasets encountered in practical applications often comprise a mixture of continuous and discrete features. To effectively handle such data, it is necessary to consider the application of diffusion models in discrete domains. Recent works, such as D3PM~\cite{austin2021structured}, have introduced diffusion frameworks for discrete state spaces, and several studies have explored their application to imputation tasks~\cite{hoogeboom2021argmax, shi2024tabdiff}.

Nevertheless, introducing a pattern recognizer under MNAR settings within a discrete diffusion framework entails non-trivial formulation challenges as well as practical considerations. We leave the formal derivation and implementation of the MNAR-aware discrete diffusion framework for future work.

\subsection{Orthogonal Refinement: Prior Distribution Optimization}

In addition to extending the model to discrete settings, we can also consider orthogonal directions for improving imputation performance within the diffusion framework. Moving beyond the imputation task, we investigate the prior gap inherent in the diffusion framework through an empirical study, defined as the potential discrepancy in the prior loss $\mathcal{L}_T=\mathbb{E}_{q(X_0)}[D_{KL}(q(X_T|X_0)||p(X_T))]$. In our diffusion design, we would like to demonstrate the practical effect of this gap turns out to be minimal. Before training, all datasets undergo standard normalization, which centers the empirical mean of $X_0$ extremely close to zero. Combined with the forward diffusion coefficients adopted from the CSDI implementation: $\beta_{min}=10^{-4}, \beta_{max}=0.5, \text{diffusion steps}=50$ with a quadratic schedule. (We can obtain $\bar{\alpha}_T\equiv\prod_{t=1}^T(1-\beta_t)=3.354\times10^{-5})$

It enables us to compute the mean $\mu_1$ and variance $\sigma_1^2$ of the terminal forward process $q(X_T|X_0)$ Assuming dimensional independence, we can calculate the closed-form KL divergence between the forward terminal distribution and the standard Gaussian prior.

To further examine the empirical effect of this mismatch, we conducted an additional experiment. Instead of sampling $X_T$ from the standard prior $p(X_T)=\mathcal{N}(0,I)$, we sampled from the data-induced terminal distribution $\mathbb{E}_{q(X_0)}[q(X_T|X_0)]$ and performed inference with the same pre-trained diffusion model. The results are summarized in the Table~\ref{tab:prior_opt}. The results show that the imputation performance remains virtually unchanged across all datasets, indicating that the prior loss $\mathcal{L}_T$ has negligible impact on the diffusion inference process in our setting.

This suggests that, beyond the current design, PRDIM can be further improved by incorporating additional mathematical techniques to reduce imputation error (or reduce training and inference time via a controllled guidance mechanism) in some orthogonal manners.

\begin{table}[h]
\centering
\caption{Out-of-Sample MAE and statistics between default diffusion sampling setting and prior optimized setting.}
\label{tab:prior_opt}
\begin{tabular}{l|ccc}
\toprule
& \textbf{ETT} & \textbf{STOCK} & \textbf{PEMS-Bay} \\
\midrule
Mean $\mu_1$ 
& $-1.532 \times 10^{-3}$ 
& $-5.95 \times 10^{-4}$ 
& $5.32 \times 10^{-4}$ \\

Variance $\sigma_1^2$ 
& $1.001236$ 
& $1.001055$ 
& $1.000088$ \\

$D_{\mathrm{KL}}\!\left(\mathbb{E}_{q(X_0)}[q(X_T|X_0)] \,\|\, \mathcal{N}(0,I)\right)$ 
& $2.61 \times 10^{-4}$ 
& $6.60 \times 10^{-5}$ 
& $5.59 \times 10^{-4}$ \\

MAE $(p(X_T)\ \text{prior})$ 
& $0.663$ 
& $0.254$ 
& $0.170$ \\

MAE $(\mathbb{E}_{q(X_0)}[q(X_T|X_0)]\ \text{prior})$ 
& $0.663$ 
& $0.255$ 
& $0.172$ \\

\bottomrule
\end{tabular}
\end{table}

\newpage

\begin{table}[t]
\centering
\caption{Overall performance on the \textbf{ETT} dataset. We report $mean\pm std$ over 5 runs according to each methodology. Best results are in \textbf{bold}. Second best results are in \underline{underline}.}
\vspace{0.2cm}
\label{tab:ett_results}
\small
\begin{tabular}{lccc|ccc}
\toprule
Method
& \multicolumn{3}{c|}{Original / Out-of-Sample} 
& \multicolumn{3}{c}{Original / In-Sample} \\
\cmidrule(lr){2-4} \cmidrule(lr){5-7}
& RMSE ($\downarrow$) & MAE ($\downarrow$) & MRE ($\downarrow$) & RMSE ($\downarrow$) & MAE ($\downarrow$) & MRE ($\downarrow$) \\
\midrule
Mean        & 2.307{\scriptsize$\pm$0.000} & 2.034{\scriptsize$\pm$0.000} & 120.233{\scriptsize$\pm$0.000} & 1.618{\scriptsize$\pm$0.000} & 1.486{\scriptsize$\pm$0.000} & 127.379{\scriptsize$\pm$0.000} \\
\midrule
\multicolumn{7}{l}{\textit{Discriminative models}} \\
TimesNet 
            & 1.393{\scriptsize$\pm$0.038}& 1.044{\scriptsize$\pm$0.065}& 69.040{\scriptsize$\pm$4.305}& 1.485{\scriptsize$\pm$0.044}& 1.154{\scriptsize$\pm$0.068}& 72.065{\scriptsize$\pm$4.222}\\
TimeMixer++
            & 1.965{\scriptsize$\pm$0.012} & 1.642{\scriptsize$\pm$0.025} & 97.093{\scriptsize$\pm$0.015}  & 1.283{\scriptsize$\pm$0.015} & 1.100{\scriptsize$\pm$0.032} & 94.319{\scriptsize$\pm$0.028} \\
BRITS       & 1.461{\scriptsize$\pm$0.048} & 0.992{\scriptsize$\pm$0.037} & 58.600{\scriptsize$\pm$0.022}  & 0.850{\scriptsize$\pm$0.020} & 0.491{\scriptsize$\pm$0.008} & 42.067{\scriptsize$\pm$0.007} \\
SAITS       & 1.247{\scriptsize$\pm$0.069} & 0.814{\scriptsize$\pm$0.046} & 48.119{\scriptsize$\pm$0.027}  & 0.626{\scriptsize$\pm$0.018} & 0.366{\scriptsize$\pm$0.014} & 31.417{\scriptsize$\pm$0.012} \\
\midrule
\multicolumn{7}{l}{\textit{Generative models}} \\
GP-VAE
            & 1.915{\scriptsize$\pm$0.006}& 1.511{\scriptsize$\pm$0.011}& 89.315{\scriptsize$\pm$0.638}& 1.147{\scriptsize$\pm$0.008}& 0.896{\scriptsize$\pm$0.018}& 76.809{\scriptsize$\pm$1.507}\\
not-MIWAE
            & 1.781{\scriptsize$\pm$0.012} & 1.311{\scriptsize$\pm$0.016} & 77.512{\scriptsize$\pm$0.972}  & 0.945{\scriptsize$\pm$0.010} & 0.637{\scriptsize$\pm$0.011} & 54.643{\scriptsize$\pm$0.918} \\
\midrule
\multicolumn{7}{l}{\textit{Diffusion-based models}} \\
CSDI        & 1.658{\scriptsize$\pm$0.001} & 1.071{\scriptsize$\pm$0.001} & 63.254{\scriptsize$\pm$0.038}  & 0.822{\scriptsize$\pm$0.001} & 0.522{\scriptsize$\pm$0.001} & 44.733{\scriptsize$\pm$0.049} \\
MTSCI       & 1.335{\scriptsize$\pm$0.001} & 0.957{\scriptsize$\pm$0.001} & 56.574{\scriptsize$\pm$0.036}  & 0.730{\scriptsize$\pm$0.000} & 0.500{\scriptsize$\pm$0.000} & 42.827{\scriptsize$\pm$0.018} \\
cDiffPuter  & \underline{1.209}{\scriptsize$\pm$0.001}& \underline{0.782}{\scriptsize$\pm$0.000}& \underline{46.188}{\scriptsize$\pm$0.020}& \underline{0.612}{\scriptsize$\pm$0.001}& \underline{0.362}{\scriptsize$\pm$0.000}& \underline{31.069}{\scriptsize$\pm$0.020}\\
\midrule
\textbf{PRDIM} 
            & \textbf{1.057}{\scriptsize$\pm$0.000} & \textbf{0.663}{\scriptsize$\pm$0.000} & \textbf{39.156}{\scriptsize$\pm$0.009} 
            & \textbf{0.538}{\scriptsize$\pm$0.000} & \textbf{0.303}{\scriptsize$\pm$0.000} & \textbf{25.986}{\scriptsize$\pm$0.015} \\
\bottomrule
\end{tabular}
\end{table}

\begin{table}[t]
\centering
\caption{Overall performance on the \textbf{STOCK} dataset. We report $mean\pm std$ over 5 runs according to each methodology. Best results are in \textbf{bold}. Second best results are in \underline{underline}.}
\vspace{0.2cm}
\label{tab:stock_results}
\small
\begin{tabular}{lccc|ccc}
\toprule
Method
& \multicolumn{3}{c|}{Original / Out-of-Sample} 
& \multicolumn{3}{c}{Original / In-Sample} \\
\cmidrule(lr){2-4} \cmidrule(lr){5-7}
& RMSE ($\downarrow$) & MAE ($\downarrow$) & MRE ($\downarrow$) & RMSE ($\downarrow$) & MAE ($\downarrow$) & MRE ($\downarrow$) \\
\midrule
Mean        & 2.079{\scriptsize$\pm$0.000} & 1.949{\scriptsize$\pm$0.000} & 128.903{\scriptsize$\pm$0.000} & 2.168{\scriptsize$\pm$0.000} & 2.039{\scriptsize$\pm$0.000} & 127.313{\scriptsize$\pm$0.000} \\
\midrule
\multicolumn{7}{l}{\textit{Discriminative models}} \\
TimesNet 
            & 1.415{\scriptsize$\pm$0.054}& 1.111{\scriptsize$\pm$0.073}& 73.528{\scriptsize$\pm$0.049}& 1.509{\scriptsize$\pm$0.057}& 1.221{\scriptsize$\pm$0.077}& 76.237{\scriptsize$\pm$0.048}\\
TimeMixer++ 
            & 1.490{\scriptsize$\pm$0.223} & 1.287{\scriptsize$\pm$0.239} & 85.153{\scriptsize$\pm$0.158}  & 1.569{\scriptsize$\pm$0.239} & 1.369{\scriptsize$\pm$0.260} & 85.456{\scriptsize$\pm$0.162} \\
BRITS       & 0.953{\scriptsize$\pm$0.016} & 0.627{\scriptsize$\pm$0.010} & 41.478{\scriptsize$\pm$0.006}  & 1.020{\scriptsize$\pm$0.016} & 0.701{\scriptsize$\pm$0.010} & 43.757{\scriptsize$\pm$0.006} \\
SAITS       & 0.743{\scriptsize$\pm$0.021} & 0.442{\scriptsize$\pm$0.022} & 29.115{\scriptsize$\pm$0.015}  & 0.801{\scriptsize$\pm$0.023} & 0.498{\scriptsize$\pm$0.025} & 31.071{\scriptsize$\pm$0.016} \\
\midrule
\multicolumn{7}{l}{\textit{Generative models}} \\
GP-VAE
            & 1.239{\scriptsize$\pm$0.118}& 0.902{\scriptsize$\pm$0.109}& 59.684{\scriptsize$\pm$7.220}& 1.333{\scriptsize$\pm$0.123}& 1.010{\scriptsize$\pm$0.118}& 63.046{\scriptsize$\pm$7.371}\\
not-MIWAE
            & 1.028{\scriptsize$\pm$0.043} & 0.681{\scriptsize$\pm$0.045} & 45.039{\scriptsize$\pm$0.296}  & 1.114{\scriptsize$\pm$0.043} & 0.759{\scriptsize$\pm$0.039} & 47.368{\scriptsize$\pm$0.243} \\
\midrule
\multicolumn{7}{l}{\textit{Diffusion-based models}} \\
CSDI        & 0.932{\scriptsize$\pm$0.000} & 0.641{\scriptsize$\pm$0.000} & 42.393{\scriptsize$\pm$0.004}  & 0.995{\scriptsize$\pm$0.000} & 0.710{\scriptsize$\pm$0.000} & 44.330{\scriptsize$\pm$0.001} \\
MTSCI       & 0.988{\scriptsize$\pm$0.002} & 0.736{\scriptsize$\pm$0.001} & 48.629{\scriptsize$\pm$0.009}  & 1.056{\scriptsize$\pm$0.001} & 0.809{\scriptsize$\pm$0.001} & 50.485{\scriptsize$\pm$0.041} \\
cDiffPuter  & \underline{0.734}{\scriptsize$\pm$0.000}& \underline{0.406}{\scriptsize$\pm$0.000}& \underline{26.878}{\scriptsize$\pm$0.100}& \underline{0.778}{\scriptsize$\pm$0.000}& \underline{0.450}{\scriptsize$\pm$0.000}& 28.064{\scriptsize$\pm$0.008} \\
\midrule
\textbf{PRDIM} 
            & \textbf{0.599}{\scriptsize$\pm$0.001} & \textbf{0.254}{\scriptsize$\pm$0.000} & \textbf{16.794}{\scriptsize$\pm$0.027} 
            & \textbf{0.633}{\scriptsize$\pm$0.000} & \textbf{0.275}{\scriptsize$\pm$0.000} & \textbf{17.150}{\scriptsize$\pm$0.008} \\
\bottomrule
\end{tabular}
\end{table}

\begin{table}[t]
\centering
\caption{Overall performance on the \textbf{PEMS-Bay} dataset. We report $mean\pm std$ over 5 runs according to each methodology. Best results are in \textbf{bold}. Second best results are in \underline{underline}.}
\vspace{0.2cm}
\label{tab:pems_results}
\small
\begin{tabular}{lccc|ccc}
\toprule
Method
& \multicolumn{3}{c|}{Original / Out-of-Sample} 
& \multicolumn{3}{c}{Original / In-Sample} \\
\cmidrule(lr){2-4} \cmidrule(lr){5-7}
& RMSE ($\downarrow$) & MAE ($\downarrow$) & MRE ($\downarrow$) & RMSE ($\downarrow$) & MAE ($\downarrow$) & MRE ($\downarrow$) \\
\midrule
Mean        & 0.901{\scriptsize$\pm$0.000} & 0.813{\scriptsize$\pm$0.000} & 119.064{\scriptsize$\pm$0.000} & 0.868{\scriptsize$\pm$0.000} & 0.789{\scriptsize$\pm$0.000} & 119.066{\scriptsize$\pm$0.000} \\
\midrule
\multicolumn{7}{l}{\textit{Discriminative models}} \\
TimesNet 
            & 0.481{\scriptsize$\pm$0.009} & 0.291{\scriptsize$\pm$0.007} & 42.579{\scriptsize$\pm$0.010}  & 0.392{\scriptsize$\pm$0.004} & 0.225{\scriptsize$\pm$0.001} & 33.970{\scriptsize$\pm$0.002} \\
TimeMixer++ 
            & 0.684{\scriptsize$\pm$0.013} & 0.579{\scriptsize$\pm$0.018} & 84.816{\scriptsize$\pm$0.026}  & 0.652{\scriptsize$\pm$0.015} & 0.557{\scriptsize$\pm$0.020} & 84.113{\scriptsize$\pm$0.030} \\
BRITS       & 0.503{\scriptsize$\pm$0.007} & 0.278{\scriptsize$\pm$0.006} & 40.758{\scriptsize$\pm$0.008}  & 0.342{\scriptsize$\pm$0.004} & 0.182{\scriptsize$\pm$0.003} & 27.490{\scriptsize$\pm$0.004} \\
SAITS       & 0.481{\scriptsize$\pm$0.011} & 0.302{\scriptsize$\pm$0.009} & 44.266{\scriptsize$\pm$0.013}  & 0.356{\scriptsize$\pm$0.004} & 0.212{\scriptsize$\pm$0.003} & 31.970{\scriptsize$\pm$0.005} \\
\midrule
\multicolumn{7}{l}{\textit{Generative models}} \\
GP-VAE
            & 0.537{\scriptsize$\pm$0.001}& 0.345{\scriptsize$\pm$0.001}& 50.561{\scriptsize$\pm$0.208}& 0.470{\scriptsize$\pm$0.003}& 0.292{\scriptsize$\pm$0.002}& 44.039{\scriptsize$\pm$0.367}\\
not-MIWAE
            & 0.623{\scriptsize$\pm$0.006}& 0.396{\scriptsize$\pm$0.005}& 57.510{\scriptsize$\pm$0.774}& 0.608{\scriptsize$\pm$0.004}& 0.352{\scriptsize$\pm$0.005}& 53.181{\scriptsize$\pm$0.746}\\
\midrule
\multicolumn{7}{l}{\textit{Diffusion-based models}} \\
CSDI        & \underline{0.338}{\scriptsize$\pm$0.002}& \underline{0.177}{\scriptsize$\pm$0.000}& \underline{25.912}{\scriptsize$\pm$0.017}& \textbf{0.302}{\scriptsize$\pm$0.000} & \underline{0.158}{\scriptsize$\pm$0.000}& \underline{23.910}{\scriptsize$\pm$0.005}\\
MTSCI       & 0.349{\scriptsize$\pm$0.000}& 0.193{\scriptsize$\pm$0.000} & 28.289{\scriptsize$\pm$0.011}  & 0.322{\scriptsize$\pm$0.000} & 0.179{\scriptsize$\pm$0.000} & 27.017{\scriptsize$\pm$0.003} \\
cDiffPuter  & 0.349{\scriptsize$\pm$0.007} & 0.182{\scriptsize$\pm$0.000} & 26.714{\scriptsize$\pm$0.011}  & 0.330{\scriptsize$\pm$0.000} & 0.168{\scriptsize$\pm$0.000} & 25.377{\scriptsize$\pm$0.005} \\
\midrule
\textbf{PRDIM} 
            & \textbf{0.334}{\scriptsize$\pm$0.002} & \textbf{0.170}{\scriptsize$\pm$0.000} & \textbf{24.966}{\scriptsize$\pm$0.015} 
            & \underline{0.306}{\scriptsize$\pm$0.000}& \textbf{0.154}{\scriptsize$\pm$0.000} & \textbf{23.304}{\scriptsize$\pm$0.006} \\
\bottomrule
\end{tabular}
\end{table}

\begin{figure}[t]
    \centering
    \setlength{\tabcolsep}{2pt}
    \renewcommand{\arraystretch}{1.0}
    \begin{tabular}{ccccc}
        & & FID: 20.596& FID: 10.762 & FID: 10.376\\
        \includegraphics[width=0.18\textwidth]{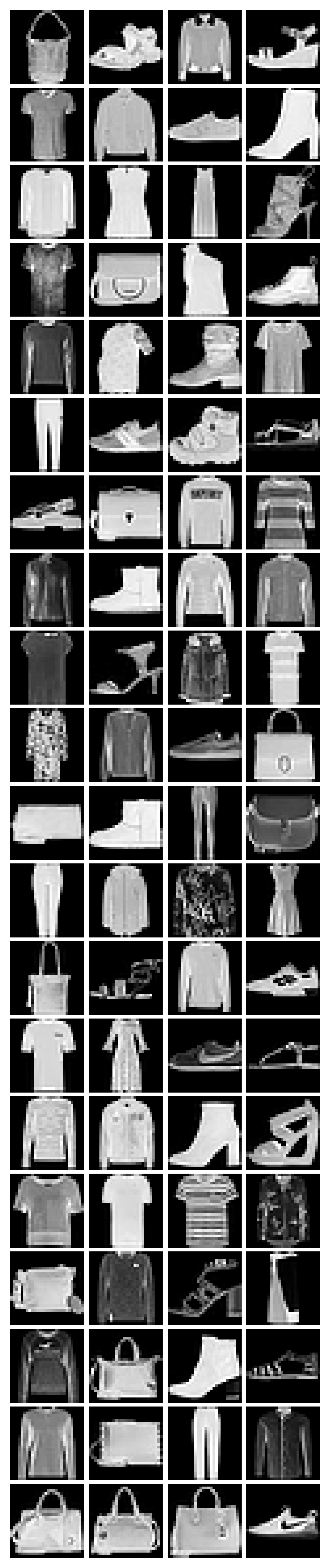} &
        \includegraphics[width=0.18\textwidth]{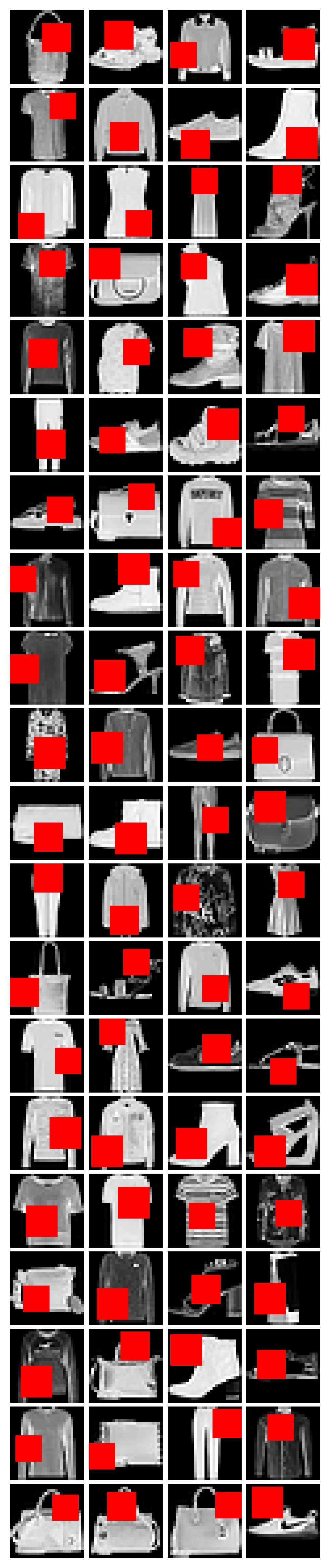} &
        \includegraphics[width=0.18\textwidth]{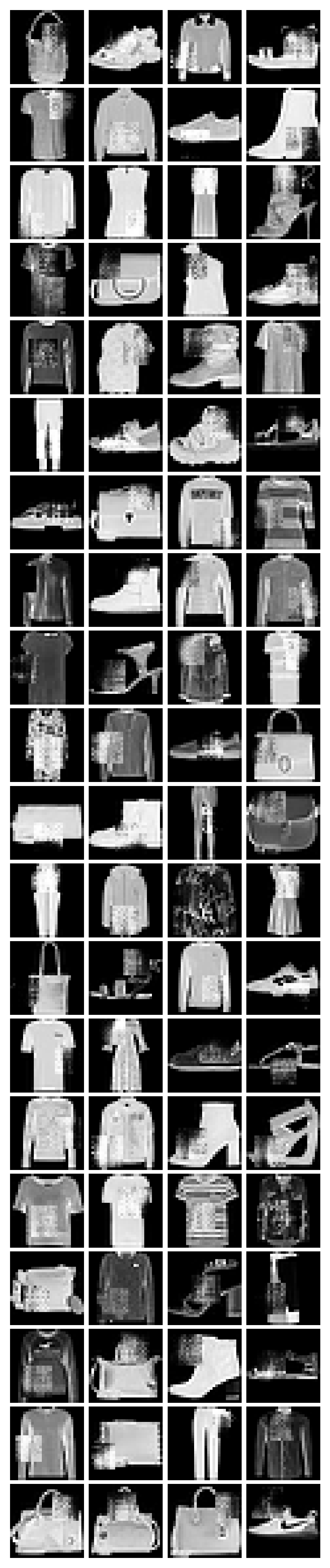} &
        \includegraphics[width=0.18\textwidth]{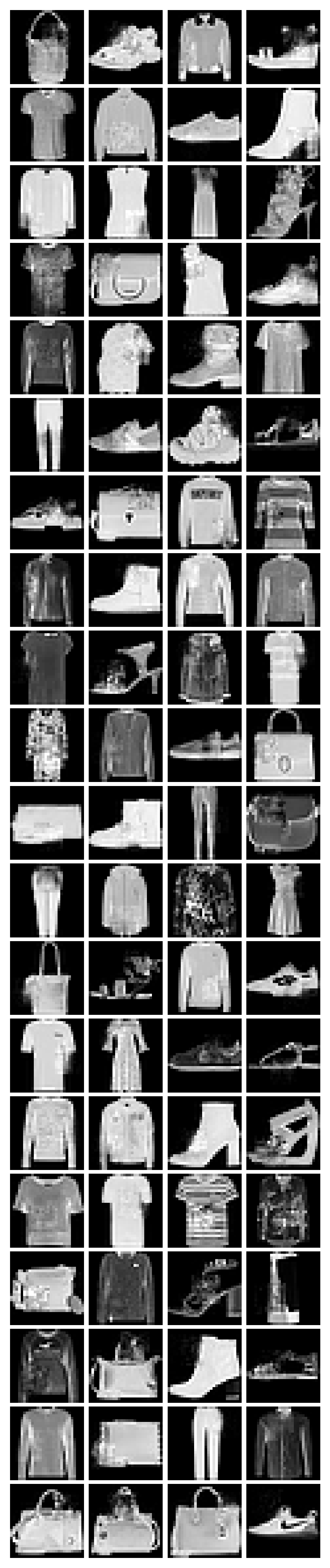} &
        \includegraphics[width=0.18\textwidth]{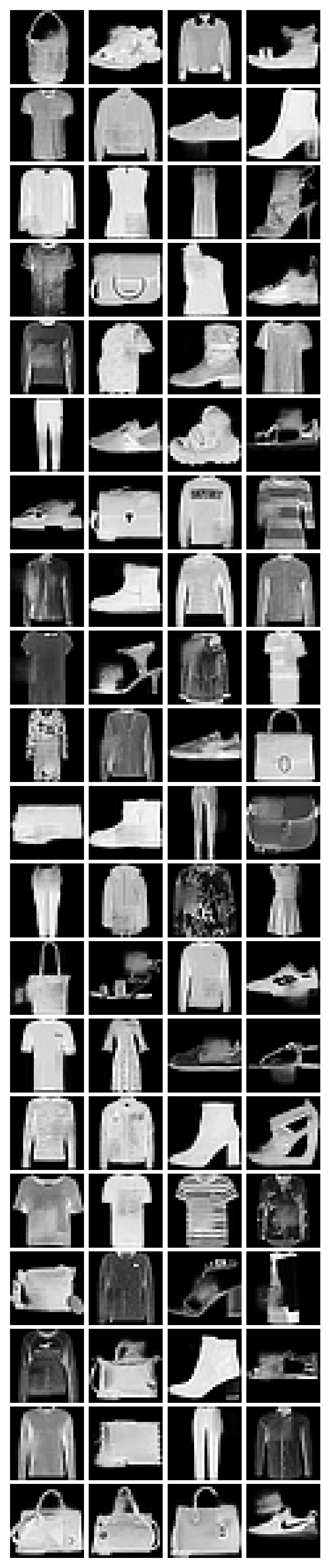} \\
        (a) Ground Truth & (b) Observed Input & (c) misGAN & (d) MCFlow & (e) PRDIM \\
    \end{tabular}
    \caption{Comparison of out-of-sample imputation results under block missing pattern.}
    \label{fig:app_fmnist_mar}
\end{figure}

\begin{figure}[t]
    \centering
    \setlength{\tabcolsep}{2pt}
    \renewcommand{\arraystretch}{1.0}
    \begin{tabular}{ccccc}
        & & FID: 198.328& FID: 32.893& FID: 15.893\\
        \includegraphics[width=0.18\textwidth]{figures/save_gt_subplot.png} &
        \includegraphics[width=0.18\textwidth]{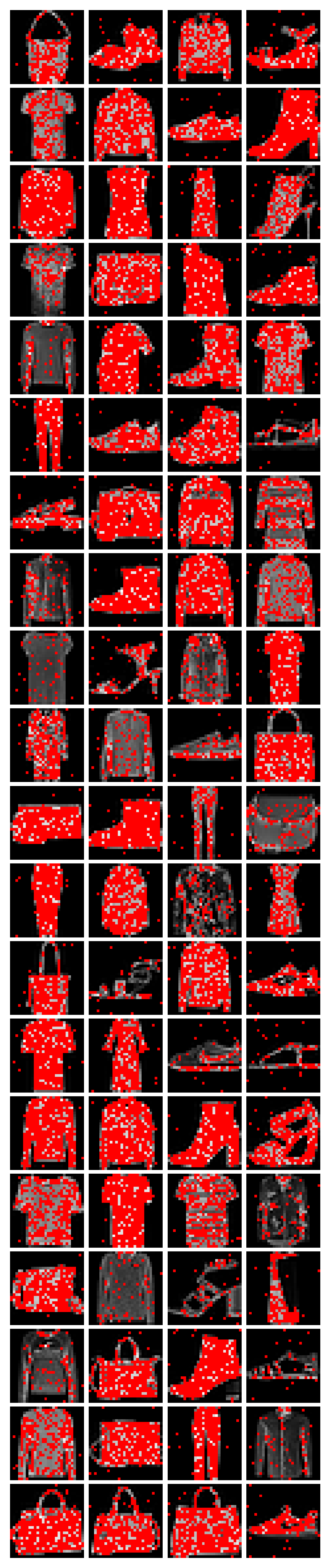} &
        \includegraphics[width=0.18\textwidth]{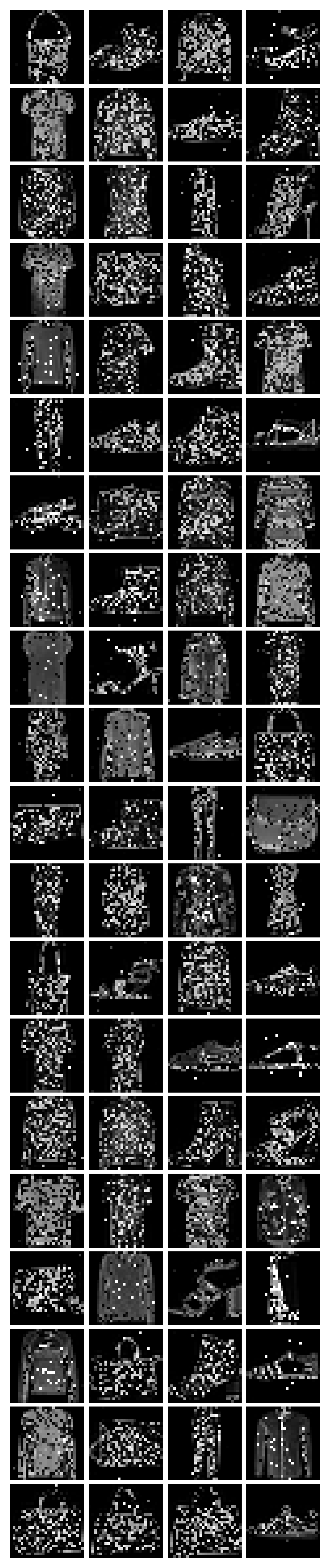} &
        \includegraphics[width=0.18\textwidth]{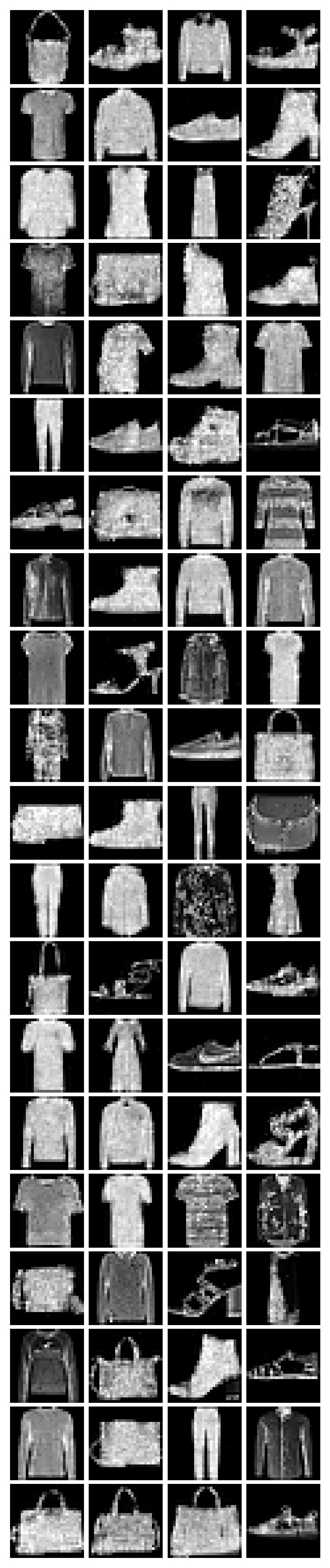} &
        \includegraphics[width=0.18\textwidth]{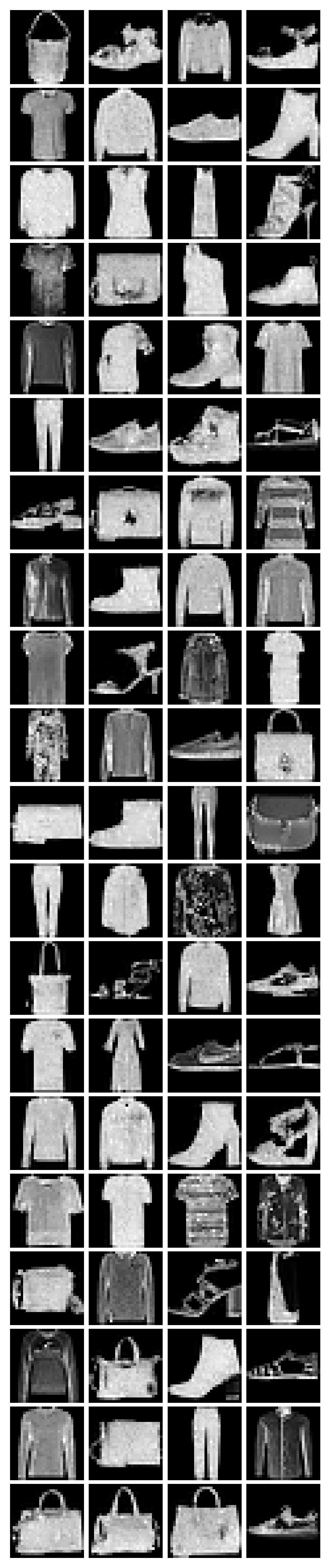} \\
        (a) Ground Truth & (b) Observed Input & (c) misGAN & (d) MCFlow & (e) PRDIM \\
    \end{tabular}
    \caption{Comparison of out-of-sample imputation results under MNAR missing pattern.}
    \label{fig:app_fmnist_mnar}
\end{figure}

\begin{figure*}[t]
\vspace{-0.3cm}
\centering
\begin{tabular}{cc}
    \subfloat{%
        \includegraphics[width=0.37\textwidth]{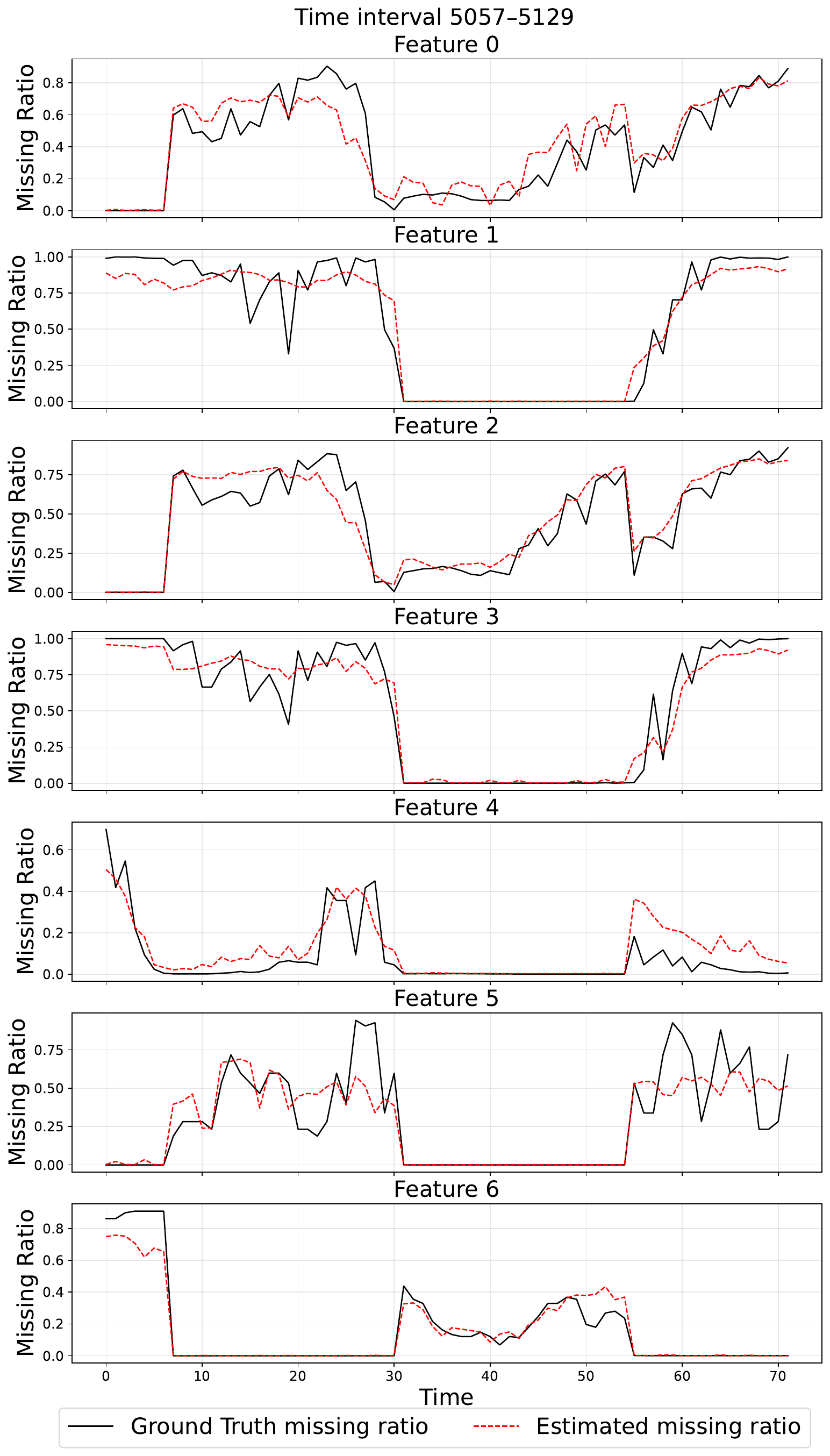}} &
    \subfloat{%
        \includegraphics[width=0.37\textwidth]{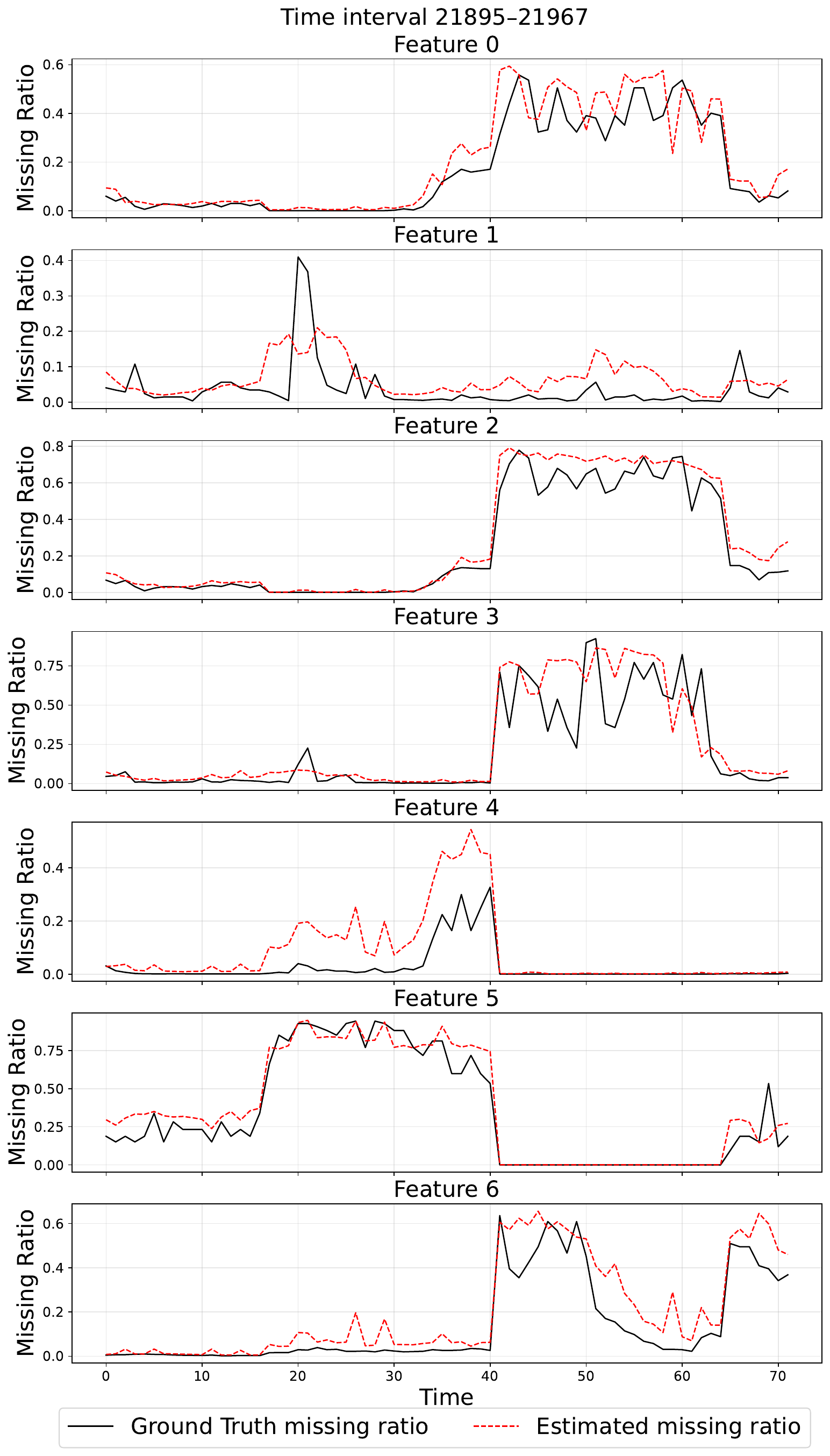}} \\
    \subfloat{%
        \includegraphics[width=0.37\textwidth]{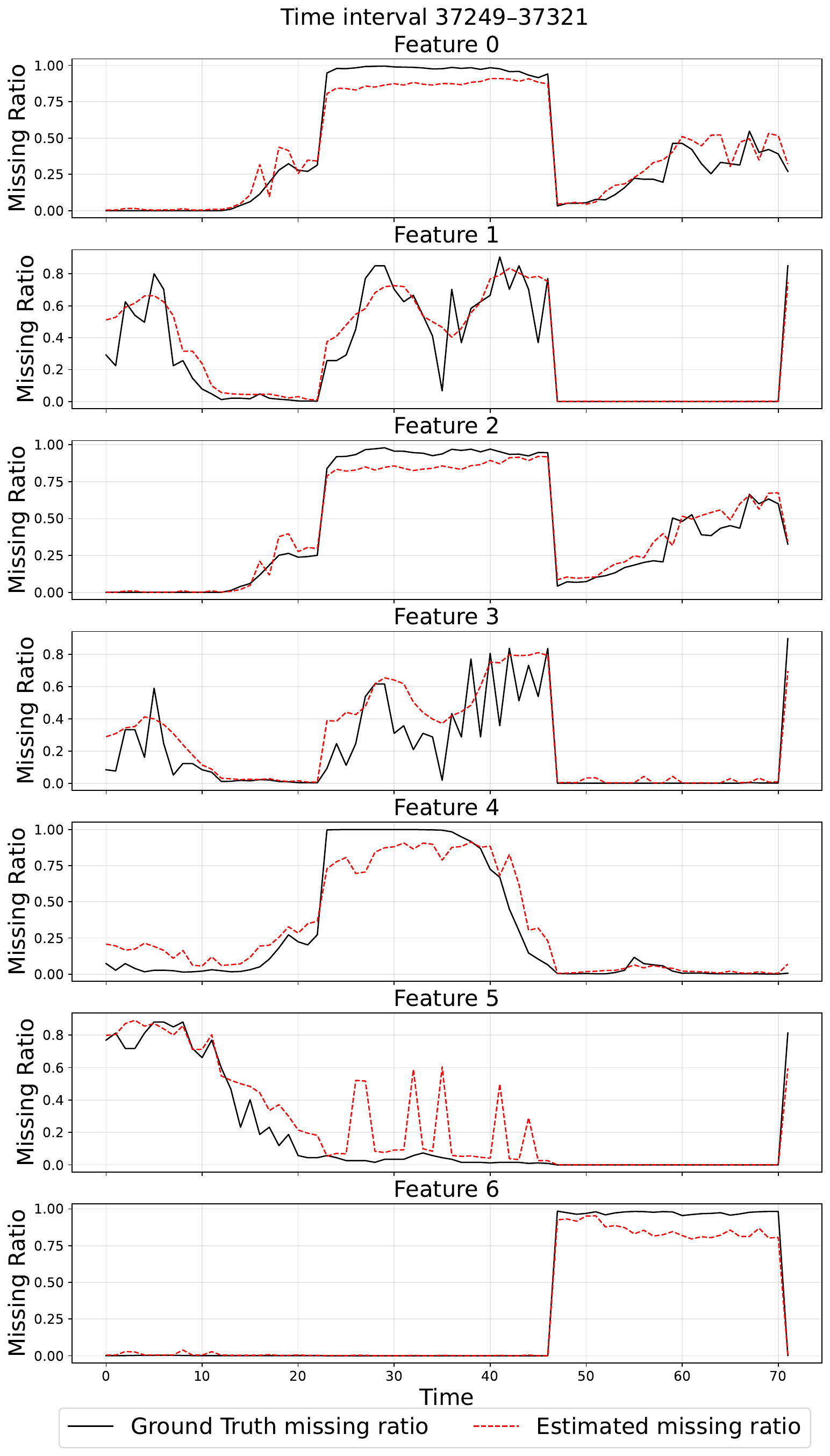}} &
    \subfloat{%
        \includegraphics[width=0.37\textwidth]{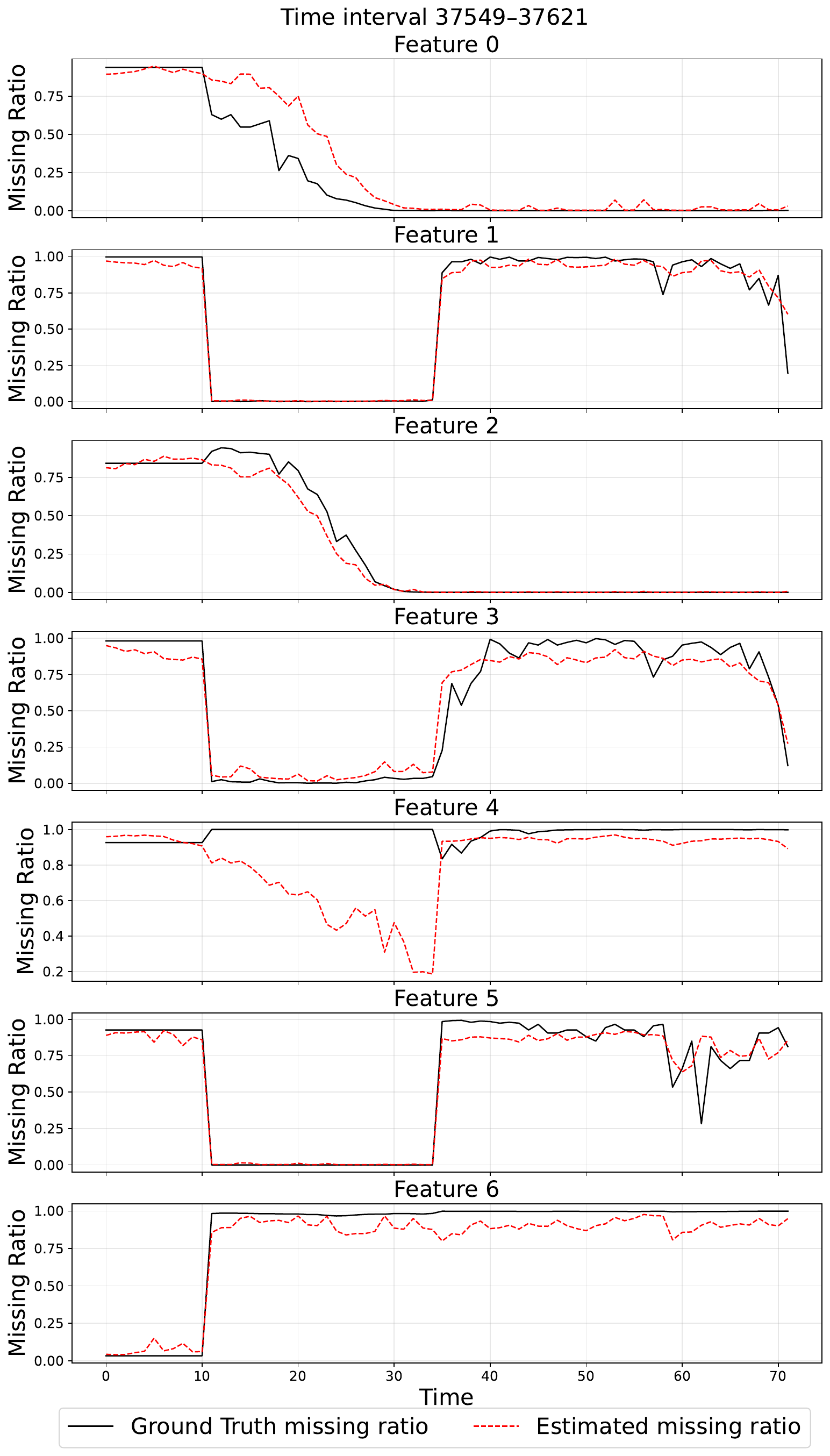}} \\
\end{tabular}
\vspace{-0.2cm}
\caption{Randomly sampled 4 ETT segments with time length 72: Ground-truth missing ratio (black) versus Pattern Recognizer-estimated missing ratio $D_\phi(\hat{X}_0)$ across 7 features.}
\label{tab:pr_case_study_ett}
\end{figure*}

\begin{figure*}[t]
\centering
\begin{tabular}{cc}
    \subfloat{%
        \includegraphics[width=0.4\textwidth]{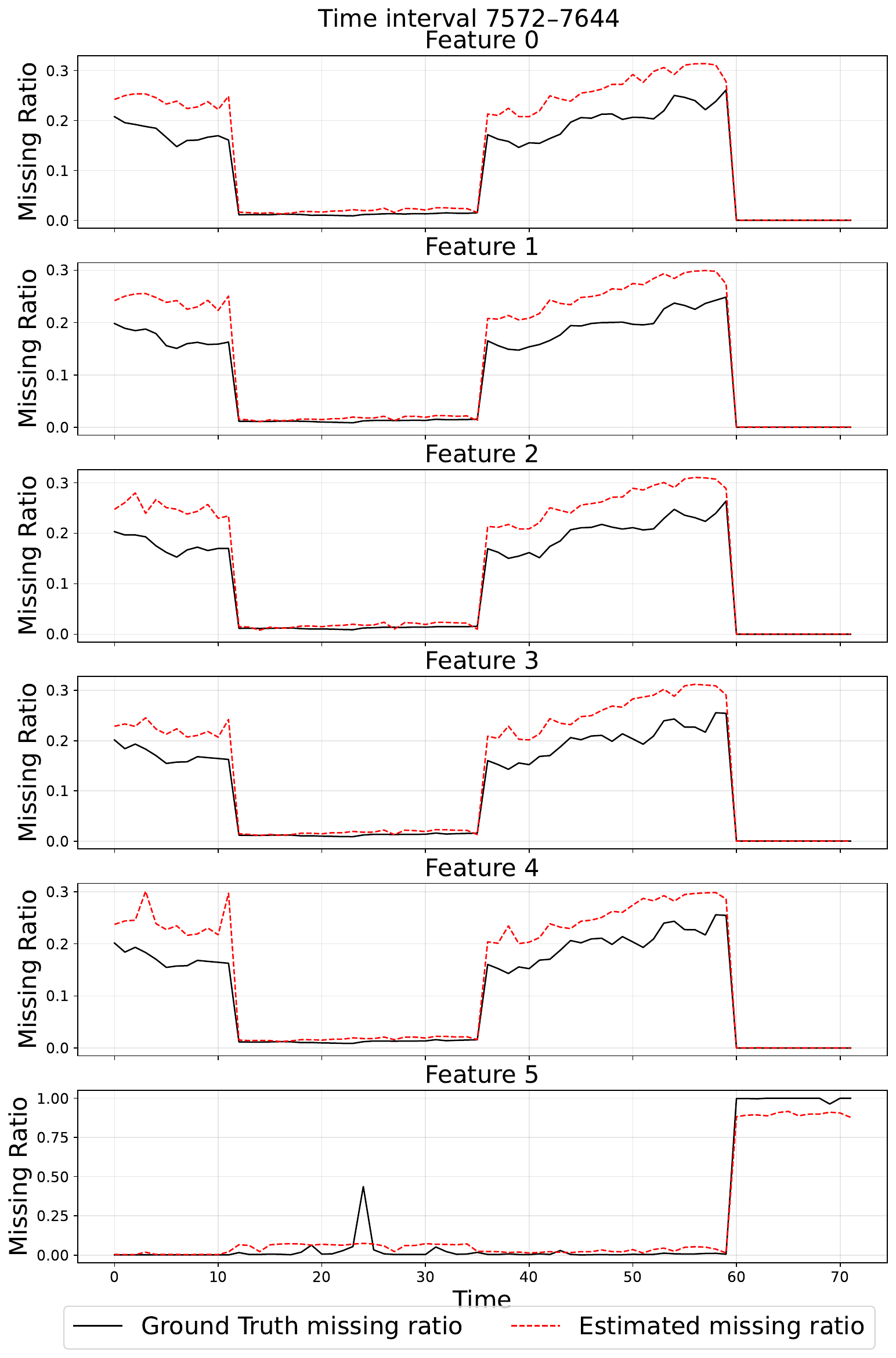}} &
    \subfloat{%
        \includegraphics[width=0.4\textwidth]{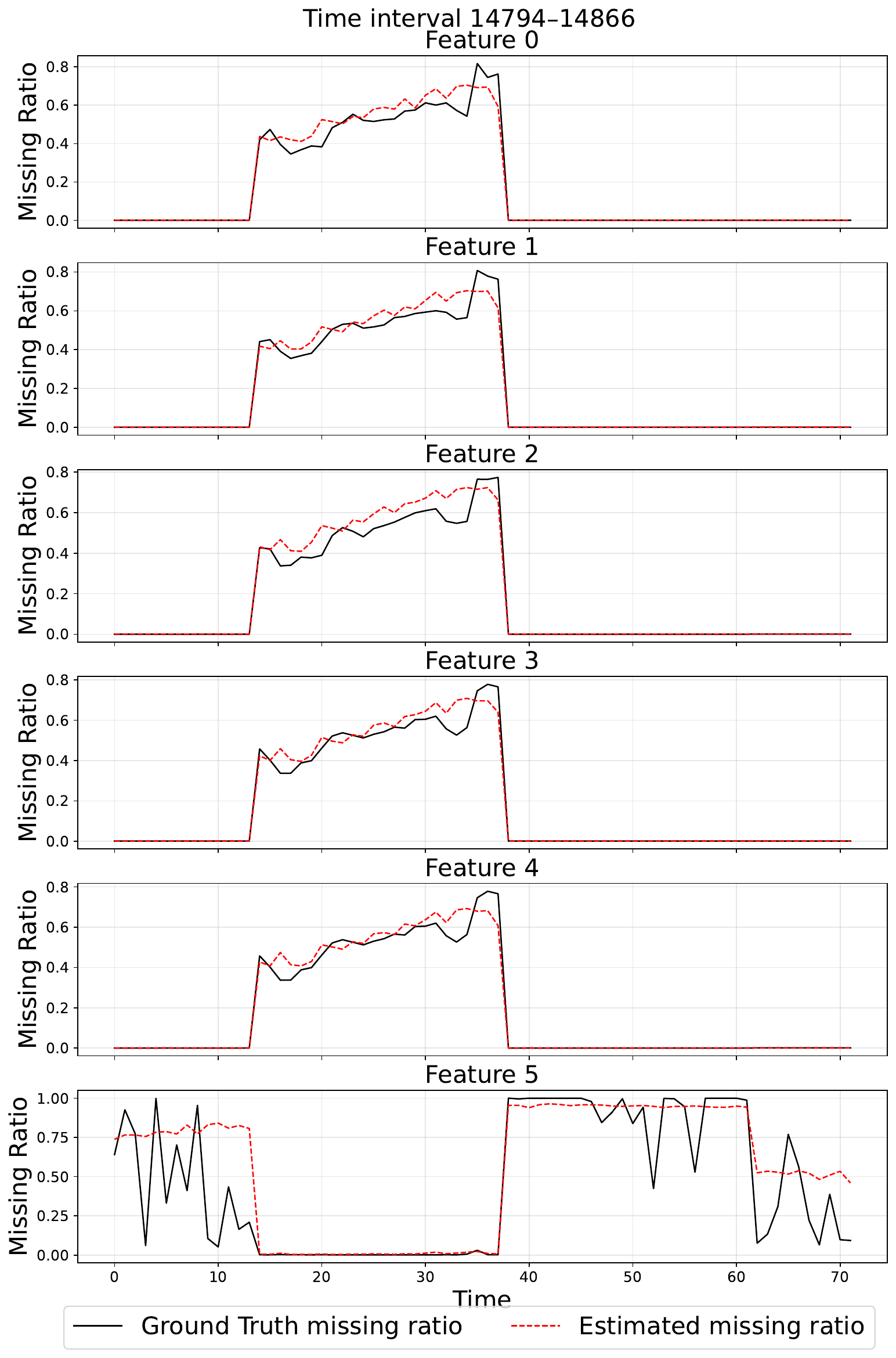}} \\
    \subfloat{%
        \includegraphics[width=0.4\textwidth]{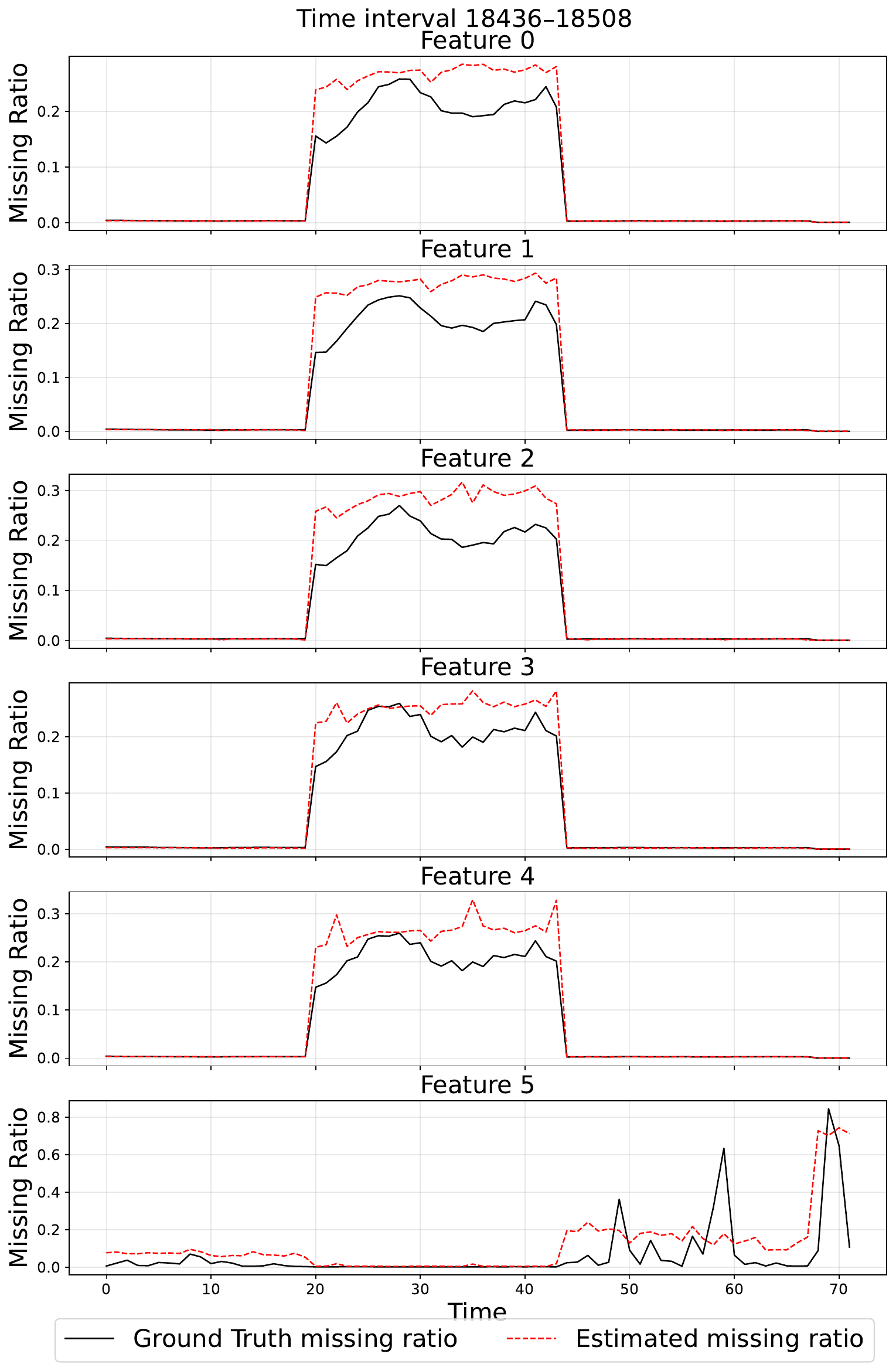}} &
    \subfloat{%
        \includegraphics[width=0.4\textwidth]{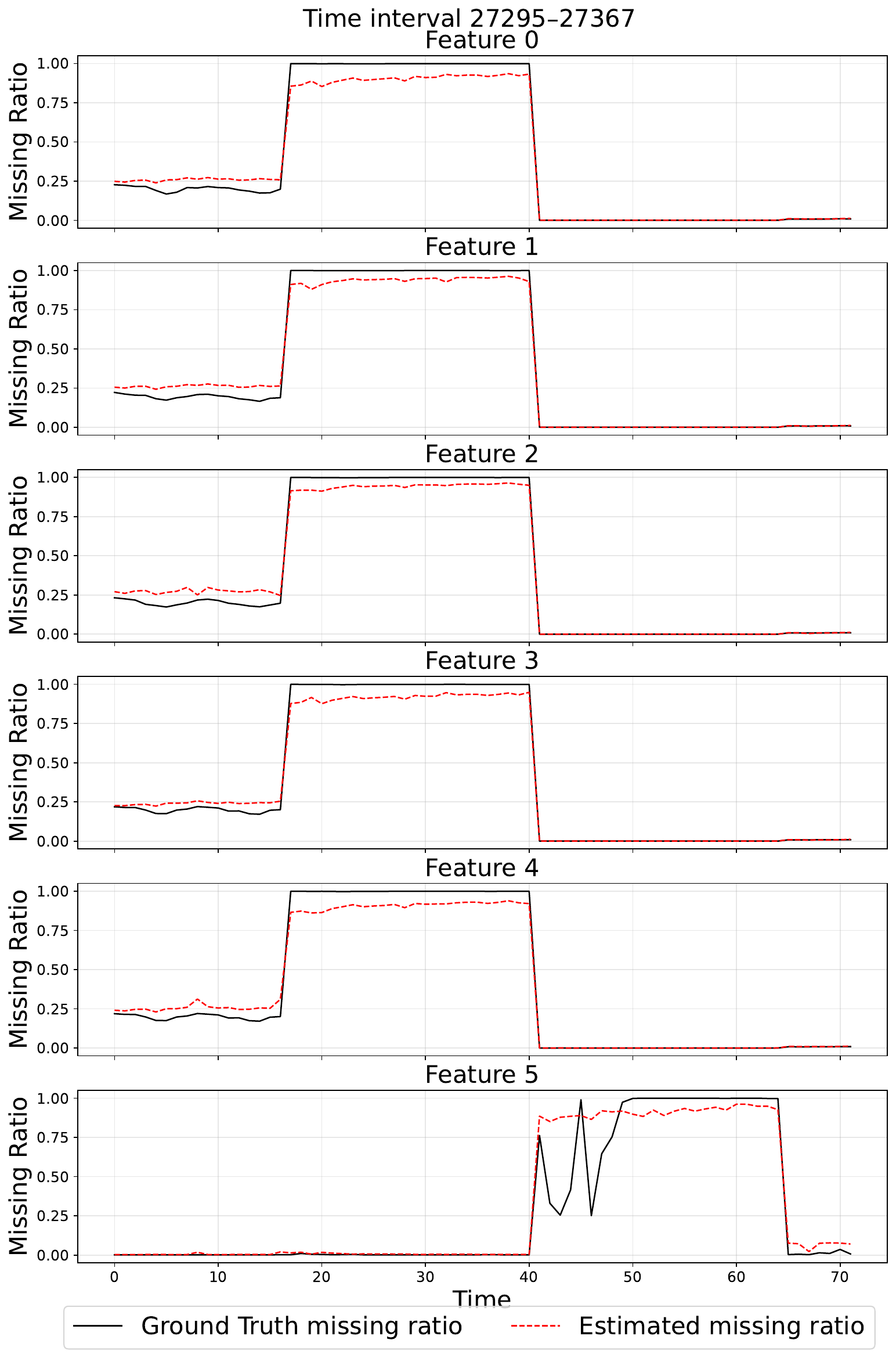}} \\
\end{tabular}
\caption{Randomly sampled 4 STOCK segments with time length 72: Ground-truth missing ratio (black) versus Pattern Recognizer-estimated missing ratio $D_\phi(\hat{X}_0)$ across 6 features.}
\label{tab:pr_case_study_stock}
\end{figure*}

\begin{figure*}[t]
\centering
\begin{tabular}{cc}
    \subfloat{%
        \includegraphics[width=0.4\textwidth]{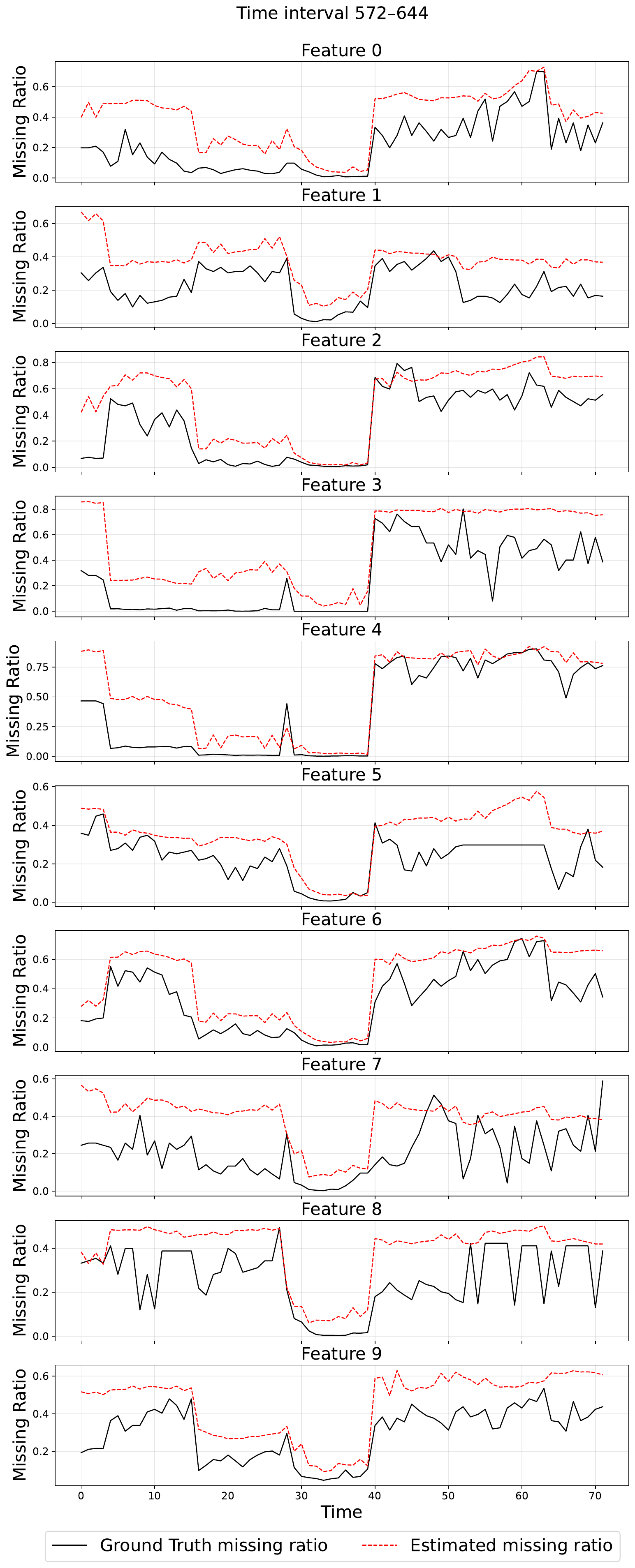}} &
    \subfloat{%
        \includegraphics[width=0.4\textwidth]{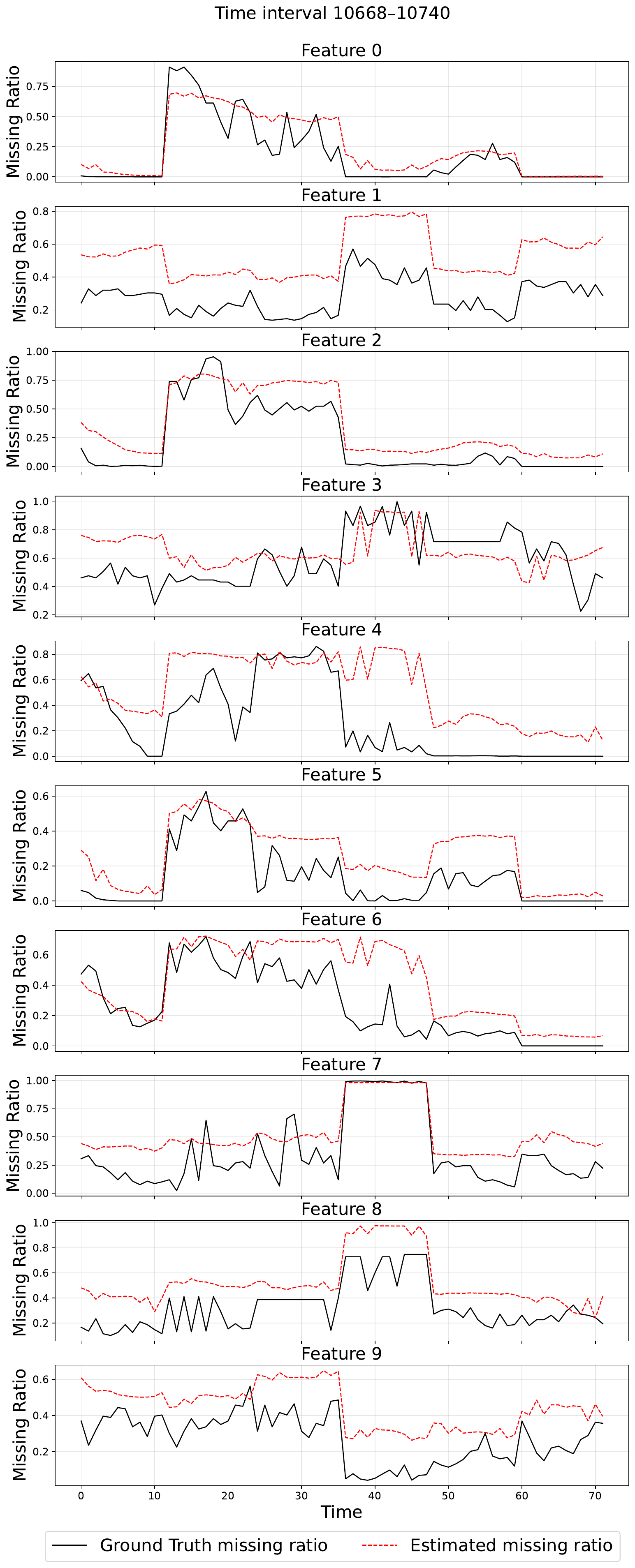}} \\
\end{tabular}
\caption{Randomly sampled 2 PEMS-Bay segments with time length 72: Ground-truth missing ratio (black) versus Pattern Recognizer-estimated missing ratio $D_\phi(\hat{X}_0)$ across 10 of 325 features.}
\label{tab:pr_case_study_pems1}
\end{figure*}

\begin{figure*}[t]
\centering
\begin{tabular}{cc}
    \subfloat{%
        \includegraphics[width=0.4\textwidth]{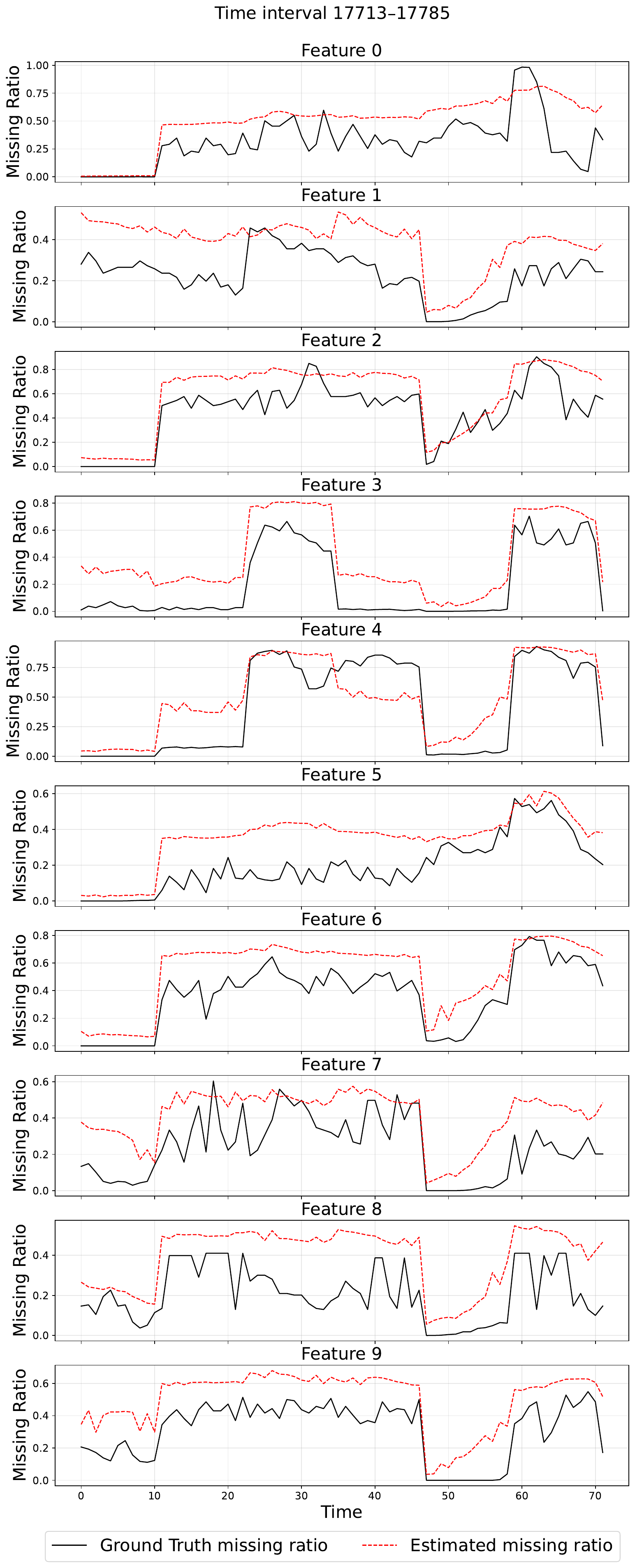}} &
    \subfloat{%
        \includegraphics[width=0.4\textwidth]{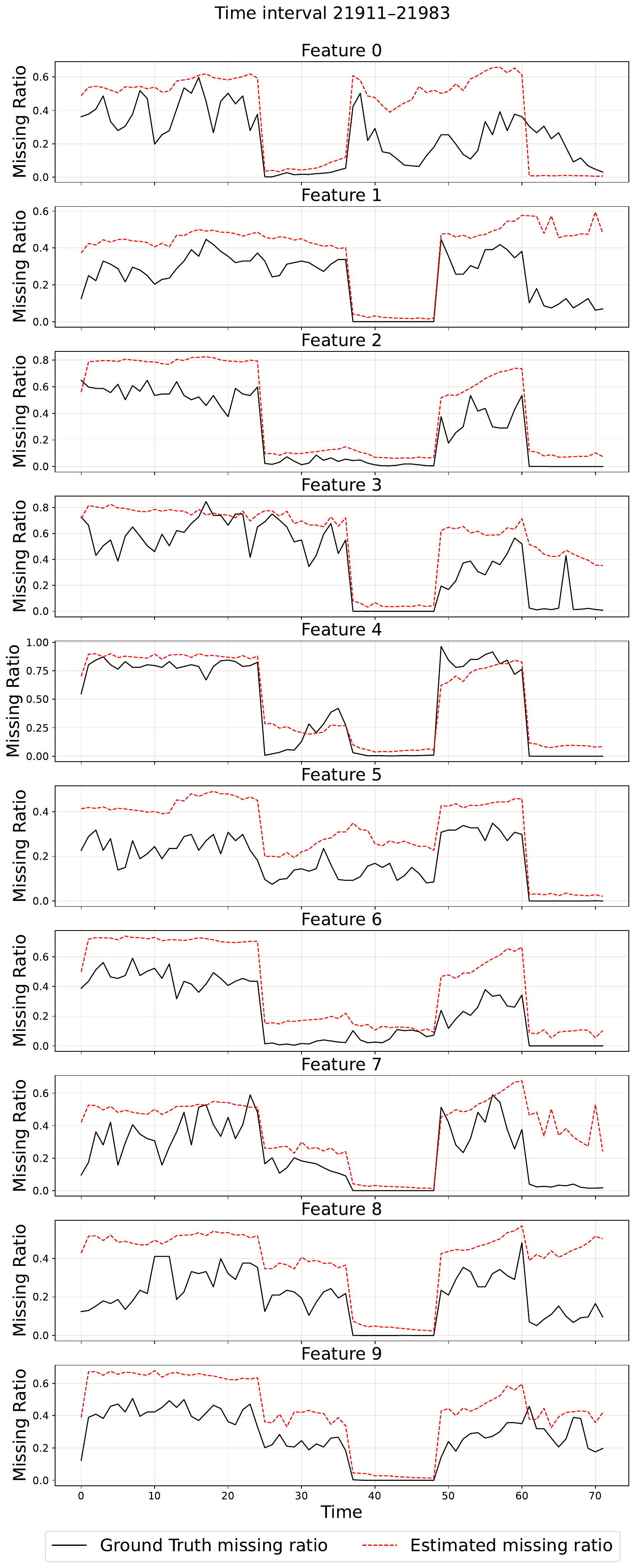}} \\
\end{tabular}
\caption{Additional 2 random samples of 2 PEMS-Bay segments with time length 72: Ground-truth missing ratio (black) versus Pattern Recognizer-estimated missing ratio $D_\phi(\hat{X}_0)$ across 10 of 325 features.}
\label{tab:pr_case_study_pems2}
\end{figure*}

\begin{figure}[t]
    \centering
    \includegraphics[width=1\linewidth]{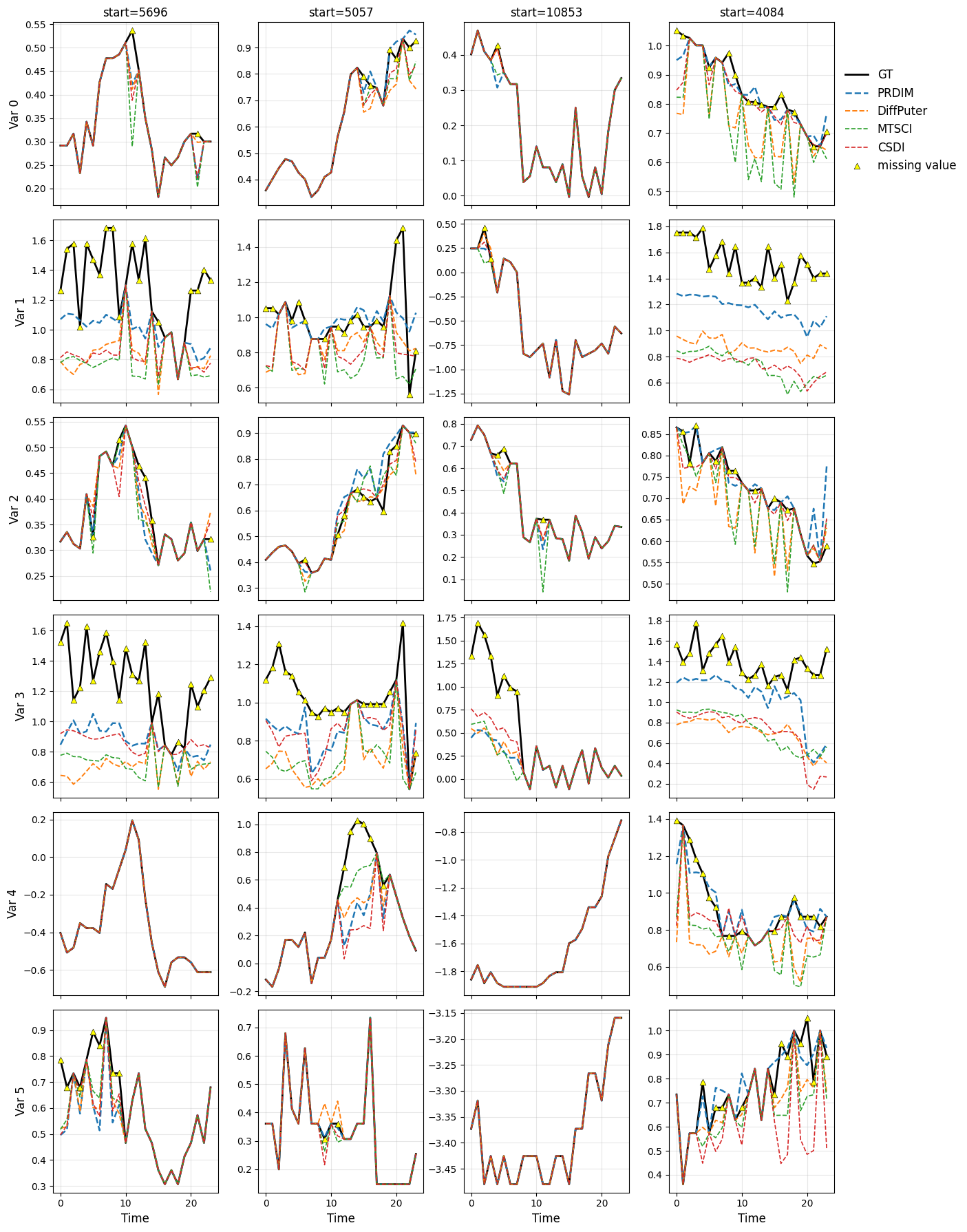}
    \caption{Qualitative results of PRDIM compared to other diffusion imputation models. 4 randomly selected imputed out-of-samples from the ETT dataset are visualized. Each panel labeled start = n corresponds to the time interval (n,n+24).}
    \label{fig:qual_ett}
\end{figure}

\begin{figure}[t]
    \centering
    \includegraphics[width=1\linewidth]{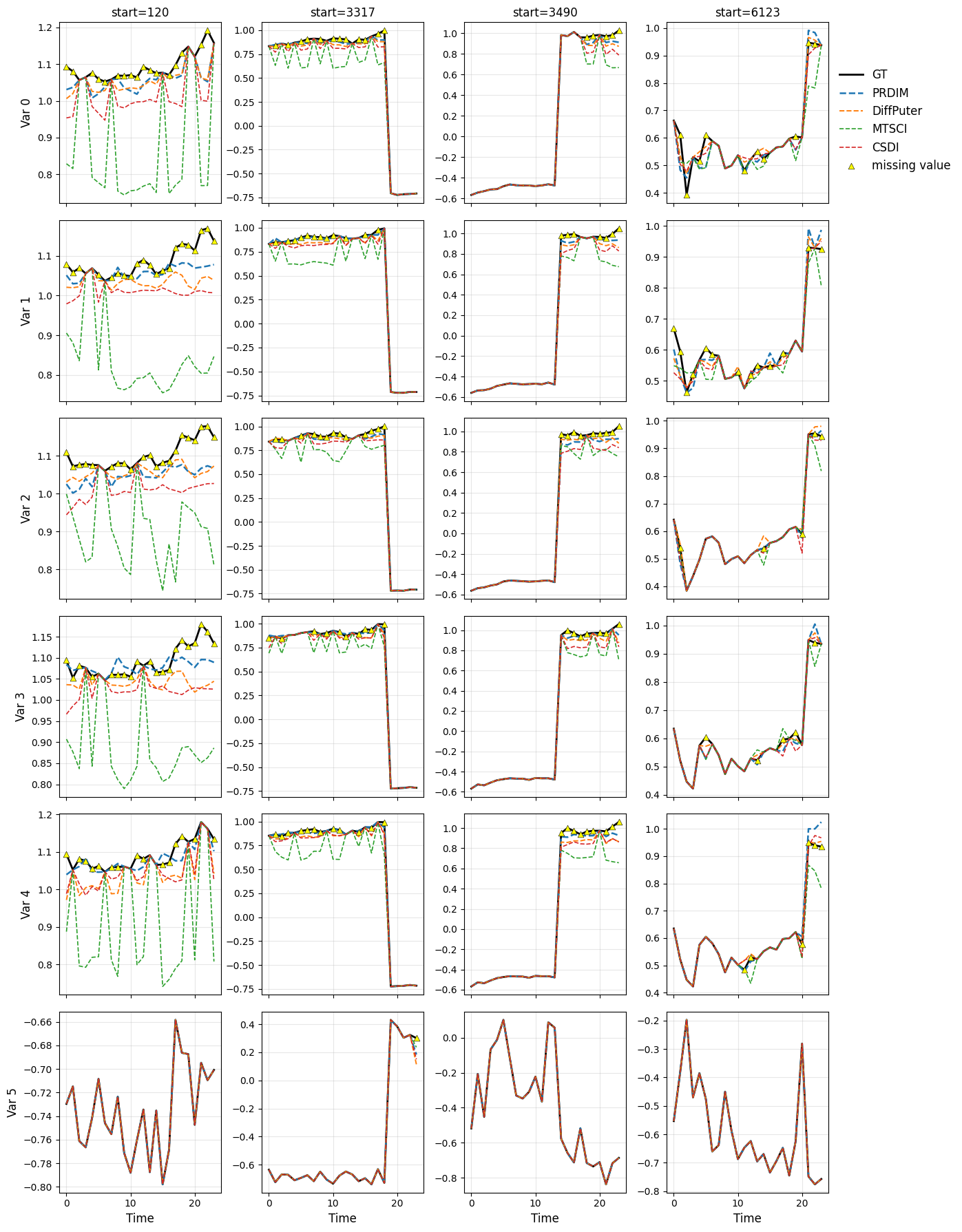}
    \caption{Qualitative results of PRDIM compared to other diffusion imputation models. 4 randomly selected imputed out-of-samples from the STOCK dataset are visualized. Each panel labeled start = n corresponds to the time interval (n,n+24).}
    \label{fig:qual_stock}
\end{figure}

\begin{figure}[t]
    \centering
    \includegraphics[width=1\linewidth]{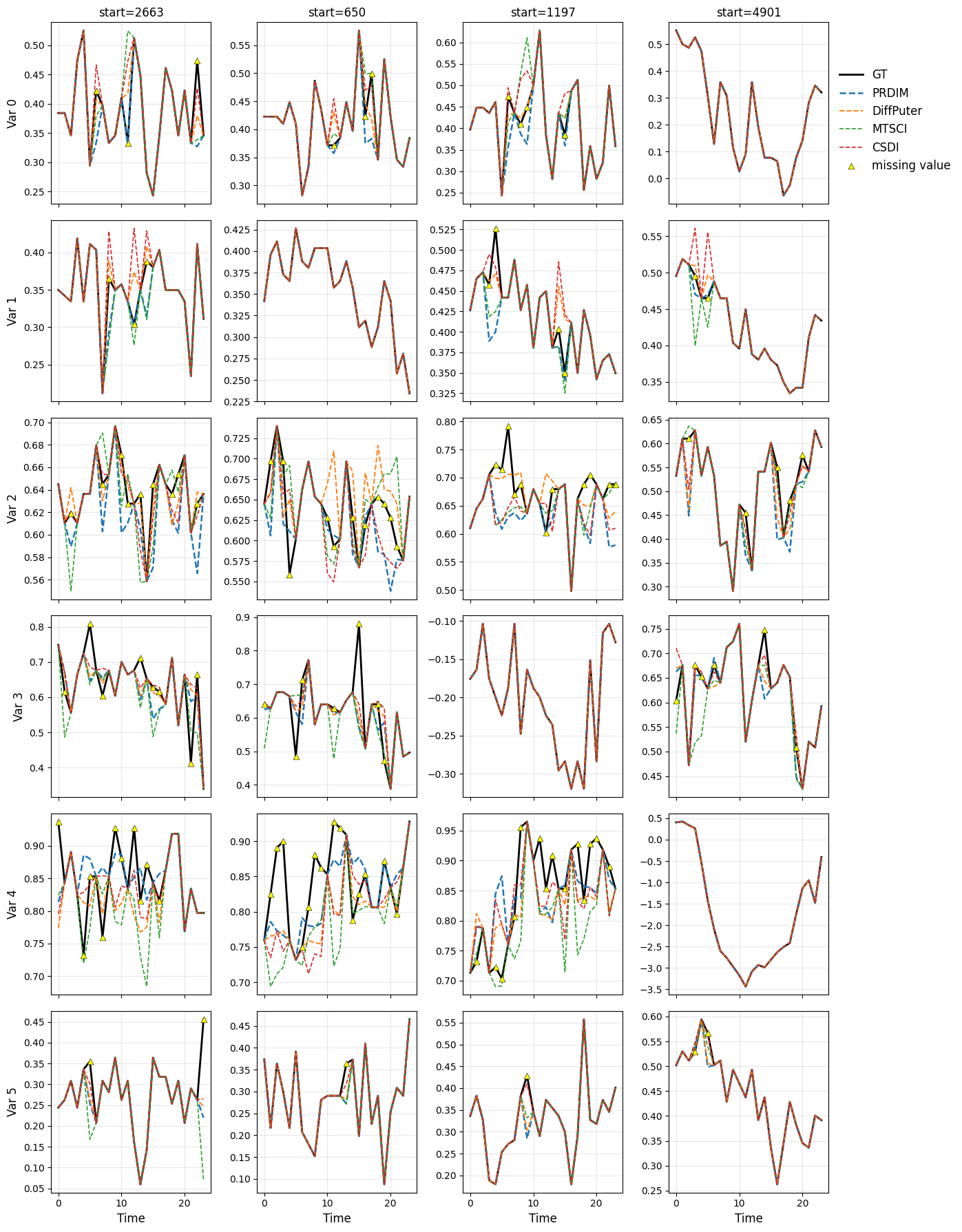}
    \caption{Qualitative results of PRDIM compared to other diffusion imputation models. 4 randomly selected imputed out-of-samples from the PEMS-Bay dataset are visualized. Each panel labeled start = n corresponds to the time interval (n,n+24).}
    \label{fig:qual_pems}
\end{figure}

\begin{figure}[t]
    \centering
    \includegraphics[width=1\linewidth]{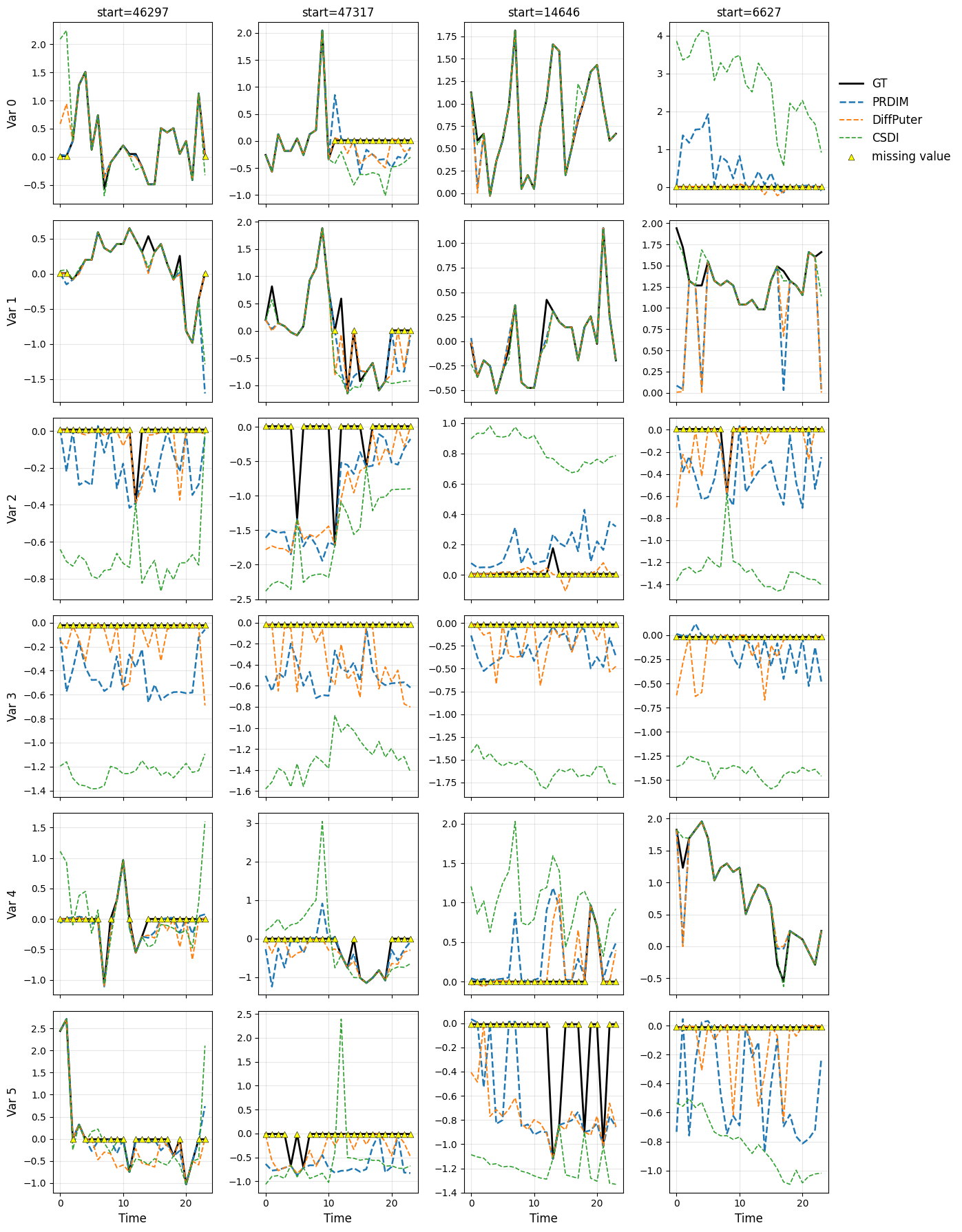}
    \caption{Qualitative results of PRDIM compared to other diffusion imputation models. 4 randomly selected imputed out-of-samples from the PhysioNet dataset are visualized. Each panel labeled start = n corresponds to the time interval (n,n+24).}
    \label{fig:qual_physio}
\end{figure}


\end{document}